\documentclass{article}
\usepackage[nonatbib,preprint]{neurips}
\makeatletter\renewcommand{\@noticestring}{Preprint.}\makeatother
\usepackage{amsmath,amssymb,amsthm}
\usepackage{booktabs}
\usepackage{microtype}
\usepackage{xcolor}
\usepackage{adjustbox}
\definecolor{linkblue}{rgb}{0.1,0.2,0.5}
\usepackage[colorlinks=true,linkcolor=linkblue,citecolor=linkblue,urlcolor=linkblue]{hyperref}
\usepackage{enumitem}
\setlist{itemsep=2pt,topsep=3pt}
\usepackage{tikz}
\usetikzlibrary{arrows.meta,calc,positioning,decorations.pathreplacing}
\usepackage{pgfplots}
\pgfplotsset{compat=1.17}
\usepgfplotslibrary{fillbetween}
\usepackage{fontspec}
\newfontfamily\faFont[Path=./fonts/]{fa-solid-900.ttf}
\newcommand{\faCarIcon}{{\faFont\symbol{"F1B9}}}
\newcommand{\faDogIcon}{{\faFont\symbol{"F6D3}}}

\newtheorem{theorem}{Theorem}[section]
\newtheorem{lemma}[theorem]{Lemma}
\newtheorem{proposition}[theorem]{Proposition}
\newtheorem{corollary}[theorem]{Corollary}
\theoremstyle{definition}
\newtheorem{definition}[theorem]{Definition}

\theoremstyle{remark}
\newtheorem{remark}[theorem]{Remark}

\newcommand{\R}{\mathbb{R}}
\newcommand{\Sph}{\mathbb{S}^{N-1}}
\newcommand{\ip}[2]{\langle #1,\, #2\rangle}
\newcommand{\norm}[1]{\lVert #1\rVert}
\newcommand{\conv}{\operatorname{conv}}
\newcommand{\img}{\mathbf{i}}
\newcommand{\txt}{\mathbf{t}}
\newcommand{\Lstar}{L^{*}}

\title{The Limits of Binding in Dual Encoders}
\author{Kin Ian Lo\\[2pt]
University College London\\[2pt]
\texttt{kin.lo.20@ucl.ac.uk}}
\date{16 August 2026}

\begin{document}
\maketitle
\begin{abstract}
\noindent
Dual-encoder models such as CLIP score an image--caption pair by a single inner product
of two independently computed unit vectors, and fail --- often near chance --- at
\emph{binding}: distinguishing ``a red car and a blue dog'' from ``a blue car and a red
dog.'' We give a mathematical account of when this failure is necessary and when it is
contingent. Working within the ideal-encoder framework proposed by Kang et al.,
we first show the relevant axioms are satisfiable, so every impossibility must enter
through an added, checkable hypothesis. We then prove three such obstructions.
\emph{Depth}: for recursive role-binding codes the swap margin obeys an exact law
$m(D)=2\,b^{-D}$ in the nesting depth $D$ (branching factor $b$), with a
finite-dimension version holding up to one explicitly flagged concentration estimate;
crossing the law with the noise floor yields a resolvable depth
$D^{*}=\log N/(2\log b)+O(1)$ --- single digits at CLIP scale, the nesting depth of
ordinary language. \emph{Objective}: architecture-free throttle theorems showing the
contrastive objective's entire reward for binding is confined to a sliver of width
$(\pi+\pi_{\mathrm{impl}})\log 2$ plus an exponentially small term, with $\pi$ the
swap-negative rate and $\pi_{\mathrm{impl}}$ the implicit collision rate, and that exactly reversed binding costs only $\pi$
times the mean binding margin $\mathbb E[M]$ --- both verified to within $0.006$ in
simulation. \emph{Geometry}: a smoothness--binding frontier $M(s)\le2\sqrt{2\delta}$
in the paraphrase-smoothness gap $\delta$, tight, whose text-only diagnostic we measure
across the text encoders of $18$ deployed models: every
model sits at ${\approx}25$--$35\%$ of its ceiling, and on $14$ of them the induced
ceiling tracks SugarCrepe's subset-level difficulty at $r=0.99$. Binding failure in deployed dual encoders is thus not
a dimension or smoothness limit today, but an incentive and code-structure limit ---
with a proved depth ceiling that remains once those are fixed.
\end{abstract}

\section{Introduction}\label{sec:mainintro}

A \emph{dual encoder} is a pair of neural networks: an image encoder and a text encoder,
trained (typically with a contrastive objective) so that matching image--caption pairs
receive a high score and mismatched pairs a low one. The score is a single inner product
between the two embedding vectors. This architecture --- CLIP \cite{clip} and its
descendants --- powers image retrieval, zero-shot classification, and the text
conditioning of image generators.

A large empirical literature documents that these models fail at \emph{compositional
binding}. Given an image of a red car next to a blue dog, CLIP-class models score the
caption ``a blue car and a red dog'' nearly as high as the correct caption --- sometimes
higher. Benchmarks built around such \emph{hard negatives} --- ARO \cite{aro}, SugarCrepe
\cite{sugarcrepe}, Winoground \cite{winoground} --- place strong models near chance
(ARO, Winoground) or at their weakest subsets (SugarCrepe's swaps) exactly on the
negative types that permute which attribute goes with which object.

Two questions must be separated. \emph{Why do trained models fail today?} --- a question
about architectures, objectives, and data, answered quantitatively in our companion
paper \cite{paper1}. And: \emph{what limits binding in principle?} --- a question about
the geometry of the embedding space itself: are there theorems that no amount of scale,
data, or training can evade? This paper is about the second question.

The framework we build on is due to Kang et al.~\cite{kang}, who proposed studying
\emph{ideal} encoders: instead of asking what a particular network learns, one asks
whether \emph{any} assignment of unit embedding vectors could satisfy the inequalities
that benchmark competence requires. Their proposal asks the right question; we work with our own formulation of the
axioms (Appendix~\ref{sec:setup}), stated so that the quantifier structure and the
negative families are fully explicit. One structural fact then constrains the whole analysis: \emph{the axioms are
satisfiable}
(Appendix~\ref{sec:possibility}) --- satisfiable in dimension as low as $N=2k$ ($k$ the number of atoms in a scene), and
realized cheaply at deployed scales by random factored sign codes --- so there is no
unconditional ``dual encoders cannot bind''
theorem. Every rigorous account of the failure must be \emph{conditional}: impossibility
or limitation \emph{given} something; the scientific content is in how natural and
how checkable that something is. Our contributions are three such
conditional obstructions, each with its hypothesis stated as a measurable or structural
condition, plus the measurements that locate deployed models relative to them.

\begin{itemize}
\item \textbf{A proved depth ceiling for recursive binding}
(\S\ref{sec:maindepth}). For recursive role-binding codes --- the classical
holographic / vector-symbolic (VSA) composition \cite{plate,kanerva}, and the recursive form
of pooled codes --- the swap margin obeys an
\emph{exact} law $m(D)=2\,b^{-D}$ in nesting depth $D$, and the replace margin
$b^{-D}$. We prove the law exactly in the zero-crosstalk idealization, and in finite
dimension $N$ as a multiplicative-error theorem, modulo one concentration estimate that
we state and flag explicitly. The resolvable depth is $D^{*}=\log N/(2\log b)+O(1)$: at CLIP's $N=512$ the
measured detection floor gives $D^{*}=2$ for $b=2$ (the pure crossing gives $4.5$) ---
the depth of ordinary language. A companion aliasing
proposition shows that with self-inverse (shared) roles some deep edits are invisible
at \emph{any} dimension; fresh per-level roles repair this, measurably.
\item \textbf{Throttle theorems for the contrastive objective}
(\S\ref{sec:mainthrottle}). Architecture-free: for \emph{any} scorer, a pairing-blind
(bag-of-words) competitor is $\varepsilon$-optimal with
$\varepsilon\le(\pi+\pi_{\mathrm{impl}})\log2$ plus an exponential tail, where $\pi$ is
the explicit swap-negative rate and $\pi_{\mathrm{impl}}$ the implicit collision rate;
and the \emph{exactly reversed} scorer costs precisely $\pi\,\mathbb E[M]$ more. At web
scale $\pi$ is vanishing, so the objective's entire reward for binding --- and its
entire penalty for anti-binding --- collapses. Both constants are verified to within
$0.006$ in simulation.
\item \textbf{A tight smoothness--binding frontier, measured across deployed models}
(\S\ref{sec:mainfrontier}). If the two swap-related captions embed $\delta$-close to a
pairing-neutral anchor (a similarity/paraphrase demand), then the binding margin obeys
$M(s)\le2\sqrt{2\delta}$, and the constant is exact. The bound induces a text-only diagnostic ---
frontier \emph{usage}, the ratio of the caption--swap embedding distance $d$ to the
ceiling $2\sqrt{2\delta}$ --- which we measure on the
text encoders of $18$ deployed models: every one sits at ${\approx}25$--$35\%$ of its ceiling, so no deployed
model is smoothness-limited; the failure lies in the code and the objective, not the
geometry's budget. The per-item ceiling
$|M(s)|\le d(s)$ tracks subset-level difficulty on a real benchmark (SugarCrepe;
$14$ models, $66{,}598$ decisions) at $r=0.99$.
\end{itemize}

We also record what does \emph{not} explain the failure (\S\ref{sec:mainnot}):
dimension (the possibility theorem), and exposure-graph connectivity (a natural
conjecture we refute --- a sign-agnostic matcher generalizes across disconnected
exposure components at accuracy $1.0$). Full definitions, proofs, protocols, and the
negative results are developed self-containedly in the appendices.

\section{Setup}\label{sec:mainsetup}

Fix finite sets of objects and attributes. An \emph{atom} is a pair $(o,a)$ (``object
$o$ has attribute $a$''); a $k$-\emph{scene} $s$ is a set of $k$ atoms with distinct
objects (and, where swaps are discussed, distinct attributes). Each scene has a caption
$c(s)$ asserting exactly its atoms and an image $I_s$ depicting them. The \emph{swap}
$\sigma s$ exchanges the attributes of two chosen atoms: $c(s)$ and $c(\sigma s)$
contain the \emph{same multiset of words} --- we call this shared multiset of object
and attribute symbols the scene's \emph{symbol multiset} --- and differ only in which
attribute is bound to which object: the binding-critical hard negative instantiated by ARO and SugarCrepe
(Appendix~\ref{sec:setup} maps the benchmark taxonomies onto the formal negative
families).

A \emph{dual encoder} is a pair of maps $\img,\txt$ into the unit sphere
$\mathbb S^{N-1}$, scoring a pair by $\ip{\img(I)}{\txt(c)}$. Benchmark competence is
formalized by two axioms on a \emph{placement} --- an arbitrary assignment of unit
vectors to images and captions --- for a margin $\gamma>0$ (their full quantifier
structure is in Appendix~\ref{sec:setup}):
\begin{itemize}
\item[$(C_\gamma)$] \emph{Margin retrieval:} for every scene $s$ and every caption
$c'$ describing a different scene,
$\ip{\img(I_s)}{\txt(c(s))}\ge\ip{\img(I_s)}{\txt(c')}+\gamma$;
\item[$(K_k)$] \emph{Constituent ranking:} the image of a scene scores every
\emph{present} object's single-object caption above every \emph{absent} object's.
\end{itemize}
The quantity every theorem controls is the \emph{binding margin}
\begin{equation}
M(s)\;:=\;\ip{\img(I_s)}{\,\txt(c(s))-\txt(c(\sigma s))\,},
\end{equation}
the swap instance of $(C_\gamma)$: benchmark accuracy on a swap item is
$\mathbf 1[M(s)>0]$.

Because sign conventions matter below: all embeddings are unit vectors, $\odot$ is the
coordinatewise product, and a \emph{sign vector} has entries $\pm1$.

\section{The depth ceiling of recursive binding}\label{sec:maindepth}

Pooled composition --- encoding a set of constituents as the normalized sum of their
vectors --- has margins $\Theta(1/k)$ in the number $k$ of constituents pooled at one
level (Appendix~\ref{sec:thmC}). Recursion is what makes language productive,
and it is where the sharpest law appears. Consider the classical recursive role--filler
composition of holographic/vector-symbolic representations \cite{plate,kanerva}: a full
$b$-ary tree of depth $D$ whose leaves carry random unit sign vectors, each internal
node encoded as the normalized role-bound sum of its children,
$\mathrm{code}(v)=\sum_j r_j\odot\mathrm{code}(\mathrm{child}_j(v))\,/\,
\norm{\sum_j r_j\odot\mathrm{code}(\mathrm{child}_j(v))}$, with sign-vector roles $r_j$
either \emph{shared} across levels (the classical convention) or drawn \emph{fresh} per
level. The \emph{deep-edit margin} $m(D)$ is $1-\cos$ between the codes of a tree and a
minimal edit of it: a \emph{sibling swap} (two leaves under one parent), a \emph{far
swap} (two leaves whose lowest common ancestor is the root), or a \emph{replace}
(Appendix~\ref{sec:depth}, Definition~\ref{def:recursive}).

\begin{figure}[t]
\centering
\begin{tikzpicture}[font=\scriptsize,
  nd/.style={draw=black!55, rounded corners=1.5pt, inner sep=2.5pt, fill=black!4},
  lf/.style={draw=black!45, rounded corners=1.5pt, inner sep=2.2pt},
  ed/.style={black!50}, arr/.style={-{Stealth[length=1.8mm]}, teal!55!black}]
\node[nd] (root) at (4.55,3.3) {$\mathrm{normalize}(r_1\odot\cdot + r_2\odot\cdot)$};
\node[nd] (n0) at (2.1,2.2) {$\oplus$};
\node[nd] (n1) at (7.0,2.2) {$\oplus$};
\node[nd] (n00) at (0.9,1.1) {$\oplus$};
\node[nd] (n01) at (3.3,1.1) {$\oplus$};
\node[nd] (n10) at (5.8,1.1) {$\oplus$};
\node[nd] (n11) at (8.2,1.1) {$\oplus$};
\foreach \i/\x in {0/0.3, 1/1.5, 2/2.7, 3/3.9, 4/5.2, 5/6.4, 6/7.6, 7/8.8}
  \node[lf] (l\i) at (\x,0.0) {$x_{\i}$};
\foreach \p/\c in {root/n0, root/n1, n0/n00, n0/n01, n1/n10, n1/n11,
                   n00/l0, n00/l1, n01/l2, n01/l3, n10/l4, n10/l5, n11/l6, n11/l7}
  \draw[ed] (\p) -- (\c);
\node[lf, draw=orange!85!black, thick] at (0.3,0.0) {$x_0$};
\node[lf, draw=orange!85!black, thick] at (1.5,0.0) {$x_1$};
\node[orange!85!black] at (0.9,-0.55) {swapped pair};
\node[teal!55!black, anchor=west] at (9.3,1.1) {$\times\,1/b$};
\node[teal!55!black, anchor=west] at (9.3,2.2) {$\times\,1/b$};
\node[teal!55!black, anchor=west] at (9.3,3.3) {$\times\,1/b$};
\draw[arr] (9.55,0.35) -- (9.55,3.0);
\node[teal!45!black, anchor=west, align=left] at (9.9,1.7)
  {each pooling level\\ dilutes the edit's\\ share of the code};
\node[align=center] at (4.55,-1.15)
  {root margin $m(D)\;=\;2\cdot b^{-D}$ \quad (Theorem~\ref{thm:depthA}; here $b=2$, $D=3$: $m=1/4$)};
\end{tikzpicture}
\caption{Recursive role-binding (Definition~\ref{def:recursive}) and the depth law. A deep
edit (orange) changes one summand among $b$ at its parent; each further level of normalized
pooling multiplies its share of the root code by $1/b$, so the discriminative margin decays
geometrically with nesting depth --- the recursive continuation of the $\Theta(1/k)$
single-level margin bound of Theorem~\ref{thm:Cswap}.}
\label{fig:tree}
\end{figure}

The core of both proofs is one deterministic lemma: a level of role-bound, normalized
pooling contracts edit-induced discrepancies by exactly $1/b$, with all error controlled
by cross-slot inner products.

\begin{lemma}[One-level contraction; Appendix~\ref{sec:depth}, Lemma~\ref{lem:contract}]
\label{lem:main-contract}
Let $z_1,\dots,z_b$ and $z'_1,\dots,z'_b$ be unit vectors agreeing outside a set $E$ of
\emph{edited slots}; write $g_j:=1-\ip{z_j}{z'_j}$ for the \emph{slot deficit} of
$j\in E$, and $G:=\sum_{j\in E}g_j$. Bind and pool: $P:=\sum_j\rho_j\odot z_j$ and
$P':=\sum_j\rho_j\odot z'_j$ with sign-vector roles $\rho_j$, and let the
\emph{crosstalk} $\varepsilon$ be the largest inner product in magnitude between
distinct bound slots. If $(b-1)\varepsilon\le\tfrac12$, then the normalized codes
satisfy
\[
\Bigl|\bigl(1-\ip{P/\norm{P}}{P'/\norm{P'}}\bigr)-\tfrac{G}{b}\Bigr|
\;\le\;8(b-1)\varepsilon .
\]
\end{lemma}

\begin{theorem}[Exact depth law; Appendix~\ref{sec:depth}, Theorem~\ref{thm:depthA}]
\label{thm:main-depthlaw}
In the zero-crosstalk idealization ($\varepsilon=0$ at every level --- the
$N\to\infty$ limit of the random model), for either role scheme, provided all bound
products appearing in the construction are distinct:
\[
m_{\mathrm{swap}}(D)\;=\;2\cdot b^{-D},\qquad
m_{\mathrm{replace}}(D)\;=\;1\cdot b^{-D},
\]
for both sibling and far swaps. More generally, an edit whose slot deficits at level
$\ell_0$ sum to $G$ has root margin $G\cdot b^{-(D-\ell_0)}$: the margin depends only on
the total deficit and the number of contraction levels, not on where in the tree the
edit sits.
\end{theorem}

The theorem is Lemma~\ref{lem:main-contract}'s upward induction; at $\varepsilon=0$
the contraction is exact, and the finite-$N$ theorem below runs the same induction with
a refined form of the lemma whose injected error is proportional to $\sqrt{g}$
(Appendix~\ref{sec:depth}, Lemma~\ref{lem:contract2}) --- which is what makes the
finite-$N$ error multiplicative rather than additive. Two consequences of the mechanism
are worth stating: the previously puzzling equality of sibling and far swaps is
transparent (a far swap is two replacements in parallel branches, and $1{+}1$ contracted
equals $2$ contracted); and depth $1$ with $b=k$ recovers the single-level constants of
Theorem~\ref{thm:Cswap} (swap $2/k$, replace $1/k$).

\begin{theorem}[Finite dimension; Appendix~\ref{sec:depth}, Theorem~\ref{thm:depthB}]
\label{thm:main-finiteN}
For the fresh role scheme, $D\ge2$, and $N\ge C\,b^2D^2\log(bQ/\eta)$ (where
$Q\le b^D+1$ counts the distinct leaf vectors, the \emph{primitives}), with probability at least $1-\eta$,
simultaneously for every edit above,
\[
m(D)\;=\;m_{\mathrm{ideal}}(D)\cdot(1\pm\kappa_N),
\qquad
\kappa_N\;\le\;C'\sqrt{\frac{b^{D}\,\log(bQD/\eta)}{N}},
\]
with $C,C'$ absolute constants. The proof is complete modulo one concentration
estimate --- the high-probability form of an $\ell_4$-norm propagation step
(Appendix~\ref{sec:depth}, Lemma~\ref{lem:ellfour}) --- which we state precisely and
leave as the single explicit gap; its expectation form is proved, and the theorem's
multiplicative-error signature is confirmed in simulation (relative spread of $m$ flat
in $D$ while the absolute spread shrinks ${\approx}370\times$ from $D{=}1$ to $D{=}10$).
\end{theorem}

\begin{corollary}[Resolvable depth]\label{cor:main-dstar}
Against the primitive-crosstalk noise floor $\nu(N)\propto\sqrt{1/N}$ of any linear
readout, binding at depth $D$ is resolvable when $2b^{-D}(1-\kappa_N)>\nu(N)$ (depth
$1$ is covered by Theorem~\ref{thm:Cswap}), i.e.
\[
D\;\le\;D^{*}(N)\;=\;\frac{\log N}{2\log b}+O(1),
\]
a staircase gaining one level per $b^2$-fold increase of $N$.
\end{corollary}

Simulation confirms the law exactly: across $b\in\{2,3\}$ ($b=2$ to depth $10$, $b=3$
to $7$) and $N$ from $256$ to $65{,}536$ ($24$ seeds per configuration; shared roles,
with a side-by-side fresh-role sweep agreeing away from the aliased cells,
Appendix~\ref{sec:depth}), the products $m\cdot b^D$ sit
at $2$ and $1$ --- within $0.014$ and $0.006$ at $N=65{,}536$ --- with the smaller-$N$
deviations following Theorem~\ref{thm:main-finiteN}'s finite-$N$ term, and the measured
$D^{*}$ staircase is exactly $1,2,3,4,5$ at $N=256,\dots,65{,}536$ for $b=2$
(Figure~\ref{fig:depthlaw}; full tables in Appendix~\ref{sec:depth}). At CLIP's
$N=512$: the pure crossing $\log N/(2\log b)$ gives $4.5$ ($b=2$) and $2.8$ ($b=3$),
and the measured $10/\sqrt N$ floor gives $D^{*}=2$ and $1$ --- single digits either
way. \textbf{The ceiling sits at ordinary natural-language nesting depth}: beyond it, resolving
composition requires structured (tree-side) scoring rather than a single pooled vector.

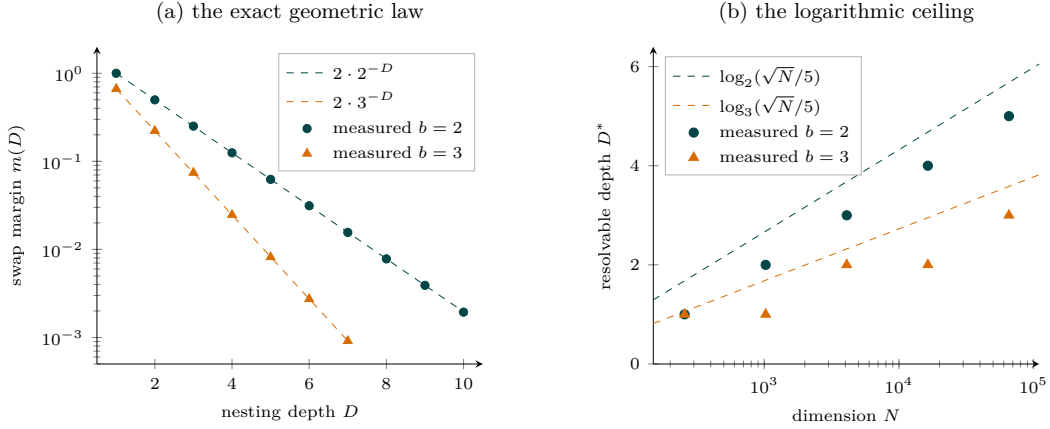
\begin{figure}[t]
\centering
\begin{adjustbox}{max width=\linewidth}\begin{tikzpicture}
\begin{semilogyaxis}[width=7.2cm, height=6.2cm, axis lines=left,
  xlabel={nesting depth $D$}, ylabel={swap margin $m(D)$},
  xmin=0.5, xmax=10.5, ymin=5e-4, ymax=2,
  tick label style={font=\scriptsize}, label style={font=\scriptsize},
  legend style={font=\scriptsize, draw=black!30}, legend cell align=left,
  title={(a) the exact geometric law}, title style={font=\small}]
\addplot[teal!55!black, domain=1:10, samples=40, dashed] {2*2^(-x)};
\addlegendentry{$2\cdot2^{-D}$}
\addplot[orange!85!black, domain=1:7, samples=40, dashed] {2*3^(-x)};
\addlegendentry{$2\cdot3^{-D}$}
\addplot[only marks, mark=*, mark size=1.7pt, teal!55!black] coordinates
  {(1,1.00)(2,0.499)(3,0.251)(4,0.125)(5,0.0624)(6,0.0313)(7,0.0156)(8,0.00781)(9,0.00391)(10,0.00194)};
\addlegendentry{measured $b=2$}
\addplot[only marks, mark=triangle*, mark size=2.1pt, orange!85!black] coordinates
  {(1,0.666)(2,0.222)(3,0.0742)(4,0.0247)(5,0.00818)(6,0.00273)(7,0.000912)};
\addlegendentry{measured $b=3$}
\end{semilogyaxis}
\begin{semilogxaxis}[at={(8.1cm,0)}, width=7.2cm, height=6.2cm, axis lines=left,
  xlabel={dimension $N$}, ylabel={resolvable depth $D^*$},
  xmin=150, xmax=110000, ymin=0, ymax=6.4,
  tick label style={font=\scriptsize}, label style={font=\scriptsize},
  legend style={font=\scriptsize, at={(0.03,0.97)}, anchor=north west, draw=black!30},
  legend cell align=left,
  title={(b) the logarithmic ceiling}, title style={font=\small}]
\addplot[teal!55!black, dashed, domain=150:110000, samples=60] {ln(sqrt(x)/5)/ln(2)};
\addlegendentry{$\log_2(\sqrt N/5)$}
\addplot[orange!85!black, dashed, domain=150:110000, samples=60] {ln(sqrt(x)/5)/ln(3)};
\addlegendentry{$\log_3(\sqrt N/5)$}
\addplot[only marks, mark=*, mark size=1.9pt, teal!55!black] coordinates
  {(256,1)(1024,2)(4096,3)(16384,4)(65536,5)};
\addlegendentry{measured $b=2$}
\addplot[only marks, mark=triangle*, mark size=2.3pt, orange!85!black] coordinates
  {(256,1)(1024,1)(4096,2)(16384,2)(65536,3)};
\addlegendentry{measured $b=3$}
\end{semilogxaxis}
\end{tikzpicture}\end{adjustbox}
\caption{The depth ceiling, measured. \textbf{(a)} Swap margins sit on
$2\cdot b^{-D}$ across three orders of magnitude (log scale; $N=65{,}536$ shown, where the
products $m\cdot b^{D}$ are within $0.014$ of $2$; at smaller $N$ the deviations grow as
the finite-$N$ term of Theorem~\ref{thm:main-finiteN} allows ---
Appendix~\ref{sec:measuredlaw}).
\textbf{(b)} Resolvable depth (largest $D$ with margin above the
$10/\sqrt N$ floor) follows $\lfloor\log_b(\sqrt N/5)\rfloor$ exactly, including the
integer staircase --- the $D^*=\Theta(\log N/\log b)$ ceiling of
Corollary~\ref{cor:main-dstar}.}
\label{fig:depthlaw}
\end{figure}

A second phenomenon goes beyond attenuation.
\begin{proposition}[Aliasing; Appendix~\ref{sec:depth}, Proposition~\ref{prop:alias}]
\label{prop:main-alias}
With roles shared across levels, sign roles are self-inverse ($r\odot r=\mathbf 1$), so
a leaf's coefficient in the root expansion depends only on the \emph{parity profile} of
its role chain --- which roles appear an odd number of times along its path. Leaves with
equal parity profiles are exactly interchangeable: for unnormalized codes the swap
margin is identically zero at any dimension $N$, with no orthogonality used; under
per-node normalization the residue is crosstalk-sized, carrying no $2b^{-D}$ term.
\end{proposition}
The simulation shows exactly this: aliased far swaps sit two to five orders of magnitude
below the law, while fresh per-level roles restore $m\cdot b^D\approx2.00$ on the same
configurations. Flat self-inverse binding does not merely attenuate deep structure ---
it aliases some of it.

\emph{Scope.} As a no-go this is class-relative: it binds the recursive role-binding
class, not arbitrary encoders --- an unconstrained encoder can look
up any finite set of deep captions. The escape routes are structured (tree-side) scoring
and chunking with cleanup, exactly the mechanisms of the VSA literature
\cite{plate,frady}; on the vision--language side, text encoders that score along the
syntactic derivation are a concrete instance of the tree-side route \cite{discoclip}.

\section{What the objective pays for binding}\label{sec:mainthrottle}

The second obstruction is about the training signal, and is architecture-free. A
\emph{contrastive sample} is a positive $(I_s,c(s))$ together with a negative caption
$c^-$: with probability $\pi$ the \emph{explicit swap} $c(\sigma s)$, otherwise the
caption of an independent scene. Writing $[c]$ for the \emph{symbol class} of a caption
--- all captions sharing its symbol multiset --- the \emph{implicit collision rate}
$\pi_{\mathrm{impl}}$ is the chance a random negative lands in the positive's class;
at web scale both rates are vanishing. The loss is the one-negative logistic contrastive
loss, over \emph{arbitrary measurable scorers} $S$ --- no architecture assumption, so
every bound applies to dual encoders a fortiori:
\[
L(S)\;=\;\mathbb E\,\ell\big(S(I_s,c(s))-S(I_s,c^-)\big),
\qquad \ell(m)=\log(1+e^{-m}).
\]
Two comparators attach to any scorer: its \emph{pairing-blind average} $\bar S$
(replace $S(I,c)$ by its average over $[c]$ --- a bag-of-words scorer) and its
\emph{reflection} $\rho S$ (score every caption as if swapped --- binding exactly
reversed); $M_S(s)$ is the scorer's swap margin. Two regularity parameters:
\emph{$\Delta$-separation} --- cross-class contrasts are easy,
$S(I_s,c(s))-S(I_s,c')\ge\Delta$ whenever $[c']\ne[c(s)]$ --- and \emph{within-class
spread} $R$: $|S(I,c)-S(I,c')|\le R$ for $c'\in[c]$. One structural fact matters for
interpretation: with unrestricted capacity the population optimum \emph{does} bind
(Appendix~\ref{sec:throttle}); the theorems say its advantage over not binding is a
sliver. We prove:

\begin{theorem}[The throttle; Appendix~\ref{sec:throttle}, Theorem~\ref{thm:T1} and
Corollary~\ref{cor:throttle}]\label{thm:main-throttle}
$L(\bar S)-L(S)\le(\pi+\pi_{\mathrm{impl}})\log 2+e^{-(\Delta-2R)}$. Consequently the
pairing-blind competitor is $\varepsilon$-optimal with
$\varepsilon\le(\pi+\pi_{\mathrm{impl}})\log2$ plus an exponential tail: \emph{the
objective's entire reward for binding --- everything beyond bag-of-words --- is confined
to a sliver of that width.} No optimization guarantee weaker than the sliver can imply
binding.
\end{theorem}

\begin{theorem}[Anti-binding is cheap; Appendix~\ref{sec:throttle},
Theorem~\ref{thm:T2}]\label{thm:main-antibind}
$\bigl|L(\rho S)-L(S)-\pi\,\mathbb E_s[M_S(s)]\bigr|
\le\pi_{\mathrm{impl}}R+2e^{-(\Delta-2R)}$: exactly reversed binding costs precisely
$\pi\,\mathbb E[M]$.
\end{theorem}

Both constants are verified to within $0.006$ in controlled simulation
(Appendix~\ref{sec:throttle}), and three phenomena measured in the companion study
\cite{paper1} become corollaries: the
dose--response of binding in $\pi$ (linear at small $\pi$, slope $\mathbb E[M]$); the equivalence of
up-weighting matched terms and raising $\pi$; and \emph{cheap below-chance binding} ---
the reversed scorer sits only $\pi\,\mathbb E[M]$ away in loss, so small perturbations
(the pooled code's cross-term interference --- \emph{leakage},
Appendix~\ref{sec:thmC} --- or implicit bias) suffice to select it, consistent with the backwards-binding
observed in learned pooled codes (Appendix~\ref{sec:thmC}).

\section{The smoothness--binding frontier, and where deployed models sit}
\label{sec:mainfrontier}

The third obstruction trades binding against the \emph{similarity} job that deployment
demands: paraphrases and order variants of a caption must embed close together. The two
demands collide at the vocabulary level: an order variant of $c(s)$ and the swap
$c(\sigma s)$ contain \emph{identical words}, so any word-multiset-invariant text
encoder performs the required merge and the forbidden one simultaneously. Formally:
\begin{theorem}[Pairing-blind collapse; Appendix~\ref{sec:thmB}, Theorem~\ref{thm:B}]
\label{thm:main-blind}
If the caption map depends only on the symbol multiset, then
$\txt(c(\sigma s))=\txt(c(s))$ and $M(s)=0$ identically: the swap instances of
$(C_\gamma)$ fail for every $\gamma>0$ and every image map, and swap benchmark accuracy
is chance in expectation under any label-independent tie-breaking.
\end{theorem}
Quantitatively, say the
caption map is $\delta$-\emph{smooth} at $s$ if both $\txt(c(s))$ and
$\txt(c(\sigma s))$ have inner product at least $1-\delta$ with some \emph{pairing-neutral
anchor} $\bar t_s$ --- a fixed vector computable from the symbol multiset alone, e.g.\
the embedding of the \emph{atom-list caption}, which lists the objects and attributes
without pairing them (the anchor quantifier is what guards the bound
from tautology; Appendix~\ref{sec:frontier}).

\begin{theorem}[Smoothness--binding frontier; Appendix~\ref{sec:frontier},
Theorem~\ref{thm:E}]\label{thm:main-frontier}
If the caption map is $\delta$-smooth at $s$, then for every image map, every dimension,
and every training procedure,
\[
M(s)\;\le\;\norm{\txt(c(s))-\txt(c(\sigma s))}\;\le\;2\sqrt{2\delta},
\]
and the bound is tight: unit configurations achieve $2\sqrt{2\delta}\,(1-O(\delta))$,
so the rate $\sqrt\delta$ and the constant $2\sqrt2$ are exact.
\end{theorem}

The anchor-free inner inequality $|M(s)|\le d(s):=\norm{\txt(c(s))-\txt(c(\sigma s))}$
is a per-item ceiling, and the theorem induces a text-only diagnostic: measure
$d$, the smoothness gap $\delta$ realized against the atom-list anchor, the ceiling
$\mathrm{cap}=2\sqrt{2\delta}$, and the \emph{frontier usage} $d/\mathrm{cap}\in[0,1]$.
Usage near $1$ means binding is smoothness-limited --- further gains must be paid for in
paraphrase geometry; usage far below $1$ means the ceiling is not the binding
constraint.

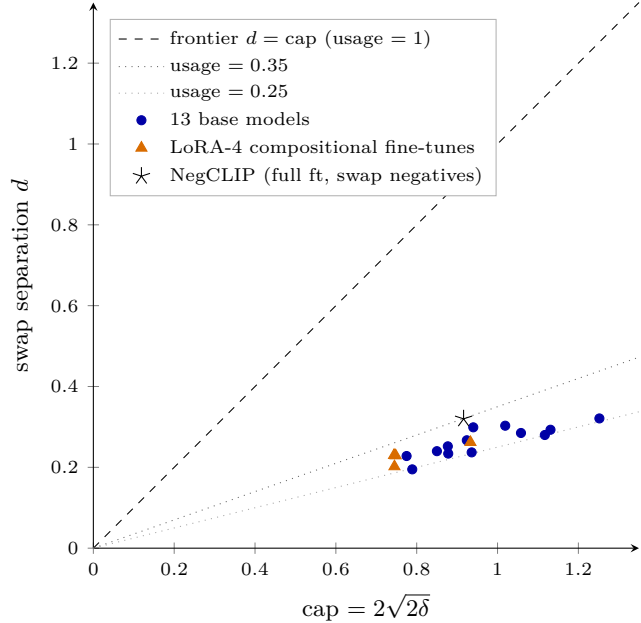
\begin{figure}[t]
\centering
\begin{tikzpicture}
\begin{axis}[width=8.8cm, height=8.8cm, axis lines=left,
  xlabel={cap $=2\sqrt{2\delta}$}, ylabel={swap separation $d$},
  xmin=0, xmax=1.35, ymin=0, ymax=1.35,
  tick label style={font=\scriptsize}, label style={font=\small},
  legend style={font=\scriptsize, at={(0.03,0.97)}, anchor=north west, draw=black!30},
  legend cell align=left]
\addplot[dashed, black] coordinates {(0,0) (1.35,1.35)};
\addlegendentry{frontier $d=\mathrm{cap}$ (usage $=1$)}
\addplot[dotted, black!60] coordinates {(0,0) (1.35,0.4725)};
\addlegendentry{usage $=0.35$}
\addplot[dotted, black!35] coordinates {(0,0) (1.35,0.3375)};
\addlegendentry{usage $=0.25$}
\addplot[only marks, mark=*, mark size=1.7pt, blue!65!black] coordinates
  {(0.789,0.195) (1.117,0.280) (0.936,0.237) (1.252,0.321) (1.131,0.293)
   (0.878,0.234) (1.058,0.285) (0.850,0.240) (0.877,0.252) (0.924,0.267)
   (0.775,0.228) (1.019,0.303) (0.940,0.299)};
\addlegendentry{13 base models}
\addplot[only marks, mark=triangle*, mark size=2.4pt, orange!85!black] coordinates
  {(0.745,0.202) (0.933,0.262) (0.748,0.229) (0.744,0.230)};
\addlegendentry{LoRA-4 compositional fine-tunes}
\addplot[only marks, mark=star, mark size=3.4pt, black] coordinates {(0.916,0.320)};
\addlegendentry{NegCLIP (full ft, swap negatives)}
\end{axis}
\end{tikzpicture}
\caption{The model zoo against the frontier (the data of Table~\ref{tab:zoo};
axes to equal scale). Every deployed model sits in the narrow wedge between usage $\approx0.25$ and $0.35$,
far below the $d=\mathrm{cap}$ line: no model is smoothness-limited. Scale, data, and
objective move points \emph{along} the wedge; only NegCLIP (star) climbs \emph{across}
it, by raising $d$ at roughly constant cap.}
\label{fig:zoo}
\end{figure}

\begin{table}[t]
\caption{The frontier diagnostic across the text encoders of $18$ deployed models:
swap separation $d$, smoothness gap $\delta$, ceiling $\mathrm{cap}=2\sqrt{2\delta}$,
usage $=d/\mathrm{cap}$, and the largest per-quartet usage; rows sorted by usage.
Model references: \cite{clip,openclip,align,siglip,siglip2,metaclip,aro,tsvlc,dac}.
Protocol,
abbreviations, and rigor checks in Appendix~\ref{sec:zoo}.}
\label{tab:zoo}
\centering
\footnotesize
\begin{adjustbox}{max width=\linewidth}
\begin{tabular}{llccccc}
\toprule
Model & Family / objective & $d$ & $\delta$ & cap & usage & max \\
\midrule
ALIGN-base        & ALIGN (BERT text)     & 0.195 & 0.080 & 0.789 & 0.248 & 0.54\\
CLIP L/14-336     & OpenAI CLIP           & 0.280 & 0.158 & 1.117 & 0.251 & 0.39\\
CLIP B/32         & OpenAI CLIP           & 0.237 & 0.112 & 0.936 & 0.253 & 0.42\\
SigLIP-large      & SigLIP (sigmoid)      & 0.321 & 0.199 & 1.252 & 0.256 & 0.55\\
CLIP L/14         & OpenAI CLIP           & 0.293 & 0.162 & 1.131 & 0.259 & 0.41\\
MetaCLIP B/16     & MetaCLIP              & 0.234 & 0.098 & 0.878 & 0.267 & 0.50\\
OpenCLIP H/14     & LAION-2B              & 0.285 & 0.141 & 1.058 & 0.270 & 0.46\\
SVLC LLM+RB       & CC3M, LoRA-4 ft       & 0.202 & 0.070 & 0.745 & 0.270 & 0.47\\
DAC-LLM           & CC3M, LoRA-4 ft       & 0.262 & 0.110 & 0.933 & 0.281 & 0.48\\
CLIP B/16         & OpenAI CLIP           & 0.240 & 0.091 & 0.850 & 0.283 & 0.49\\
OpenCLIP B/32     & LAION-2B              & 0.252 & 0.097 & 0.877 & 0.287 & 0.52\\
OpenCLIP L/14     & LAION-2B              & 0.267 & 0.108 & 0.924 & 0.289 & 0.44\\
MetaCLIP B/32     & MetaCLIP              & 0.228 & 0.076 & 0.775 & 0.294 & 0.51\\
SigLIP-base       & SigLIP (sigmoid)      & 0.303 & 0.132 & 1.019 & 0.297 & 0.55\\
SVLC RB           & CC3M, LoRA-4 ft       & 0.229 & 0.071 & 0.748 & 0.306 & 0.51\\
DAC-SAM           & CC3M, LoRA-4 ft       & 0.230 & 0.070 & 0.744 & 0.309 & 0.47\\
SigLIP2-base      & SigLIP2 (sigmoid)     & 0.299 & 0.112 & 0.940 & 0.318 & 0.64\\
\textbf{NegCLIP}  & COCO full ft, swap negs & \textbf{0.320} & 0.107 & 0.916 & \textbf{0.349} & 0.57\\
\midrule
\multicolumn{2}{l}{13 base models: mean 0.275, median 0.270} & & & & 0.248--0.318 & \\
\bottomrule
\end{tabular}
\end{adjustbox}
\end{table}

We measured the diagnostic on the text encoders of $18$ deployed models --- CLIP
variants across architectures and pretraining corpora, SigLIP \cite{siglip}, ALIGN
\cite{align}, and the compositionality fine-tunes NegCLIP \cite{aro}, SVLC \cite{tsvlc}, and DAC \cite{dac}
--- on a fixed grid of $675$ \emph{quartets}: a quartet fixes two objects and two color
attributes, determining a scene and its swap; every model receives the identical grid
and anchor template (protocol and rigor checks in
Appendices~\ref{sec:zoo}--\ref{sec:bench}). Table~\ref{tab:zoo} reports the grid
aggregates, and Figure~\ref{fig:zoo} plots them: every model sits at \emph{aggregate
usage $0.248$--$0.349$}, a factor of three or more below the ceiling for all but
NegCLIP (which comes closest, at $2.9\times$) --- so binding
failure in the deployed models is \emph{not} a smoothness limit; by the throttle and
the code-structure results it is an incentive and code limit. How models move within the
figure's wedge is informative: three of the four LoRA-rank-4 compositional fine-tunes move \emph{along} the wedge
(spending smoothness, not gaining separation; DAC-LLM raises $d$ modestly), while NegCLIP --- a full fine-tune with
caption-level swap negatives, i.e.\ raised $\pi$ --- is the only model that climbs
\emph{across} it (usage $0.349$). And the per-item ceiling is predictive on a real
benchmark: across $14$ models $\times$ $4{,}757$ SugarCrepe items ($66{,}598$
decisions), subset-level mean $d$ tracks accuracy at $r=0.99$, with the swap subsets
lowest exactly as the theory orders them (Appendix~\ref{sec:bench}).

\section{What does not explain the failure}\label{sec:mainnot}

Two natural explanations are eliminated. \emph{Dimension.} The axioms are satisfiable,
and the threshold is known exactly:
\begin{proposition}[Dimension dichotomy; Appendix~\ref{sec:possibility},
Proposition~\ref{prop:dichotomy}]\label{prop:main-dichotomy}
For a concept system with at least $N+2$ objects, the constituent-ranking axioms for
all scene sizes $j\le k$
are satisfiable on $\mathbb S^{N-1}$ iff $k\le\lfloor N/2\rfloor$, and in that regime
placements exist satisfying margin retrieval $(C_\gamma)$ as well, for some $\gamma>0$.
\end{proposition}
The mathematics is published, twice independently \cite{msp,med}; random factored sign
codes realize the possibility direction cheaply at deployed scales
(Appendix~\ref{sec:possibility}). So no dimension bound at deployed scales explains
swap failure. \emph{Exposure-graph connectivity.} The \emph{exposure graph} is the
bipartite graph recording which object--attribute pairs occur in training; it is
tempting to conjecture that binding cannot generalize across its disconnected
components (the gauge ambiguity of rank-one completion \cite{kiraly} --- codes are
identified only up to a per-component sign --- supports this for learners that commit
to a sign). The conjecture is false as an information bound: a sign-agnostic matcher,
scoring candidates by the \emph{absolute value} of the inner product, achieves accuracy
$1.000$ on cross-component swap tests even with zero same-component test pairs;
coverage --- how many pairs have been seen --- not connectivity, is the threshold
(Appendix~\ref{sec:notlimit}). Learners that do commit to a per-component sign drop
below chance ($0.27$--$0.37$) --- matching the below-chance binding that the throttle
predicts is cheap.

\section{Related work}\label{sec:mainrelated}

Kang et al.~\cite{kang} initiated the ideal-encoder program; we formalize and extend
it (Appendix~\ref{sec:setup}). Dimension/capacity limits for retrieval are characterized in the LIMIT line
\cite{limit} and its possibility counterpart \cite{med} --- an orthogonal axis: those bounds vanish as $N$ grows,
while the frontier of \S\ref{sec:mainfrontier} does not involve $N$, and the depth
ceiling of \S\ref{sec:maindepth} \emph{grows} only logarithmically in it. Finite
semantic resolution limits identification under non-compositional representations
\cite{semanticity}, without swap/margin content; Chen et al.~\cite{chen} prove
existence of swap-insensitive contrastive optima (existence only --- our throttle
quantifies the price and its reversal); SugarCrepe++ \cite{sugarcrepepp} documents the
paraphrase/hard-negative dissociation that Theorem~\ref{thm:main-frontier} formalizes.
For depth, the sequence-memory theory of Frady et al.~\cite{frady} supplies closely
related crosstalk accounting, but no per-level margin law and no discrimination
between a structure and its minimal edit --- the gap our Theorems~\ref{thm:main-depthlaw}--\ref{thm:main-finiteN}
close; Plate's classical discussion \cite{plate} is qualitative and centers the
chunking escape. Compositional-generalization theory
\cite{lubana,schug,abbe} concerns learners and data; our connectivity refutation
(\S\ref{sec:mainnot}) sharpens which graph property carries the information. On the
constructive side, compositional tensor-network encoders for vision--language ---
scoring along the syntactic derivation \cite{discoclip}, and tensor-based
compositional semantics for concept generalization \cite{vqccg} --- implement the
structured escape that the depth analysis motivates. Full
positioning is in Appendix~\ref{sec:synthesis}.

\section{Discussion}\label{sec:maindiscussion}

\paragraph{What the results support.} The contingent causes of today's failures --- vanishing $\pi$ (the throttle), the pooled
code's leakage (Appendix~\ref{sec:thmC}), and the training distribution --- are all
fixable in principle, and the
zoo measurements show the fixes moving models exactly as predicted (only raised $\pi$
climbs the wedge). Behind them stand limits that survive all fixes: the
smoothness--binding frontier (with measured slack today), and the depth ceiling
$D^{*}=\Theta(\log N/\log b)$, which sits at natural-language depth for CLIP-sized
models and binds the recursive role-binding class no matter the training.

\paragraph{Limitations.} Our no-gos are class- or hypothesis-relative by necessity ---
the possibility theorem forbids anything stronger. The depth theorem's finite-$N$ form
rests on one explicitly flagged concentration estimate (its expectation form is proved,
and the multiplicative-error signature is confirmed in simulation); the frontier
measurement instantiates two-object color binding, not relations or counting; and the
NegCLIP comparison carries a three-way confound (full fine-tune, corpus, negatives)
that a rank-matched control would break. Open problems, including the margin-robust
dimension question and the naturalistic frontier, are collected in
Appendices~\ref{sec:synthesis} and~\ref{sec:map2}.

\paragraph{Reproducibility.} Every measured claim carries its protocol in the
appendices; measurement scripts and raw results are available from the author on
request.

\clearpage
\appendix
\part*{Appendix}
\noindent The appendices are a self-contained technical development: complete
definitions and axioms, all proofs, measurement protocols with their raw numbers, and
negative results. Results stated in the main text reappear here inside their full
development, with the same notation. One proof gap is left deliberately explicit: the
concentration estimate inside Lemma~\ref{lem:ellfour} (Appendix~\ref{sec:depth}), which
is the single unproved step behind Theorem~\ref{thm:depthB}.
\medskip

\section{Notation and glossary}\label{sec:glossary}

Throughout, $\Sph=\{x\in\R^N:\norm{x}=1\}$ is the unit sphere in $\R^N$; $\ip{\cdot}{\cdot}$
is the Euclidean inner product and $\norm{\cdot}$ the Euclidean norm. For unit vectors the
inner product equals the cosine similarity. We use $c$, $c_1$, $c_2$ for absolute constants
and $O(\cdot)$, $\Theta(\cdot)$ with their usual meaning.

\begin{center}
\small
\begin{adjustbox}{max width=\linewidth}
\begin{tabular}{llp{8.6cm}}
\toprule
Symbol & Name & Meaning (first defined in) \\
\midrule
$O,A$ & objects, attributes & finite sets, $|O|=M$, $|A|=K$ (\S\ref{sec:setup}) \\
$s$ & scene & a set of $k$ object--attribute pairs with distinct objects (\S\ref{sec:setup}) \\
$c(s),I_s$ & caption, image & canonical linguistic/visual realization of $s$ (\S\ref{sec:setup}) \\
$\mathcal C(s),\mathcal I(s)$ & realization classes & all captions / all depictions of $s$ (\S\ref{sec:setup}) \\
$\sigma s$ & swap & the scene with two attributes exchanged (\S\ref{sec:setup}) \\
$\mathcal N_{\mathrm{full}}(s),\mathcal N_1(s)$ & negative families & all other scenes' captions; hard core = one edit (replace / add--drop / swap) (\S\ref{sec:setup}) \\
$\img,\txt$ & encoders & image/text embedding maps into $\Sph$ (\S\ref{sec:setup}) \\
$M(s)$ & binding margin & $\ip{\img(I_s)}{\txt(c(s))-\txt(c(\sigma s))}$ (\S\ref{sec:setup}) \\
$d(s)$ & swap separation & $\norm{\txt(c(s))-\txt(c(\sigma s))}$ (\S\ref{sec:frontier}) \\
$\bar t_s$ & pairing-neutral anchor & a unit vector depending only on the symbol multiset of $s$ (\S\ref{sec:frontier}) \\
$\delta$ & smoothness gap & both pairings within cosine $1-\delta$ of $\bar t_s$ (\S\ref{sec:frontier}) \\
$\mathrm{cap}$ & frontier ceiling & $2\sqrt{2\delta}$ (\S\ref{sec:frontier}) \\
usage & frontier position & $d/\mathrm{cap}\in[0,1]$ (\S\ref{sec:frontier}) \\
$g,\ell,\ell_0,L$ & gain, leakage, cross-talk & code-quality functionals for pooled codes (\S\ref{sec:thmC}) \\
$\pi$ & swap-negative rate & probability the training negative is a swap (\S\ref{sec:thmD}) \\
\bottomrule
\end{tabular}
\end{adjustbox}
\end{center}

\section{The formal setup}\label{sec:setup}

\subsection{Concept systems, scenes, and negatives}

\begin{definition}[Concept system and scenes]\label{def:scene}
Fix finite sets $O$ (\emph{objects}, $|O|=M$) and $A$ (\emph{attributes}, $|A|=K$).
An \emph{atom} is a pair $(o,a)\in O\times A$, read ``object $o$ has attribute $a$.''
A \emph{$k$-scene} is a set
$s=\{(o_1,a_1),\dots,(o_k,a_k)\}$
with the $o_j$ pairwise distinct. Where a swap is discussed we additionally require the
$a_j$ pairwise distinct (otherwise the swap can be the identity). We write $V(s)=\{o_1,\dots,o_k\}$
for the object set of $s$. Each scene has a \emph{caption} $c(s)$ (a linguistic description
asserting exactly the atoms of $s$) and an \emph{image} $I_s$ (a depiction of exactly those
atoms). Example with $k=2$: $s=\{(\text{car},\text{red}),(\text{dog},\text{blue})\}$,
$c(s)=$ ``a red car and a blue dog.'' The \emph{symbol multiset} of $s$ is the multiset of
object and attribute symbols occurring in $s$ \emph{with the pairing forgotten} (here:
$\{$car, dog, red, blue$\}$). Atoms are bound pairs, so a scene and its swap (below) have
\emph{different} atom sets but the \emph{same} symbol multiset; captions realize the symbol
multiset as a word multiset.
\end{definition}

\begin{remark}[Why $c(s)$ but $I_s$]\label{rem:notation}
The asymmetric notation is deliberate. The caption constructor is used \emph{as a function}
throughout: the axioms apply it to \emph{edited} scenes ($c(\sigma s)$ below, and every
hard negative is the caption of an edited scene)---including scenes that are never
depicted; a false caption needs no image to exist. The image realization, by contrast, is
never applied to an edited scene: no axiom will mention $I_{\sigma s}$. So $c(\cdot)$ is an
operator composed with scene transformations, while $I_s$ labels a given datum. This
mirrors the benchmarks themselves, whose negatives are synthesized on the text side.
\end{remark}

\begin{definition}[Swap and the negative families]\label{def:negatives}
For a scene $s$ with $k\ge 2$ and two chosen positions, the \emph{swap} $\sigma s$ is the
scene with the two attributes exchanged: from
$\{(o_1,a_1),(o_2,a_2),\dots\}$ to $\{(o_1,a_2),(o_2,a_1),\dots\}$ --- a single
\emph{transposition} of the pairing (general permutations are compositions of
transpositions; see Remark~\ref{rem:negatives}).
Note $c(s)$ and $c(\sigma s)$ contain the \emph{same multiset of words}; they differ only
in which attribute is bound to which object.

Two negative families are distinguished. The \emph{full family}
$\mathcal N_{\mathrm{full}}(s)$ consists of \emph{every caption that describes a scene
other than $s$} (formalized by realization classes in
Definition~\ref{def:realizations}). The \emph{hard core}
$\mathcal N_1(s)\subset\mathcal N_{\mathrm{full}}(s)$ consists of the captions of scenes
\emph{one edit away} from $s$:
(i) \emph{replace} one atom's object or attribute by one not occurring in $s$;
(ii) \emph{add} an atom whose object does not occur in $s$, or \emph{drop} one atom;
(iii) \emph{swap}: one transposition of the pairing.
The axiom $(C_\gamma)$ below quantifies over the full family; the hard core is the
binding-critical worst case and the part that benchmarks instantiate
(Remark~\ref{rem:negatives}).

The hard core maps onto the hard-negative benchmarks as follows.
SugarCrepe's seven subsets are exactly \textsc{replace}-object/attribute/relation,
\textsc{add}-object/attribute, and \textsc{swap}-object/attribute \cite{sugarcrepe}: its
taxonomy coincides with (i)--(iii) minus \emph{drop} (on which see
Remark~\ref{rem:negatives}). The swap family (iii) is the canonical content of ARO
\cite{aro}: VG-Attribution exchanges the two attributes between two objects, and
VG-Relation exchanges the two arguments of a relation. Winoground \cite{winoground}
realizes pairing permutations with \emph{identical word multisets} (each caption pair
contains the same words, differently bound)---i.e.\ swap-type negatives in our sense,
though a fraction of its items are known to require capabilities beyond pairing \cite{diwan}. One
benchmark family is deliberately \emph{absent} from $\mathcal N_{\mathrm{full}}(s)$: ARO's
COCO-Order/Flickr-Order subsets, whose negatives are ungrammatical word \emph{shuffles};
those are detectable from the text alone, with no image---the ``hackability'' SugarCrepe
was constructed to remove \cite{sugarcrepe}---and correspond to no scene transformation.
This identification is what makes the axioms below a formalization of ``benchmark
competence'' rather than an arbitrary list of geometric conditions.
\end{definition}

\begin{remark}[What exactly is a negative]\label{rem:negatives}
Four clarifications fix the ambiguities latent in ``negative.''
(1) \emph{Retrieval, not falsity.} $(C_\gamma)$ is an argmax-with-margin condition: given
$I_s$, the \emph{exact} description $c(s)$ must out-score every caption of every other
scene. Negatives need not be false in the natural-language sense; they need to describe a
different scene. This matters for \emph{drop}: under the closed-world convention of
Definition~\ref{def:scene} (captions assert their atoms \emph{exactly}), the dropped
caption ``a red car'' describes the one-object scene, which $I_s$ does not depict, so
preferring the full caption is a completeness-of-description demand. Under open-world
semantics a dropped caption is simply \emph{true}, which is why no benchmark tests drops
(SugarCrepe has no drop subset). Our zoo probe (\S\ref{sec:zoo}) uses swaps; the
benchmark protocol of \S\ref{sec:bench} uses the five swap/replace subsets (the two add
subsets would serve equally---their omission is scoping, not principle). Drop
is an axiom-internal negative, not an empirical one.
(2) \emph{Why the hard core is the right worst case.} Within the pooled class of
\S\ref{sec:thmC}, margins are monotone in edit distance: single-add
$\approx\tfrac1{2(k+1)}$ $<$ single-replace $\approx\tfrac1k$ $<$ swap $\approx\tfrac2k$
$<$ multi-edit negatives (an $m$-edit difference has $m$ matched terms and margin
$\approx m/k$ by an outlined extension of the route of Theorem~\ref{thm:Cswap}). The minimum over
$\mathcal N_{\mathrm{full}}$ is attained on $\mathcal N_1$, so for such codes proving the
hard core proves the family. \emph{Caveat stated plainly:} for \emph{arbitrary} encoders,
hard-core margin implies nothing about further captions --- which is exactly why the axiom
quantifies over $\mathcal N_{\mathrm{full}}$. No theorem is weakened: Theorems~\ref{thm:B}
and~\ref{thm:E} use only the swap instance, Theorem~\ref{thm:Cswap} proves the hard core
and inherits the rest by the monotonicity above, and the possibility construction of
Proposition~\ref{prop:dichotomy}(ii) satisfies the full family outright
($\ip{\img(I_s)}{\txt(c(s))}=1$ beats every distinct unit vector).
(3) \emph{Freshness.} Replacement atoms must not occur in $s$: for objects this is forced
(scenes have distinct objects), for attributes it preserves the distinct-attribute
convention.
(4) \emph{Swaps are transpositions.} A general permutation of the pairing rebinds $\ge3$
atoms and, in the pooled class, separates \emph{more} than a transposition
($\approx m/k$ for an $m$-cycle); the transposition is the binding-critical case, and it
is what ARO/SugarCrepe instantiate.
\end{remark}

\begin{definition}[Realization classes]\label{def:realizations}
A scene admits many linguistic and visual realizations. Write $\mathcal C(s)$ for the set
of captions asserting exactly the atoms of $s$---order variants (``a red car and a blue
dog'' / ``a blue dog and a red car''), determiner and phrasing variants, paraphrases---and
$\mathcal I(s)$ for the set of depictions of exactly $s$. Since captions assert their atoms
exactly, the classes $\mathcal C(s)$ are pairwise disjoint across scenes; in particular the
swap caption is \emph{not} an order variant: it lies in the different class
$\mathcal C(\sigma s)$. We fix canonical representatives $c(s)\in\mathcal C(s)$ and
$I_s\in\mathcal I(s)$, and adopt the convention that every axiom instance quantifies over
realizations: $(C_\gamma)$ reads ``for all $I\in\mathcal I(s)$, all
$c^{+}\in\mathcal C(s)$, and all $c'\in\mathcal N_{\mathrm{full}}(s):=\bigcup_{s'\ne s}\mathcal C(s')$'' (the
hard core $\mathcal N_1(s)$ is the union over one-edit scenes only). Every theorem below holds verbatim for
each choice of realizations---each proof manipulates one caption pair and one image at a
time---so statements are written with representatives; the class structure is revisited
where it carries content (Remark~\ref{rem:invariance},
Corollary~\ref{cor:classfrontier}). Image multiplicity, in particular, changes nothing
anywhere: all bounds are per-image, and the constructions place images inside open cones
with room for any number of them.
\end{definition}

\begin{remark}[Required vs.\ forbidden invariance collide on the vocabulary]
\label{rem:invariance}
The similarity job (a deployment demand, formalized in \S\ref{sec:frontier}; not an axiom of Definition~\ref{def:axioms}) requires $\txt$ to (approximately) identify $\mathcal C(s)$---in
particular its order variants; the binding job requires $\txt$ to separate $\mathcal C(s)$
from $\mathcal C(\sigma s)$. At the level of \emph{word multisets} these demands are
indistinguishable: the order variant ``a blue dog and a red car''
$\in\mathcal C(s)$ and the swap ``a blue car and a red dog'' $\in\mathcal C(\sigma s)$
contain the identical words; they differ only in syntactic attachment. Consequently
(a) any word-multiset-invariant (bag-of-words) encoder performs the required conflation
and the forbidden one \emph{simultaneously}---Theorem~\ref{thm:B} is the formal statement;
(b) any encoder that separates swaps must compute something about the \emph{parse}, not
the vocabulary. This is the sharpest form we know of ``binding requires syntax,'' and it
is visible only once realization classes are explicit; Figure~\ref{fig:parses} draws the collision.
\end{remark}

\subsection{Dual encoders and the axioms}

\begin{definition}[Dual encoder]\label{def:encoder}
A \emph{dual encoder} is a pair of injective maps
$\img:\{\text{images}\}\to\Sph$, $\txt:\{\text{captions}\}\to\Sph$.
The \emph{score} of an image--caption pair is $\ip{\img(I)}{\txt(c)}$ (cosine similarity;
this is the deployed scoring rule of CLIP-class models). We first treat $\img,\txt$ as
\emph{placements}---arbitrary assignments of unit vectors, the ``oracle'' reading of
Kang et al.~\cite{kang}---and only in \S\ref{sec:thmD} as maps produced by training.
\end{definition}

\begin{definition}[The axioms]\label{def:axioms}
Fix a margin $\gamma>0$ and a maximal scene size $k_{\max}$. Convention: throughout the
document $k$ denotes the size of the scene at hand ($k=|s|$); the axioms quantify over all
scenes with $k\le k_{\max}$.
\begin{itemize}
\item[$(C_\gamma)$] \textbf{Margin retrieval.} For every scene $s$ with at most $k_{\max}$ objects
and every $c'\in\mathcal N_{\mathrm{full}}(s)$, the full negative family of
Definition~\ref{def:negatives}:
\[
\ip{\img(I_s)}{\txt(c(s))}\;\ge\;\ip{\img(I_s)}{\txt(c')}+\gamma .
\]
\item[$(K_k)$] \textbf{Constituent ranking.} For every $k$-scene $s$, every $x\in V(s)$ and
every $z\in O\setminus V(s)$:
\[
\ip{\img(I_s)}{\txt(x)}\;>\;\ip{\img(I_s)}{\txt(z)},
\]
where $\txt(x)$ denotes the embedding of the single-object caption for $x$ (the caption domain includes such object-only captions; $\txt$ is defined on all of them).
\end{itemize}
\end{definition}

$(C_\gamma)$ says the image of a scene prefers its own caption to every hard negative, with
margin. $(K_k)$ says the image of a scene ranks all objects \emph{present} above all objects
\emph{absent}---the ``does the embedding know what is in the picture'' condition.

\begin{definition}[Binding margin]\label{def:margin}
For a scene $s$ and its swap, the \emph{binding margin} is
\[
M(s)\;:=\;\ip{\img(I_s)}{\,\txt(c(s))-\txt(c(\sigma s))\,}.
\]
The swap instance of $(C_\gamma)$ is precisely $M(s)\ge\gamma$. Benchmark accuracy on a swap
item is $\mathbf 1[M(s)>0]$; $M(s)$ is the quantity every theorem below controls.
Figure~\ref{fig:setup} shows the whole pipeline on the running example.
\end{definition}

\begin{figure}[t]
\centering
\begin{adjustbox}{max width=\linewidth}\begin{tikzpicture}[font=\small,
  box/.style={draw=black!60, rounded corners=2pt, inner sep=4pt, align=center},
  enc/.style={draw=black!70, fill=black!8, rounded corners=3pt, inner sep=6pt, align=center},
  vec/.style={draw=black!60, fill=blue!5, rounded corners=2pt, inner sep=4pt, align=center},
  arr/.style={-{Stealth[length=2.2mm]}, black!70}]
\node[box, minimum width=32mm, minimum height=17mm] (img) at (0,0) {};
\node[font=\LARGE, text=red!75!black]  at (-0.55,0.18) {\faCarIcon};
\node[font=\scriptsize] at (-0.55,-0.32) {red car};
\node[font=\LARGE, text=blue!65!black] at (0.62,0.18) {\faDogIcon};
\node[font=\scriptsize] at (0.62,-0.32) {blue dog};
\node[font=\scriptsize, anchor=north west, inner sep=2pt] at (img.north west) {$I_s$};
\node[box] (cap1) at (0,-1.85)
  {``a \textcolor{red!75!black}{red} car and a \textcolor{blue!65!black}{blue} dog''};
\node[font=\scriptsize, anchor=west] at (2.38,-1.85) {$=c(s)$};
\node[box] (cap2) at (0,-3.05)
  {``a \textcolor{blue!65!black}{blue} car and a \textcolor{red!75!black}{red} dog''};
\node[font=\scriptsize, anchor=west] at (2.38,-3.05) {$=c(\sigma s)$};
\node[enc] (ienc) at (5.1,0)     {image\\encoder $\img$};
\node[enc] (tenc) at (5.1,-2.45) {text\\encoder $\txt$};
\node[vec] (iv)  at (8.3,0)     {$\img(I_s)\in\Sph$};
\node[vec] (tv1) at (8.3,-1.85) {$\txt(c(s))$};
\node[vec] (tv2) at (8.3,-3.05) {$\txt(c(\sigma s))$};
\draw[arr] (img)       -- (ienc);
\draw[arr] (ienc)      -- (iv);
\draw[arr] (cap1.east) -- (tenc.170);
\draw[arr] (cap2.east) -- (tenc.190);
\draw[arr] (tenc.10)   -- (tv1.west);
\draw[arr] (tenc.-10)  -- (tv2.west);
\node[align=center, font=\small] at (4.3,-4.35)
  {binding margin \; $M(s)=\ip{\img(I_s)}{\txt(c(s))-\txt(c(\sigma s))}$;
   \quad the benchmark item is correct $\iff M(s)>0$.};
\end{tikzpicture}\end{adjustbox}
\caption{The dual-encoder binding problem (Definitions~\ref{def:scene}--\ref{def:margin}).
One image of the scene $s$; two captions with the \emph{identical word multiset}, differing
only in which attribute is bound to which object. Both towers embed into the unit sphere,
and the deployed score is a single inner product. Every theorem in this paper is a
statement about the margin $M(s)$.}
\label{fig:setup}
\end{figure}

\begin{figure}[t]
\centering
\begin{adjustbox}{max width=\linewidth}\begin{tikzpicture}[font=\small,
  bag/.style={draw=black!55, fill=black!4, rounded corners=3pt, inner sep=6pt, align=center},
  surf/.style={rounded corners=2pt, inner sep=2pt, align=center, font=\scriptsize},
  neu/.style={black!45},
  arr/.style={-{Stealth[length=2.0mm]}, black!45}]

\node[bag] (bag) at (0,0.15)
  {word multiset\\[2pt]
   $\{$\,\textcolor{red!75!black}{red},\ \textcolor{blue!65!black}{blue},\ car,\ dog\,$\}$\\[4pt]
   {\LARGE\textcolor{red!75!black}{\faCarIcon}\hspace{5pt}\textcolor{blue!65!black}{\faDogIcon}}};

\node[draw=blue!65!black, fill=blue!65!black!6, rounded corners=3pt, inner sep=6pt] (lp) at (-5.0,0.15) {%
  \begin{adjustbox}{max width=\linewidth}\begin{tikzpicture}[font=\normalsize, baseline]
    \node[text=red!75!black]  (lr)   at (0,0.62)   {red};
    \node[text=blue!65!black] (lb)   at (1.6,0.62) {blue};
    \node                     (lcar) at (0,-0.62)  {car};
    \node                     (ldog) at (1.6,-0.62){dog};
    \draw[red!75!black, line width=1.1pt] (lr) -- (lcar);
    \draw[blue!65!black, line width=1.1pt] (lb) -- (ldog);
    \foreach \n in {lr,lb,lcar,ldog}{\fill[black!55] (\n) circle (0.6pt);}
  \end{tikzpicture}\end{adjustbox}};
\node[font=\scriptsize, blue!65!black, anchor=south] at (lp.north) {parse of $c(s)$};

\node[surf] (ls1) at (-5.0,-1.85) {``a \textcolor{red!75!black}{red} car and a \textcolor{blue!65!black}{blue} dog''};
\node[surf] (ls2) at (-5.0,-2.45) {``a \textcolor{blue!65!black}{blue} dog and a \textcolor{red!75!black}{red} car''};
\draw[arr, blue!65!black!65] (ls1.north) -- (lp.-65);
\draw[arr, blue!65!black!65] (ls2.north) .. controls (-4.2,-1.55) .. (lp.-45);
\draw[decorate, decoration={brace, amplitude=4pt}, blue!65!black]
  (-6.75,-1.62) -- (-6.75,-2.68);
\node[font=\scriptsize, blue!65!black, anchor=east, align=right] at (-6.85,-2.15)
  {similarity\\\emph{must merge}\\order variants};
\node[blue!65!black, font=\scriptsize, anchor=north] at (-5.0,-2.98) {class $\mathcal C(s)$};

\node[draw=orange!85!black, fill=orange!85!black!7, rounded corners=3pt, inner sep=6pt] (rp) at (5.0,0.15) {%
  \begin{adjustbox}{max width=\linewidth}\begin{tikzpicture}[font=\normalsize, baseline]
    \node[text=red!75!black]  (rr)   at (0,0.62)   {red};
    \node[text=blue!65!black] (rb)   at (1.6,0.62) {blue};
    \node                     (rcar) at (0,-0.62)  {car};
    \node                     (rdog) at (1.6,-0.62){dog};
    \draw[red!75!black, line width=1.1pt] (rr) -- (rdog);
    \draw[blue!65!black, line width=1.1pt] (rb) -- (rcar);
    \foreach \n in {rr,rb,rcar,rdog}{\fill[black!55] (\n) circle (0.6pt);}
  \end{tikzpicture}\end{adjustbox}};
\node[font=\scriptsize, orange!85!black, anchor=south] at (rp.north) {parse of swap $c(\sigma s)$};

\node[surf] (rs1) at (5.0,-1.85) {``a \textcolor{blue!65!black}{blue} car and a \textcolor{red!75!black}{red} dog''};
\draw[arr, orange!85!black!65] (rs1.north) -- (rp.-65);
\node[orange!85!black, font=\scriptsize, anchor=north] at (5.0,-2.98) {class $\mathcal C(\sigma s)$};

\draw[arr] (bag.west) -- node[above, font=\scriptsize, black!55]{same words} (lp.east);
\draw[arr] (bag.east) -- node[above, font=\scriptsize, black!55]{same words} (rp.west);

\draw[{Stealth[length=2mm]}-{Stealth[length=2mm]}, black!70, line width=0.8pt]
  (lp.south east) .. controls (-2.6,-1.35) and (2.6,-1.35) .. (rp.south west);
\node[font=\scriptsize, black!75, align=center, fill=white, inner sep=1pt] at (0,-1.32)
  {binding \emph{must separate}\\$\mathcal C(s)$ vs.\ $\mathcal C(\sigma s)$};

\node[draw=black!45, fill=black!5, rounded corners=3pt, inner sep=5pt,
      font=\scriptsize, align=center] (thm) at (0,-4.3)
  {both demands share the \emph{same bag} $\Rightarrow$ a word-multiset (bag-of-words) map cannot tell them apart\\
   \big(Theorem~\ref{thm:B}: swap margin $\equiv 0$\big)\ \ ---\ \ separating swaps requires reading the \textbf{parse}, not the words};

\end{tikzpicture}\end{adjustbox}
\caption{Same vocabulary, two parses (Remark~\ref{rem:invariance}). An order variant of
$c(s)$ and the swap $c(\sigma s)$ share the identical word multiset; they differ only in the
binding (which colour attaches to which object). The invariance similarity \emph{requires}
---merge the order variants of $\mathcal C(s)$ (left)---and the invariance binding
\emph{forbids}---separate the swap class $\mathcal C(\sigma s)$ (right)---are
indistinguishable at the vocabulary level. A bag-of-words map performs both at once
(Theorem~\ref{thm:B}); separating swaps requires reading the parse.}
\label{fig:parses}
\end{figure}
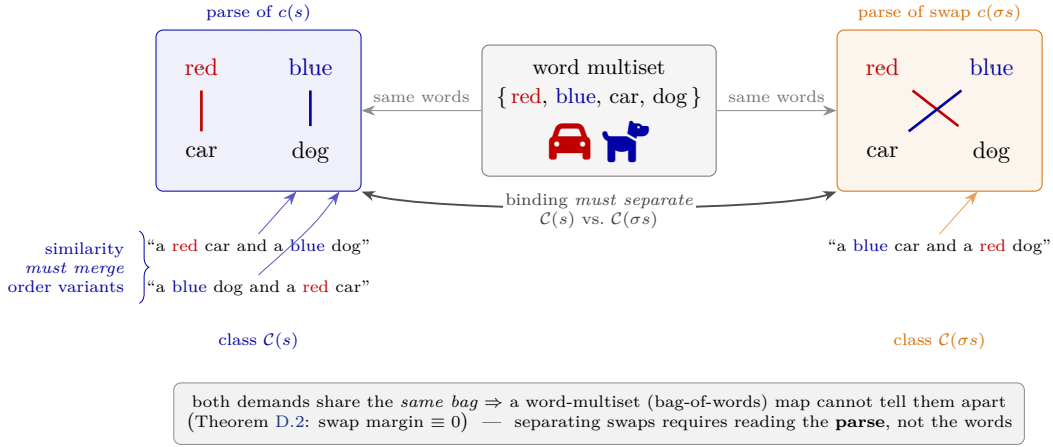

\begin{remark}[Relation to Kang et al.]\label{rem:kang}
Kang et al.~\cite{kang} formulate ideal-encoder conditions closely related to
Definition~\ref{def:axioms} and derive from them an impossibility claim for binding. Our
formulation differs in three presentational respects, discussed further in Paper~1
(\S3): every condition is quantified explicitly over scenes and realization classes; the
condition list is organized around the single margin family $(C_\gamma)$, since a margin
subsumes the various distinctness conditions in one clause; and a uniform convention for
strict versus weak inequalities is fixed, which matters at the collapse point. We retain
their Condition~1.1 as $(K_k)$, which carries the interesting geometry. Every
theorem below is proved from Definition~\ref{def:axioms} directly, so the development is
self-contained.
\end{remark}

\section{Possibility: the axioms have solutions}\label{sec:possibility}

Before proving impossibility theorems one must know what is possible; otherwise an
``impossibility'' may simply be hiding in an inconsistent axiom system. The content of this
section is a dichotomy for $(K_k)$, which turns out to be \emph{published}---twice,
independently---so we state it as a cited proposition, prove the short direction (the
impossibility) in full because it is used later, and outline the construction.
The section's purpose in the larger argument is its Corollary~\ref{cor:satisfiable}:
\emph{the full axiom system, including all swap conditions, is satisfiable with a positive
margin}. Hence every impossibility theorem in this paper must---and does---enter through
an explicitly added hypothesis.

\subsection{Convex-geometry vocabulary}

\begin{definition}\label{def:convex}
Let $X\subset\R^N$ be finite. $\conv(X)$ is the convex hull. A point $x\in X$ is
\emph{extreme} if $x\notin\conv(X\setminus\{x\})$. A subset $T\subseteq X$ is
\emph{(strictly linearly) separable} from $X\setminus T$ if there is $w\in\R^N$ and
$\theta\in\R$ with $\ip{w}{x}>\theta$ for $x\in T$ and $\ip{w}{x}<\theta$ for
$x\in X\setminus T$. A \emph{face} of the polytope $P=\conv(X)$ is the intersection of $P$
with a supporting hyperplane. $P$ is \emph{$k$-neighborly} if every subset of at most $k$
vertices spans a face. The \emph{trigonometric moment curve} in $\R^{2m}$ is
$\gamma(\theta)=(\cos\theta,\sin\theta,\cos 2\theta,\sin 2\theta,\dots,\cos m\theta,\sin m\theta)$;
it lies on the sphere of radius $\sqrt m$, and the convex hull of any finite set of distinct
points on it is a \emph{cyclic polytope}, which is $m$-neighborly \cite{gale}.
\end{definition}

One elementary fact we use twice:

\begin{lemma}[Sphere points are extreme]\label{lem:extreme}
Distinct points on $\Sph$ are extreme points of their convex hull.
\end{lemma}
\begin{proof}
If $x=\sum_i\lambda_i y_i$ with $y_i\in\Sph$, $\lambda_i>0$, $\sum\lambda_i=1$, then
$1=\norm{x}\le\sum_i\lambda_i\norm{y_i}=1$, with equality in the triangle inequality only
if all $y_i$ are positive multiples of one another, hence (being unit) all equal to $x$.
\end{proof}

\subsection{The dimension dichotomy}

Observe that $(K_k)$ constrains only the \emph{single-object} embeddings
$\{\txt(x):x\in O\}$: for the scene $s$, the linear functional $y\mapsto\ip{\img(I_s)}{y}$
must strictly separate $\{\txt(x):x\in V(s)\}$ from the embeddings of all absent objects.
So $(K_j)$ for all scenes of size $j$ requires: \emph{every $j$-subset of the $M$ points
$\{\txt(x)\}\subset\Sph$ is strictly separable from its complement.}
Figure~\ref{fig:separation} shows the demand on one scene; the dichotomy below asks for
\emph{which} $(N,k)$ it can hold for \emph{all} scenes at once.

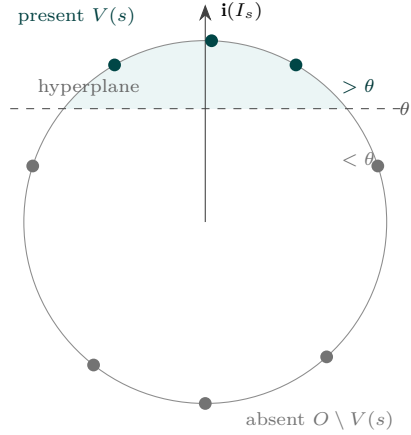
\begin{figure}[t]
\centering
\begin{tikzpicture}[scale=1.5, font=\small,
  arr/.style={-{Stealth[length=2.2mm]}, black!75}]
\coordinate (O) at (0,0);
\begin{scope}
  \clip (O) circle (1.6);
  \fill[teal!8] (-1.75,1.0) rectangle (1.75,1.72);
\end{scope}
\draw[black!45] (O) circle (1.6);
\draw[dashed, black!70] (-1.72,1.0) -- (1.72,1.0);
\node[font=\scriptsize, black!70, anchor=west] at (1.63,1.0) {$\theta$};
\node[font=\scriptsize, black!60, anchor=south west] at (-1.56,1.03) {hyperplane};
\draw[arr] (O) -- (0,1.92);
\node[font=\scriptsize, anchor=west] at (0.06,1.86) {$\img(I_s)$};
\foreach \a in {60,88,120} \fill[teal!55!black] (\a:1.6) circle (1.6pt);
\node[teal!50!black, font=\scriptsize] at (122:2.14) {present $V(s)$};
\node[teal!45!black, font=\scriptsize, anchor=west] at (1.12,1.20) {$>\theta$};
\foreach \a in {18,162,232,270,312} \fill[black!55] (\a:1.6) circle (1.6pt);
\node[black!55, font=\scriptsize] at (300:2.02) {absent $O\setminus V(s)$};
\node[black!55, font=\scriptsize, anchor=west] at (1.12,0.55) {$<\theta$};
\end{tikzpicture}
\caption{What $(K_k)$ demands, for one scene (the reduction opening this subsection).
The single-object embeddings $\{\txt(x)\}$ lie on the sphere; the scene's image
vector $\img(I_s)$ acts as the linear functional $y\mapsto\ip{\img(I_s)}{y}$---``height
along the arrow.'' Condition $(K_k)$ requires this functional to rank every
\emph{present} object $x\in V(s)$ above every \emph{absent} one, i.e.\ the dashed
hyperplane $\ip{\img(I_s)}{\cdot}=\theta$ strictly separates $\{\txt(x):x\in V(s)\}$ (teal,
in the shaded accept-region) from its complement (grey). So ``$(K_j)$ for all $j$-scenes''
is exactly ``every $j$-subset of the $M$ points is linearly separable''---a demand
Radon's theorem forbids once $j>\lfloor N/2\rfloor$ (Figure~\ref{fig:radon} is the same
picture when separation \emph{fails}).}
\label{fig:separation}
\end{figure}

\begin{proposition}[Dimension dichotomy]\label{prop:dichotomy}
Let $M\ge N+2$ and consider the axioms $(K_j)$ for all scene sizes $j\le k$ (the
\emph{cumulative} reading: a system competent on $k$-scenes is competent on smaller ones).
\begin{itemize}
\item[(i)] If $k>\lfloor N/2\rfloor$, no placement of $\{\txt(x)\}_{x\in O}$ on $\Sph$
satisfies them.
\item[(ii)] If $k\le\lfloor N/2\rfloor$, placements exist satisfying $(K_j)$ for all
$j\le k$ \emph{and} $(C_\gamma)$ for all scenes with at most $k$ objects, for some
$\gamma>0$.
\end{itemize}
\end{proposition}

\noindent\emph{Attribution.} This dichotomy is published twice: as the \emph{convex
dimension of complete $k$-uniform hypergraphs} (Mart\'inez-Sandoval and Padrol
\cite{msp}: $\mathrm{cd}(K_n^{(k)})=2k$ for $n\ge 2k+2$, with the same Radon lower bound
and neighborly upper bound), and as the \emph{minimal embeddable dimension} for top-$k$
retrieval (\cite{med}: $\mathrm{MED}(m,k)=\min\{2k,m-1\}$ for linear scoring, their Prop.~A.2; $\Theta(k)$ for inner-product/Euclidean/cosine in their main text). Neither literature cites the
other; the bibliographic bridge appears to be new, the mathematics is not. We claim nothing
here beyond the use we make of it.

\begin{proof}[Proof of (i)]
Radon's theorem states: any $N+2$ points in $\R^N$ admit a partition into two nonempty sets
$A,B$ with $\conv(A)\cap\conv(B)\ne\emptyset$. Apply it to any $N+2$ of the $M$ embedding
points. By Lemma~\ref{lem:extreme}, no single point lies in the hull of the others, so both
Radon sides have at least $2$ points. The smaller side $A$ has
$|A|\le\lfloor (N+2)/2\rfloor=\lfloor N/2\rfloor+1\le k$.
Now consider any scene $s$ with $V(s)=A$ (such a scene exists since $2\le|A|\le k$). Strict
separation of $\{\txt(x):x\in A\}$ from the rest is impossible: any functional constant-sign
on $A$ and opposite-sign on its complement would separate $\conv(A)$ from
$\conv(\text{rest})\supseteq\conv(B)$, but these intersect. So $(K_{|A|})$ fails.
Figure~\ref{fig:radon} shows the smallest case.
\end{proof}

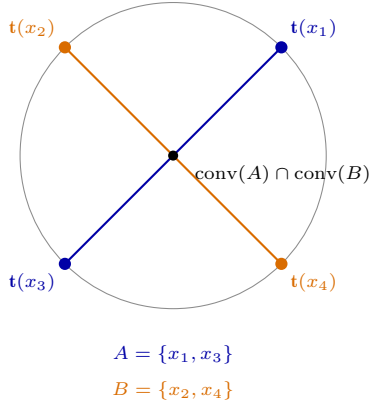
\begin{figure}[t]
\centering
\begin{tikzpicture}[scale=1.35, font=\small,
  arr/.style={-{Stealth[length=2mm]}, black!70}]
\draw[black!45] (0,0) circle (1.5);
\coordinate (p1) at (45:1.5);
\coordinate (p2) at (135:1.5);
\coordinate (p3) at (225:1.5);
\coordinate (p4) at (315:1.5);
\draw[blue!65!black, thick]   (p1) -- (p3);
\draw[orange!85!black, thick] (p2) -- (p4);
\fill[blue!65!black]   (p1) circle (1.7pt) node[above right, font=\scriptsize] {$\txt(x_1)$};
\fill[orange!85!black] (p2) circle (1.7pt) node[above left,  font=\scriptsize] {$\txt(x_2)$};
\fill[blue!65!black]   (p3) circle (1.7pt) node[below left,  font=\scriptsize] {$\txt(x_3)$};
\fill[orange!85!black] (p4) circle (1.7pt) node[below right, font=\scriptsize] {$\txt(x_4)$};
\fill[black] (0,0) circle (1.4pt);
\node[font=\scriptsize, anchor=west] at (0.12,-0.18) {$\conv(A)\cap\conv(B)$};
\node[blue!65!black,   font=\scriptsize] at (0,-1.95) {$A=\{x_1,x_3\}$};
\node[orange!85!black, font=\scriptsize] at (0,-2.30) {$B=\{x_2,x_4\}$};
\end{tikzpicture}
\caption{The Radon impossibility (Proposition~\ref{prop:dichotomy}(i)) in its smallest
instance: $N=2$, $M=4=N+2$ object embeddings on the circle $\mathbb S^1$. Radon's theorem
partitions them into $A$ and $B$ whose hulls (the two chords) intersect; no linear
functional is therefore constant-sign on $\{\txt(x_1),\txt(x_3)\}$ and opposite on the
rest, so no image vector can rank the $2$-scene $\{x_1,x_3\}$ above its complement:
$(K_2)$ fails. In general the smaller Radon side has at most
$\lfloor N/2\rfloor+1$ points, giving the threshold $k\le\lfloor N/2\rfloor$.}
\label{fig:radon}
\end{figure}

\begin{proof}[Proof sketch of (ii)]
Two steps. \emph{Step 1 (single-object layer).} Cumulative separability of all $\le k$-subsets
is equivalent to $k$-neighborliness of the hull (for point sets in the regime $M\ge 2k+2$,
which holds here since $M\ge N+2\ge 2k+2$; the equivalence genuinely requires the cumulative
quantifier: for $4$ points on a circle \emph{every} $3$-subset is separable, since its
complement is a single extreme point, yet \emph{none} is a face; cumulative separability
already fails at the $2$-subsets, the Radon diagonals). Place the $M$ object embeddings on the trigonometric moment
curve: for even $N$ its cyclic polytopes are $\lfloor N/2\rfloor$-neighborly and inscribed;
for odd $N$ use the construction in an equatorial $\mathbb S^{N-2}$, which loses nothing
because $\lfloor(N-1)/2\rfloor=\lfloor N/2\rfloor$ for odd $N$. Every $\le k$-subset now
spans a face, whose supporting functional can be perturbed to a strict separator with an
open cone of witnesses.
\emph{Step 2 (composite layer).} For each scene $s$ place $\txt(c(s))$ at a generic point of
the (open, nonempty) witness cone of $V(s)$, and set $\img(I_s):=\txt(c(s))$ (cross-modal
coincidence is permitted---injectivity is per-map; if undesired, perturb $\img(I_s)$ within
the cone). Then $(K_j)$ holds because $\img(I_s)$ is a separating witness for $V(s)$, and
each $(C_\gamma)$ instance holds because $\ip{\img(I_s)}{\txt(c(s))}=1$ strictly exceeds the
inner product with any \emph{distinct} unit vector, including the swap caption; positivity
of a uniform $\gamma$ follows from finiteness.
\end{proof}

\begin{corollary}[No free-placement obstruction to binding]\label{cor:satisfiable}
For $k\le\lfloor N/2\rfloor$ (e.g.\ any realistic caption length at $N=512$), there exist
placements satisfying every axiom of Definition~\ref{def:axioms} with $\gamma>0$. In
particular the swap conditions---binding---are satisfiable. Any binding impossibility
theorem on this setup must therefore introduce an additional hypothesis, and is exactly as
strong as that hypothesis is natural.
\end{corollary}

This corollary is central to the whole development. Where Kang et al.\ emphasize
impossibility, the picture here is complementary: ideal CLIP is possible under the
axioms alone, and impossibility appears exactly when one adds a structural,
statistical, or semantic hypothesis. The next three sections add those hypotheses one
at a time.

\section{Obstruction I: pairing-blind encoders}\label{sec:thmB}

\begin{definition}[Pairing-blind maps]
A caption map $\txt$ is \emph{pairing-blind} (or \emph{additive}) if $\txt(c)$ depends only
on the \emph{symbol multiset} of the described scene (Definition~\ref{def:scene})---
equivalently, on the word multiset of $c$---not on which attribute is bound to which
object. The canonical example: any encoder whose caption representation is a sum or mean of
per-\emph{word} vectors (bag-of-words pooling). A sum of per-\emph{atom} codes is
\emph{not} pairing-blind---atoms are bound pairs---and \S\ref{sec:thmC} is built on exactly
that distinction.
\end{definition}

\begin{theorem}[Exact structural impossibility]\label{thm:B}
If $\txt$ is pairing-blind then for every scene $s$, $\txt(c(\sigma s))=\txt(c(s))$, hence
$M(s)=0$ identically: the swap instances of $(C_\gamma)$ fail for every $\gamma>0$ and every
image map, and swap benchmark accuracy is chance in expectation under any
label-independent tie-breaking. For a pairing-blind \emph{image} map the conclusion holds
in pair form: $M(s)+M(\sigma s)=0$, so the swap instances for $s$ and $\sigma s$ cannot
both hold (though a single one may).
\end{theorem}
\begin{proof}
$c(s)$ and $c(\sigma s)$ share the same symbol multiset (Definition~\ref{def:scene}), so
a pairing-blind $\txt$ maps them to the same point; the margin
$M(s)=\ip{\img(I_s)}{\txt(c(s))-\txt(c(\sigma s))}=\ip{\img(I_s)}{0}=0$. For a
pairing-blind image map, $\img(I_s)=\img(I_{\sigma s})$, and $M(s)+M(\sigma s)=
\ip{\img(I_s)}{\txt(c(s))-\txt(c(\sigma s))}+\ip{\img(I_s)}{\txt(c(\sigma s))-\txt(c(s))}=0$.
\end{proof}

Three lines, drawn in Figure~\ref{fig:collapse}---and yet this is, we contend, the true core of the ``CLIP cannot bind''
phenomenon in the literature. Everything else in derivations of that flavor (sphere
geometry, superposition lemmas, argmax selections) is scaffolding around the observation
that \emph{a multiset function cannot distinguish rearrangements}. The theorem also
calibrates what an impossibility result must look like: it holds with no dimension
assumption, no training assumption, and an \emph{exact} conclusion---because the hypothesis
(additivity) is maximally strong. The next section weakens the hypothesis to realistic
compositional codes and watches the exact impossibility become a quantitative margin theory.

\begin{figure}[t]
\centering
\begin{tikzpicture}[font=\small,
  box/.style={draw=black!60, rounded corners=2pt, inner sep=4pt, align=center},
  enc/.style={draw=black!70, fill=black!8, rounded corners=3pt, inner sep=6pt, align=center},
  arr/.style={-{Stealth[length=2.2mm]}, black!70}]
\node[box] (cap1) at (0,0.82)
  {``a \textcolor{red!75!black}{red} car and a \textcolor{blue!65!black}{blue} dog''};
\node[font=\scriptsize, anchor=west, inner sep=2pt] at (cap1.east) {$\,=c(s)$};
\node[box] (cap2) at (0,-0.82)
  {``a \textcolor{blue!65!black}{blue} car and a \textcolor{red!75!black}{red} dog''};
\node[font=\scriptsize, anchor=west, inner sep=2pt] at (cap2.east) {$\,=c(\sigma s)$};
\node[enc] (tenc) at (5.75,0) {bag-of-words /\\ additive $\txt$};
\coordinate (sph) at (9.5,0);
\coordinate (tdot) at ($(sph)+(150:0.82)$);   
\coordinate (itip) at ($(sph)+(-32:1.02)$);    
\draw[black!45] (sph) circle (0.82);
\node[font=\scriptsize, black!55, anchor=north] at ($(sph)+(0.35,-0.86)$) {$\Sph$};
\draw[arr] (sph) -- (itip);
\node[font=\scriptsize, anchor=west, inner sep=1.5pt] at (itip) {$\img(I_s)$};
\fill[teal!55!black] (tdot) circle (2.1pt);
\node[teal!45!black, font=\scriptsize, anchor=south, align=center]
  at ($(tdot)+(0.05,0.92)$) {$\txt(c(s))=\txt(c(\sigma s))$\\[-1pt]{\color{black!55}(one point)}};
\draw[black!45, thin] ($(tdot)+(0.05,0.88)$) -- (tdot);
\draw[arr] (cap1.east) .. controls +(0.7,-0.25) and +(-0.7,0.3) .. (tenc.160);
\draw[arr] (cap2.east) .. controls +(0.7,0.25) and +(-0.7,-0.3) .. (tenc.200);
\draw[arr] (tenc.32)  .. controls +(0.85,0.45) and +(-0.85,0.18) .. (tdot);
\draw[arr] (tenc.-18) .. controls +(1.05,0.0)  and +(-0.7,-0.5) .. (tdot);
\node[align=center, font=\small] at (4.55,-2.15)
  {$M(s)=\ip{\img(I_s)}{\txt(c(s))-\txt(c(\sigma s))}=\ip{\img(I_s)}{\mathbf 0}=0$};
\node[align=center, font=\scriptsize, text=black!65] at (4.55,-2.62)
  {same word multiset $\Rightarrow$ same point $\Rightarrow$ zero margin, for \emph{every} image encoder};
\end{tikzpicture}
\caption{Theorem~\ref{thm:B} (bag-of-words collapse). A pairing-blind text map sends both pairings
to the \emph{same} point---they share a word multiset---so the binding margin is exactly zero
for every image encoder, dimension, and training. Contrast Figure~\ref{fig:setup}, where the
two captions are distinct points; the whole ``CLIP cannot bind'' phenomenon is this one
observation that a multiset function cannot distinguish rearrangements.}
\label{fig:collapse}
\end{figure}

\section{Obstruction II: pooled codes and leakage}\label{sec:thmC}

Real caption towers are neither free placements nor exactly additive: they compose
\emph{per-atom codes} into a single vector.

\subsection{The homomorphism form: what ``compositional'' means here}\label{sec:homomorphism}

The classical formalization of compositionality (Montague's) is a homomorphism condition:
\[
(H)\qquad \txt(\text{whole}) \;=\; \Phi\bigl(\{\txt(\text{part}_i)\}_i\bigr)
\]
for a fixed combination operator $\Phi$ (reserving $g$ for the gain of Definition~\ref{def:leakage}): the representation of an expression is a function of
the representations of its parts. Two calibrations before we use it. First, $(H)$ has
content only relative to a decomposition into parts that is \emph{fixed in advance}: with
$\Phi$ and the parts both free, every map is trivially compositional \cite{zadrozny}; and
real transformer encoders are contextual (the vector of ``car'' inside a caption is not a
fixed $\txt(\text{car})$), so $(H)$ is an idealization---precisely the idealization the
pooled codes below make exact. Second, the choice of parts is the whole game:
\begin{itemize}
\item \textbf{Parts $=$ words}, $\Phi$ symmetric. The scene and its swap present the
identical word multiset, so $(H)$ forces $\txt(c(s))=\txt(c(\sigma s))$: this \emph{is}
pairing-blindness, and Theorem~\ref{thm:B} is the one-line corollary.
\item \textbf{Parts $=$ bound constituents} (the parse's phrases), via an inner
$h(\txt(\text{attribute}),\txt(\text{object}))$ and an outer symmetric $\Phi$ (notation of the definitions below). The swap now
changes \emph{which} phrases exist, so the margin can survive. The pooled factored codes
of this section are exactly this two-level homomorphism: $h=\odot$ on typed codebooks,
$\Phi=$ normalized sum.
\end{itemize}
This restates ``it is the code, not the pooling'' algebraically: the outer pooling may be
symmetric---what decides binding is whether the pooled units are words or \emph{bracketed
pairs}. Note also what need \emph{not} be asymmetric: $\odot$ is commutative, role
information is carried by the typed codebooks ($u$ for objects, $v$ for attributes), and
the load-bearing structure is the bracketing (Figure~\ref{fig:parses}). The recursion
implicit in $(H)$ also marks the one axis this paper leaves open---nesting depth
(\S\ref{sec:synthesis}, open problems).

\subsection{Pooled codes, gain, leakage}

The natural intermediate class is:

\begin{definition}[Pooled codes]\label{def:pooled}
A \emph{pooled code} assigns to each atom $(o,a)$ unit vectors $W_\img(o,a)$,
$W_\txt(o,a)\in\R^N$ (the \emph{codebooks}) and embeds
\[
\img(I_s)=\frac{S_\img}{\norm{S_\img}},\quad S_\img=\sum_{j=1}^{k}W_\img(o_j,a_j),
\qquad
\txt(c)=\frac{S_\txt}{\norm{S_\txt}},\quad S_\txt=\sum_{j=1}^{k}W_\txt(o_j,a_j).
\]
The code is \emph{tied} if $W_\img=W_\txt=:W$.
\end{definition}

\begin{definition}[Gain, leakage, cross-talk]\label{def:leakage}
For a pooled code define
\[
g:=\ip{W_\img(o,a)}{W_\txt(o,a)} \quad\text{(exact-match gain)},
\]
\[
\ell:=\max_{\text{one shared atom}}\bigl|\ip{W_\img(o,a)}{W_\txt(o,a')}\bigr|\vee
\bigl|\ip{W_\img(o,a)}{W_\txt(o',a)}\bigr| \quad\text{(leakage)},
\]
\[
\ell_0:=\max_{\text{no shared atom}}\bigl|\ip{W_\img(o,a)}{W_\txt(o',a')}\bigr|
\quad\text{(cross-talk)},\qquad L:=\max(\ell,\ell_0).
\]
Leakage is the score a code grants a \emph{partially} matching atom---right object, wrong
attribute, or vice versa. Paper~1 identified accumulated leakage as the mechanism of the
anti-generalization of learned MLP composition codes (trained codes develop
$\ell\approx+0.17$ while random factored codes have $\ell\approx 0.00$); here it becomes
the organizing functional of a theorem. Figure~\ref{fig:pooling} draws the code (a) and the functional (b).
\end{definition}

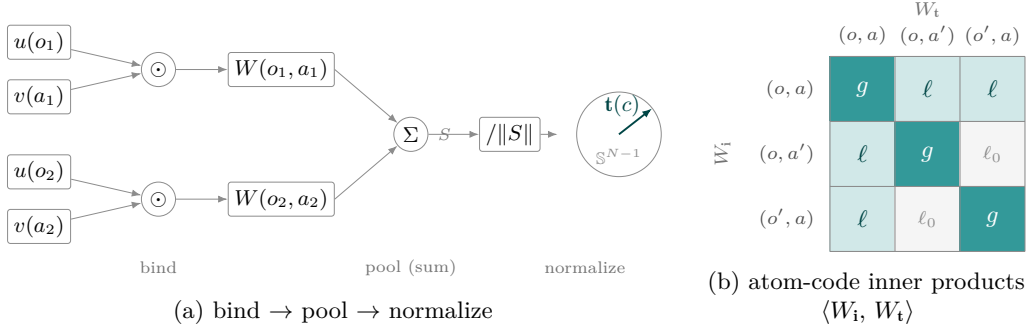
\begin{figure}[t]
\centering
\begin{minipage}[b]{0.62\linewidth}
\centering
\begin{adjustbox}{max width=\linewidth}\begin{tikzpicture}[
  font=\small,
  >={Latex[length=1.4mm]},
  code/.style={draw=black!45, rounded corners=1.2pt, inner sep=2.6pt,
               fill=white, line width=0.4pt},
  op/.style={draw=black!45, circle, inner sep=0.5pt, minimum size=5mm,
             fill=white, line width=0.4pt},
  ar/.style={-{Latex[length=1.4mm]}, black!45, line width=0.4pt},
]
\node[code] (u1) at (0, 1.35) {$u(o_1)$};
\node[code] (v1) at (0, 0.55) {$v(a_1)$};
\node[code] (u2) at (0,-0.55) {$u(o_2)$};
\node[code] (v2) at (0,-1.35) {$v(a_2)$};
\node[op] (b1) at (1.75, 0.95) {$\odot$};
\node[op] (b2) at (1.75,-0.95) {$\odot$};
\node[code] (w1) at (3.55, 0.95) {$W(o_1,a_1)$};
\node[code] (w2) at (3.55,-0.95) {$W(o_2,a_2)$};
\node[op] (sum) at (5.45, 0) {$\Sigma$};
\node[black!55, font=\scriptsize, right=0.5pt of sum] {$S$};
\node[code] (nrm) at (6.9,0) {$/\norm{S}$};
\begin{scope}[shift={(8.5,0)}]
  \draw[black!45, line width=0.4pt] (0,0) circle (0.62);
  \draw[teal!55!black, -{Latex[length=1.5mm]}, line width=0.8pt] (0,0) -- (38:0.62);
  \node[teal!55!black, font=\footnotesize] at (0.05,0.42) {$\txt(c)$};
  \node[black!45, font=\scriptsize] at (0.0,-0.32) {$\Sph$};
\end{scope}
\draw[ar] (u1) -- (b1); \draw[ar] (v1) -- (b1);
\draw[ar] (u2) -- (b2); \draw[ar] (v2) -- (b2);
\draw[ar] (b1) -- (w1); \draw[ar] (b2) -- (w2);
\draw[ar] (w1.east) -- (sum); \draw[ar] (w2.east) -- (sum);
\draw[ar] (sum) -- (nrm);
\draw[ar] (nrm) -- (7.6,0);
\node[black!55, font=\scriptsize] at (1.75,-1.95) {bind};
\node[black!55, font=\scriptsize] at (5.45,-1.95) {pool (sum)};
\node[black!55, font=\scriptsize] at (8.0,-1.95) {normalize};
\end{tikzpicture}\end{adjustbox}
\\[2pt]
{\footnotesize (a) bind $\rightarrow$ pool $\rightarrow$ normalize}
\end{minipage}\hfill
\begin{minipage}[b]{0.37\linewidth}
\centering
\begin{adjustbox}{max width=\linewidth}\begin{tikzpicture}[font=\small]
  \def\s{0.86}
  \def\hA{$(o,a)$}\def\hB{$(o,a')$}\def\hC{$(o',a)$}
  \fill[teal!80] (0,2*\s) rectangle (1*\s,3*\s);
  \fill[teal!18] (1*\s,2*\s) rectangle (2*\s,3*\s);
  \fill[teal!18] (2*\s,2*\s) rectangle (3*\s,3*\s);
  \fill[teal!18] (0,1*\s) rectangle (1*\s,2*\s);
  \fill[teal!80] (1*\s,1*\s) rectangle (2*\s,2*\s);
  \fill[black!4]  (2*\s,1*\s) rectangle (3*\s,2*\s);
  \fill[teal!18] (0,0) rectangle (1*\s,1*\s);
  \fill[black!4]  (1*\s,0) rectangle (2*\s,1*\s);
  \fill[teal!80] (2*\s,0) rectangle (3*\s,1*\s);
  \draw[black!45, line width=0.4pt] (0,0) grid[step=\s] (3*\s,3*\s);
  \node[white,font=\footnotesize] at (0.5*\s,2.5*\s) {$g$};
  \node[white,font=\footnotesize] at (1.5*\s,1.5*\s) {$g$};
  \node[white,font=\footnotesize] at (2.5*\s,0.5*\s) {$g$};
  \node[teal!55!black,font=\footnotesize] at (1.5*\s,2.5*\s) {$\ell$};
  \node[teal!55!black,font=\footnotesize] at (2.5*\s,2.5*\s) {$\ell$};
  \node[teal!55!black,font=\footnotesize] at (0.5*\s,1.5*\s) {$\ell$};
  \node[teal!55!black,font=\footnotesize] at (0.5*\s,0.5*\s) {$\ell$};
  \node[black!45,font=\scriptsize] at (2.5*\s,1.5*\s) {$\ell_0$};
  \node[black!45,font=\scriptsize] at (1.5*\s,0.5*\s) {$\ell_0$};
  \foreach \j/\lab in {1/\hA,2/\hB,3/\hC}{
    \node[black!70,font=\scriptsize] at ({(\j-0.5)*\s},{3*\s+0.28}) {\lab};}
  \foreach \i/\lab in {1/\hA,2/\hB,3/\hC}{
    \node[black!70,font=\scriptsize,anchor=east] at (-0.08,{(3.5-\i)*\s}) {\lab};}
  \node[black!55,font=\scriptsize] at (1.5*\s,{3*\s+0.62}) {$W_\txt$};
  \node[black!55,font=\scriptsize,rotate=90] at (-1.42,1.5*\s) {$W_\img$};
\end{tikzpicture}\end{adjustbox}
\\[3pt]
{\footnotesize (b) atom-code inner products $\ip{W_\img}{W_\txt}$}
\end{minipage}

\caption{Pooled factored codes and leakage (Section~\ref{sec:thmC}). \textbf{(a)} Each atom
$(o_j,a_j)$ is bound by a Hadamard product $u(o_j)\odot v(a_j)$; the bound codes are
summed and normalized to the caption embedding $\txt(c)$. \textbf{(b)} The code's
quality is set by \emph{leakage} $\ell$ --- the inner product a \emph{partial} match
(right object, wrong attribute) receives --- against the exact-match gain $g$. Random
tied factored codes achieve $g=1$, $\ell\approx0$; additive (bag-of-words) codes sit at
$\ell=\Theta(g)$ and cannot bind.}
\label{fig:pooling}
\end{figure}

Throughout this section scenes have distinct objects \emph{and} distinct attributes, and
$k\ge 2$.

\subsection{The norm-cancellation lemma}

Margins of pooled codes are differences of two cosines whose numerators \emph{and
denominators} both change under the edit. Bounding the two norms independently is fatally
lossy; the following trivial identity is what rescues the computation.

\begin{lemma}[Norm cancellation]\label{lem:C0}
Let $S'=S-D$. Then
\[
\norm{S'}^2-\norm{S}^2=\norm{D}^2-2\ip{S}{D}.
\]
Consequently, if $D$ involves at most $4$ atom codes and all pairwise inner products between
atom codes are bounded by $L$, then $\norm{D}^2=O(1)$, $\ip{S}{D}=O(1)+O(kL)$, and
\[
\bigl|\norm{S'}-\norm{S}\bigr|
=\frac{\bigl|\norm{S'}^2-\norm{S}^2\bigr|}{\norm{S'}+\norm{S}}
=O\!\Bigl(\frac{1+kL}{\sqrt k}\Bigr).
\]
\end{lemma}
\begin{proof}
Expand $\norm{S-D}^2$. The consequence follows since $\norm{S},\norm{S'}=\Theta(\sqrt k)$
when $kL$ is bounded away from $1$ (shown below).
\end{proof}

The point: the \emph{naive} bound $|\norm{S'}-\norm{S}|\le\norm{D}=O(1)$ loses a factor
$\sqrt k$, and a margin computed that way carries an error $O(1/\sqrt k)$ that swamps the
true $\Theta(1/k)$ signal. Every margin bound below routes through the identity.

\subsection{Random factored codes and their primitives}

\begin{definition}[Tied factored sign codes]\label{def:factored}
Draw $u(o)\in\{\pm1\}^N$ for each object and $v(a)\in\{\pm1\}^N$ for each attribute,
i.i.d.\ uniform. The \emph{factored code} is the tied pooled code with
\[
W(o,a):=\frac{u(o)\odot v(a)}{\sqrt N}\qquad(\odot=\text{coordinatewise product}),
\]
which is exactly unit-norm. (This is the classical Hadamard binding of vector-symbolic
architectures \cite{plate,kanerva}.)
\end{definition}

All quantities in the margin computations are linear in the following \emph{primitives}:
{\small\[
\frac{\ip{u(o)}{u(o')}}{N},\quad
\frac{\ip{v(a)}{v(a')}}{N},\quad
\frac{\ip{u(o)\odot u(o')}{v(a)\odot v(a')}}{N},\quad
\frac{\ip{u(o)\odot u(o')}{v(a)}}{N},\quad
\frac{\ip{v(a)}{\mathbf 1}}{N},
\]}
for distinct arguments. Each is an average of $N$ i.i.d.\ $\pm1$ variables (coordinatewise
products of independent Rademacher vectors are Rademacher), and there are at most $4(MK)^2$
of them.

\begin{lemma}[Primitive control]\label{lem:primitives}
Fix a failure probability $\eta\in(0,1)$ (reserving $\delta$ for the smoothness gap of \S\ref{sec:frontier}) and set
\[
\Lstar:=\sqrt{\frac{2\log\bigl(8(MK)^2/\eta\bigr)}{N}}.
\]
With probability at least $1-\eta$ over the codebooks, \emph{every} primitive is at most
$\Lstar$ in absolute value simultaneously.
\end{lemma}
\begin{proof}
Hoeffding: an average of $N$ i.i.d.\ $\pm1$'s exceeds $t$ in absolute value with probability
at most $2e^{-Nt^2/2}$. Union bound over at most $4(MK)^2$ primitives:
$4(MK)^2\cdot 2e^{-Nt^2/2}\le\eta$ iff $t\ge\Lstar$.
\end{proof}

\begin{remark}
The union bound runs over the $O((MK)^2)$ \emph{primitives}, not over the exponentially many
axiom instances---every instance margin is a fixed linear combination of primitives. This is
why the dimension requirement below scales as $\log(MK)$ and not $k\log(MK)$.
\end{remark}

On the event of Lemma~\ref{lem:primitives}: the factored code has $g=1$ exactly (tied), and
every pairwise inner product between distinct atom codes is a primitive, hence
$\ell,\ell_0\le\Lstar$. Note also the necessity of tying: if $W_\img$ and $W_\txt$ are drawn
\emph{independently}, then $g$ itself is a primitive-sized quantity $O(\Lstar)$---the same
order as the leakage---and every margin collapses. Tying (or training toward alignment) is a
real hypothesis, not a convenience.

\subsection{The swap margin, in full}\label{sec:swapmargin}

\begin{theorem}[Swap margin of the factored code]\label{thm:Cswap}
Let all primitives be bounded by $L$ (a bound that dominates the $\ell,\ell_0$ of Definition~\ref{def:leakage}) and suppose $kL\le 1/5$, $k=|s|\ge2$. For every scene
(distinct objects and attributes) and every swap,
\[
\Bigl|\,M_{\mathrm{swap}}-\frac{2}{k}\Bigr|\;\le\;19\,L .
\]
In particular, for $L<2/(19k)$ the margin is $\Theta(1/k)$ \emph{from both sides}: pooled
codes do not merely achieve a positive margin, they are pinned at $2/k$. (At the edge of
the stated regime $kL\le1/5$ the interval is vacuous; the typical case
$L\approx\sqrt{1/N}$ is far inside it.)
\end{theorem}

\begin{proof}
Write $W_j:=W(o_j,a_j)$, $S:=\sum_{j\le k}W_j$, and for the swap on positions $1,2$:
\[
S'=S-D,\qquad D=W(o_1,a_1)+W(o_2,a_2)-W(o_1,a_2)-W(o_2,a_1).
\]
Because the code is tied, $\img(I_s)=\txt(c(s))=S/\norm S$ and $\txt(c(\sigma s))=S'/\norm{S'}$, so
\[
M_{\mathrm{swap}}
=1-\frac{\ip{S}{S'}}{\norm S\,\norm{S'}}
=1-\frac{p-a}{\sqrt{pq}},\qquad
p:=\norm S^2,\;\; q:=\norm{S'}^2,\;\; a:=\ip{S}{D}.
\]
\emph{Bounds on the pieces.} (1) Distinct positions share no object and no attribute, so
$\ip{W_j}{W_l}$ ($j\ne l$) is a quadruple primitive; hence
$p\in[k-k(k-1)L,\;k+k(k-1)L]\subseteq k[1-kL,\,1+kL]$, and---since $S'$ is also a sum of $k$
atom codes with distinct objects and attributes---likewise $q\in k[1-kL,\,1+kL]$.
(2) $a=\sum_j\ip{W_j}{D}$. For $j=1$: the four inner products are $1$ (exact match), one
quadruple primitive, one attribute primitive, one object primitive, so
$\ip{W_1}{D}\in[1-3L,\,1+3L]$; same for $j=2$; for $j\ge3$ all four are primitives:
$\ip{W_j}{D}\in[-4L,4L]$. Hence $a\in[\,2-(4k-2)L,\;2+(4k-2)L\,]\subseteq[2-4kL,\,2+4kL]$.
(3) $\norm D^2=4+2(\text{six signed primitives})\in[4-12L,\,4+12L]$.
(4) By Lemma~\ref{lem:C0}, $p-q=2a-\norm D^2\in[-(8k+12)L,\,(8k+12)L]$; for $k\ge2$,
$(8k+12)L\le14kL$.

\emph{Assembling.} Split
\[
M_{\mathrm{swap}}=1-\sqrt{p/q}+\frac{a}{\sqrt{pq}} .
\]
Set $x:=(p-q)/q$, so $\sqrt{p/q}=\sqrt{1+x}$. From $q\ge k(1-kL)\ge\frac45k$ and the two
bounds on $|p-q|$ (the cancellation bound $14kL$, and the crude bound
$|p-q|\le 2k^2L$ from $p,q\in k[1\pm kL]$):
\[
|x|\;\le\;\frac{14kL}{\frac45 k}=17.5\,L,
\qquad
|x|\;\le\;\frac{2k^2L}{\frac45 k}=2.5\,kL\;\le\;\tfrac12 .
\]
For the upper side, concavity gives $\sqrt{1+x}\le1+\frac x2\le1+8.75\,L$. For the lower
side use the identity $1-\sqrt{1+x}=\frac{-x}{1+\sqrt{1+x}}$: with $x\ge-\frac12$ we have
$\sqrt{1+x}\ge\sqrt{1/2}$, hence
\[
1-\sqrt{1+x}\;\le\;\frac{|x|}{1+\sqrt{1/2}}\;\le\;0.586\,|x|\;\le\;10.26\,L,
\quad\text{i.e.}\quad \sqrt{p/q}\;\ge\;1-10.26\,L .
\]
For the last term, $\sqrt{pq}\in k[1-kL,1+kL]$, so
\[
\frac{a}{\sqrt{pq}}\;\ge\;\frac{2-4kL}{k(1+kL)}\;\ge\;\frac{(2-4kL)(1-kL)}{k}\;\ge\;\frac2k-6L,
\qquad
\frac{a}{\sqrt{pq}}\;\le\;\frac{2+4kL}{k(1-kL)}\;\le\;\frac2k+8L,
\]
using $kL\le1/5$ in the last step ($\frac{1}{1-kL}\le 1+\frac54 kL$). Combining,
\[
\frac2k-6L-8.75L\;\le\;M_{\mathrm{swap}}\;\le\;\frac2k+8L+10.26L,
\]
i.e.\ $|M_{\mathrm{swap}}-2/k|\le18.26L\le19L$. Constants are not optimized.
\end{proof}

\begin{remark}[Numerical check]\label{rem:numerics}
The bound was verified against simulation with real random sign codes:
$M{=}K{=}30$, $400$ scenes per configuration.
At $N=4096$ the measured mean swap margins are $0.9995$, $0.4999$, $0.2493$ for
$k=2,4,8$, against $2/k=1$, $0.5$, $0.25$: concentration \emph{at} $2/k$ to three
decimals, consistent with the two-sided pinch (typical primitive size is
$\approx\sqrt{1/N}$, far below the worst-case $\Lstar$).
\end{remark}

The same computation with a replace negative ($D=W(o_1,a_1)-W(o_1,a')$, two codes instead of
four) gives $M_{\mathrm{repl}}\ge 1/k-7L$; an add negative gives, exactly at orthogonality
and stably under $O(L)$ perturbation,
\[
M_{\mathrm{add}}=1-\sqrt{\frac{k}{k+1}}-O(L)\;\ge\;\frac{1}{2(k+1)}-O(L).
\]
(The exact form $1-\sqrt{k/(k+1)}$ matters: simulation gives $0.1831,\,0.1056,\,0.0571$ for
$k=2,4,8$, matching it to three decimals, whereas the first-order approximation
$\frac1{2(k+1)}$ is visibly off at $k=2$.) Drop is symmetric to add. \emph{Swap and replace are proved in full; the add case is outlined --- derivative bookkeeping around the orthogonal point is
routine and omitted.}

Two structural readings (Figure~\ref{fig:margindecay} shows all three margins tracking their rates). First, \emph{all} single-edit margins are $\Theta(1/k)$: the margin
budget of a pooled representation is systemic, so $(C_\gamma)$ is meaningful only for
$\gamma\lesssim1/k$. Second, the swap margin is the \emph{largest} of the family ($2/k$: a
swap edits two atoms)---pooled codes are not intrinsically worse at swaps than at replaces;
what distinguishes swaps is the text side of \S\ref{sec:frontier} and the data side of
\S\ref{sec:thmD}.

\begin{figure}[t]
\centering
\begin{tikzpicture}
\begin{axis}[
    width=11cm, height=7.6cm,
    xlabel={scene size $k$ (number of pooled atoms)},
    ylabel={normalized margin $1-\cos$},
    xmin=1.6, xmax=16.4,
    ymin=0, ymax=1.05,
    xtick={2,4,6,8,12,16},
    ytick={0,0.2,0.4,0.6,0.8,1.0},
    tick align=outside,
    tick label style={font=\small},
    label style={font=\small},
    legend style={font=\footnotesize, at={(0.98,0.98)}, anchor=north east,
                  cells={anchor=west}, draw=black!40, fill=white, fill opacity=0.9,
                  text opacity=1, row sep=1pt},
    legend cell align=left,
    axis lines=left,
    every axis plot/.append style={line width=0.9pt},
    clip=false,
]
\addplot[teal!55!black, samples=200, domain=2:16] {2/x};
\addlegendentry{swap theory $2/k$}
\addplot[blue!60!black, samples=200, domain=2:16] {1/x};
\addlegendentry{replace theory $1/k$}
\addplot[orange!85!black, samples=200, domain=2:16] {1-sqrt(x/(x+1))};
\addlegendentry{add theory $1-\sqrt{k/(k+1)}$}
\addplot[only marks, mark=o, mark size=2.6pt, teal!55!black, line width=0.9pt]
    coordinates {(2,1.00205) (3,0.66666) (4,0.49992) (6,0.33296) (8,0.24937) (12,0.16661) (16,0.12474)};
\addlegendentry{swap measured}
\addplot[only marks, mark=square, mark size=2.3pt, blue!60!black, line width=0.9pt]
    coordinates {(2,0.49891) (3,0.33338) (4,0.25006) (6,0.16667) (8,0.12501) (12,0.08332) (16,0.06227)};
\addlegendentry{replace measured}
\addplot[only marks, mark=triangle, mark size=2.9pt, orange!85!black, line width=0.9pt]
    coordinates {(2,0.18366) (3,0.13413) (4,0.10559) (6,0.07420) (8,0.05718) (12,0.03921) (16,0.02974)};
\addlegendentry{add measured}
\end{axis}
\end{tikzpicture}
\caption{The margin ceiling is $\Theta(1/k)$ (Theorem~\ref{thm:Cswap}). Marks: measured
normalized margins of the tied factored sign code ($N=4096$, $M=K=40$, 300 scenes per size)
against swap, replace, and add negatives. Lines: the theoretical $2/k$, $1/k$, and
$1-\sqrt{k/(k+1)}$. Measured values sit on the predictions to plotting accuracy---in
particular the \emph{exact} add-margin form $1-\sqrt{k/(k+1)}$, not the loose
$1/(2(k+1))$.}
\label{fig:margindecay}
\end{figure}
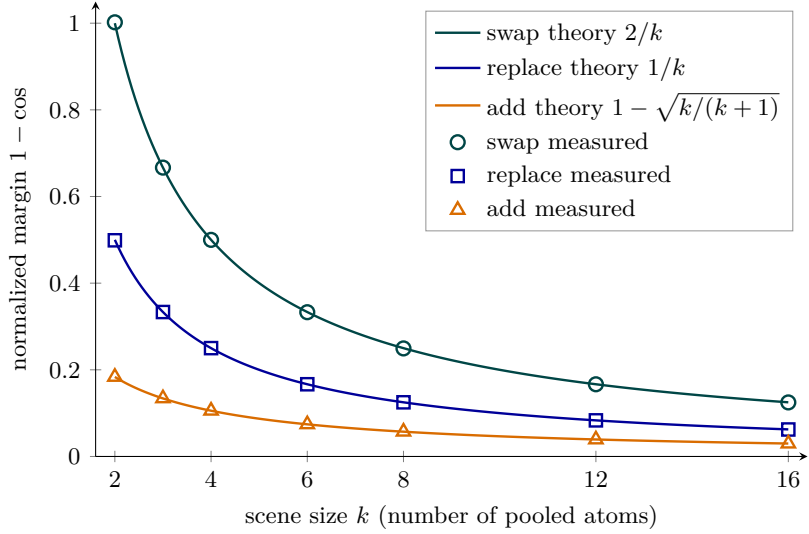

\subsection{Constituent ranking fails for the plain code; augmentation}

A finding of the constants pass that corrects our own earlier sketch: the factored code, as
defined, \emph{cannot} satisfy $(K_k)$.

\begin{proposition}[The unbinding obstruction]\label{prop:unbind}
For the factored code, embed the single-object caption as $\txt(x)=u(x)/\sqrt N$. Then for
every scene $s$ and every object $x$ (present or absent),
\[
\ip{\img(I_s)}{\txt(x)}=\frac{1}{\norm S}\sum_{j\le k}\rho_j,\qquad
\rho_j:=\frac{\ip{u(o_j)\odot v(a_j)}{u(x)}}{N},
\]
and \emph{every} $\rho_j$ is a primitive-sized quantity: for $x=o_j$,
$\rho_j=\ip{v(a_j)}{\mathbf 1}/N$ (the mean of $v(a_j)$), and for $x\ne o_j$ a triple
primitive. Hence all scores, of present and absent objects alike, lie in
$[-c\sqrt k\,\Lstar,\,c\sqrt k\,\Lstar]$: the code offers \emph{no signal} separating
$V(s)$ from its complement, and $(K_k)$ cannot be certified at any fixed margin. In
simulation the $(K_k)$ event (all present objects outranking all absent ones) occurs at
rate $0.00$--$0.01$.
\end{proposition}

This is the classical vector-symbolic fact that bound products are dissimilar to their
factors---one \emph{unbinds} to query, one does not take inner products with raw factors
\cite{plate,kanerva}---resurfacing as an axiom failure. The repair is to pool an
object-presence term alongside the bound pair:

\begin{theorem}[Augmented factored code]\label{thm:Caug}
Define $\widehat W(o,a):=\bigl(u(o)\odot v(a)+u(o)\bigr)/\sqrt{2N}$ and pool as before
(each $\widehat W$ is unit-norm up to $O(\Lstar)$, absorbed into the constants). On
the event of Lemma~\ref{lem:primitives} with $k\Lstar\le1/5$:
\begin{itemize}
\item[(i)] $(K_k)$ holds with margin $\Theta(1/\sqrt k)$: present objects score
$\tfrac{1}{\sqrt{2k}}(1-O(k\Lstar))$, absent objects $O(\sqrt k\,\Lstar)$. \emph{(Outlined:
the computation is the display above plus the new primitive types; simulation pass rate
$1.00$ in-regime.)}
\item[(ii)] The \emph{attribute-edit} margins are exactly halved: the presence part
$u(o)$ cancels in any difference that preserves the object multiset (swap,
attribute-replace), while the $\sqrt{2}$ normalization dilutes the surviving bound part.
Add/drop and \emph{object}-replace margins are \emph{unchanged}---their differences carry a
presence term. Quantitatively, $|M_{\mathrm{swap}}-1/k|\le 27L$. \emph{(Proved by the route
of Theorem~\ref{thm:Cswap}, adjusted for the augmented code's weaker primitives
($p,q\in k[1\pm2kL]$, forcing the cruder $1-\sqrt{1+x}\le|x|$ step; constant unoptimized).
Simulation: augmented swap margin $0.1243$ vs.\ half $0.1247$ at $k=8$; augmented
object-replace and add margins match the \emph{plain} values to $3$ decimals.)}
\item[(iii)] A mixing weight ($u\odot v+\alpha u$) trades the two margins continuously; no
single pooled vector maximizes both. \emph{(Outlined: the $(K_k)$ margin is increasing and
the attribute-edit margins decreasing in $\alpha$.)}
\end{itemize}
\end{theorem}

Part (iii) deserves emphasis: \emph{within} the pooled class, the retrieval axiom $(K_k)$
and the binding axioms compete for the same normalization budget. This is a miniature,
class-internal echo of the global tension that \S\ref{sec:frontier} proves for arbitrary
encoders.

\subsection{Dimension rates and the leakage dichotomy}

\begin{corollary}[Binding is cheap if the code is right]\label{cor:rate}
Rates are keyed to the instance family they certify: the uniform margin over the whole
hard core is set by its weakest member---add/drop, $\approx\tfrac1{2(k+1)}$
(Remark~\ref{rem:negatives}(2))---so no bound of the form $\gamma=c/k$ with large $c$ can
cover the full family. For the plain factored code: the \emph{swap} instances hold at
$\gamma=c/k$, $c<2$, once $\Lstar\le(2-c)/(19k)$, i.e.
\[
N\;\ge\;\frac{722\,k^2\,\log\bigl(8(MK)^2/\eta\bigr)}{(2-c)^2};
\]
the \emph{replace} instances at $\gamma=c/k$, $c<1$, once $\Lstar\le(1-c)/(7k)$
($N\ge98\,k^2\log(8(MK)^2/\eta)/(1-c)^2$); and a \emph{uniform} margin
$\gamma=c/(2(k+1))$, $c<1$, over the whole hard core at $N=O(k^2\log MK)$ (the add/drop
constant is Outlined). For the augmented code the swap instances hold at $\gamma=c/k$,
$c<1$, once $\Lstar\le(1-c)/(27k)$, i.e.\
$N\ge1458\,k^2\log(8(MK)^2/\eta)/(1-c)^2$, with $(K_k)$ joining at margin
$\Theta(1/\sqrt k)$. In every case: logarithmic in the number of concepts, quadratic in
scene size. (Constants from Theorems~\ref{thm:Cswap} and~\ref{thm:Caug}(ii); none
optimized.)
\end{corollary}

The encoder-class landscape is now a three-way dichotomy organized by leakage:
$\ell\approx0$ (factored codes: all axioms satisfiable at $N=O(k^2\log MK)$);
$\ell=\Theta(g)$ with pairing symmetry (additive codes: margin exactly $0$,
Theorem~\ref{thm:B}); $\ell$ uncontrolled (generic learned codes: no guarantee, and
Paper~1 measures trained MLP composition heads at $\ell\approx+0.17$ with below-chance
extrapolation to unseen pairings as the consequence). \emph{It is the code, not the
pooling}---now with a rate.

\begin{remark}[Pooled atom codes have exactly the required invariance]
\label{rem:orderinv}
In the language of Definition~\ref{def:realizations}: a pooled code over \emph{bound
pairs} is a set sum, hence invariant under the order variants within $\mathcal C(s)$
(every order variant of a scene receives the identical embedding---zero intra-class
spread) while separating $\mathcal C(\sigma s)$ at margin $\Theta(1/k)$
(Theorem~\ref{thm:Cswap}), because the swap changes \emph{which pairs} are summed. Pooling
over \emph{words} (the additive code) is invariant under both, fatally
(Remark~\ref{rem:invariance}). Pooling is thus compatible with binding precisely when the
pooled units are the bound pairs---a second, sharper sense of ``it is the code, not the
pooling.''
\end{remark}

\section{Obstruction III: the training distribution}\label{sec:thmD}

The placements of \S\S\ref{sec:possibility}--\ref{sec:thmC} answer ``could vectors exist'';
real encoders are \emph{selected by an objective on data}. This section states, at the level
of precision they currently have, the two data-side laws established in Paper~1, because the
measurements of \S\S\ref{sec:zoo}--\ref{sec:bench} constantly refer to them. Both are
ported from Paper~1: formal statements with proof outlines and quantitative empirical
support, not theorems of the present paper.

\begin{lemma}[What contrastive training optimizes; outline]\label{lem:infonce}
Population InfoNCE with matching pairs $(I,c)\sim p$ and negatives $c'\sim Q$ is minimized
(over unconstrained score functions) at $s^*(I,c)=\log\frac{p(c\mid I)}{Q(c)}+h(I)$ for an arbitrary per-image function $h$
\cite{infonce}. Consequently the objective carries signal for separating $c(s)$ from
$c(\sigma s)$ only insofar as the pair $(p,Q)$ distinguishes them on observed data.
\end{lemma}

\begin{proposition}[Throttle and coverage; ported]\label{prop:D}
Let $\pi$ denote the rate at which the negative process presents the \emph{swap} of the
positive caption.
\begin{itemize}
\item[(i)] \textbf{(Throttle.)} The \emph{trainable} swap margin---the margin attainable by gradient training on $D$---scales linearly in $\pi$ for small
$\pi$. With \emph{in-batch} negatives (the other captions of the same training batch) over a caption corpus of size $|C|$, $\pi=O(1/|C|)$:
asymptotically no swap signal, and the selected optimum is a bag-of-words-equivalent
solution violating every swap instance of $(C_\gamma)$. \emph{Measured} (Paper~1, real-COCO
$\pi$-sweep): swap accuracy rises $0.62\to0.70$, concavely, as $\pi:0\to1$, with retrieval
metrics flat; a loss-weight dial $\lambda$ is near-equivalent to $\pi$.
\item[(ii)] \textbf{(Coverage.)} For atom pairs never co-observed in training, the loss is
flat in their relative placement: held-out binding is decided by the encoder class alone.
Factored-structure codes generalize along an occupancy law
$\mathrm{acc}(\rho)\approx\tfrac12+\tfrac12(1-e^{-\rho K})(1-e^{-\rho M})$ in the pair-density
$\rho$ (the expected number of training exposures per object--attribute pair), with
threshold at $\rho\approx2$--$3$; additive codes sit at chance
(Theorem~\ref{thm:B}); generic MLP codes extrapolate \emph{below} chance via the leakage
mechanism of Definition~\ref{def:leakage}. \emph{Measured} (Paper~1 five-arm code ablation;
external refit to two published training curves).
\end{itemize}
\end{proposition}

The reading that matters downstream: the axioms of Definition~\ref{def:axioms}, pushed
through training, become \emph{conditions on the data distribution}---a positive swap rate
$\pi$ on the contrasted families, and either pair coverage or a factored-class code. Neither
condition holds for standard web-scale contrastive training, which is Paper~1's explanation
of the status quo. The present paper asks what remains \emph{after} those conditions are
met; the answer is the frontier. (Appendix~\ref{sec:throttle} upgrades part (i) of the
proposition to proved, architecture-free theorems.)

\section{The smoothness--binding frontier}\label{sec:frontier}

\subsection{The two-jobs tension}\label{sec:twojobs}

An embedding space in deployment serves two masters:
\begin{itemize}
\item \textbf{Semantic similarity.} Paraphrases and near-variants should embed close
  together---this is what retrieval, similarity search, and deduplication consume, and what
  similarity benchmarks reward.
\item \textbf{Binding.} $c(s)$ and $c(\sigma s)$ should embed far apart---this is what swap
  benchmarks reward.
\end{itemize}
The tension is that a swap is a \emph{minimal} semantic edit: same words, same objects, same
attributes, different wiring. The SugarCrepe++ literature \cite{sugarcrepepp} documents the
dissociation empirically (strong hard-negative performance need not transfer to
paraphrase consistency, and several compositionality fine-tunes degrade on it), but to
our knowledge no quantitative statement existed. The
frontier theorem makes the trade-off exact.

\subsection{Anchors, smoothness, and the guard against tautology}

\begin{definition}[Pairing-neutral anchor; $\delta$-smoothness]\label{def:anchor}
Fix a scene $s$. A \emph{pairing-neutral anchor} is any unit vector $\bar t_s$ that depends
only on the \emph{symbol multiset} of $s$ (Definition~\ref{def:scene})---not on the pairing. (Examples: the embedding of
the bag-of-atoms caption ``a photo of a car and a dog, red and blue''; the mean embedding of
a paraphrase family generated without pairing information.) The caption map is
\emph{$\delta$-smooth at $s$ with respect to $\bar t_s$} if both pairings lie in the
spherical cap of cosine-radius $1-\delta$ about it:
\[
\ip{\txt(c(s))}{\bar t_s}\ \ge\ 1-\delta
\qquad\text{and}\qquad
\ip{\txt(c(\sigma s))}{\bar t_s}\ \ge\ 1-\delta .
\]
\end{definition}

Why the anchor must be pairing-neutral---and why the definition would otherwise be
vacuous---is exposed by the following computation, which we will reuse twice.

\begin{lemma}[Midpoint identity]\label{lem:midpoint}
Let $t_1,t_2\in\Sph$, $d:=\norm{t_1-t_2}$, and let
$\bar m:=(t_1+t_2)/\norm{t_1+t_2}$ be their normalized midpoint. Then both $t_i$ satisfy
$\ip{t_i}{\bar m}=1-\delta_{\mathrm{mid}}$ with
\[
d\;=\;2\sqrt{2\delta_{\mathrm{mid}}}\cdot\sqrt{1-\delta_{\mathrm{mid}}/2}.
\]
\end{lemma}
\begin{proof}
$\ip{t_i}{\bar m}=\frac{1+\ip{t_1}{t_2}}{\norm{t_1+t_2}}=\frac{\norm{t_1+t_2}}{2}$, so
$2\delta_{\mathrm{mid}}=2-\norm{t_1+t_2}$. Then
$d^2=2-2\ip{t_1}{t_2}=4-\norm{t_1+t_2}^2
=(2-\norm{t_1+t_2})(2+\norm{t_1+t_2})=2\delta_{\mathrm{mid}}(4-2\delta_{\mathrm{mid}})$.
\end{proof}

So \emph{against its own midpoint, every pair sits exactly on the curve
$d=2\sqrt{2\delta}$ (to leading order)}. If the anchor were allowed to depend on the pair,
the theorem below would be an identity with zero content. The force of
Definition~\ref{def:anchor} is the quantifier: the bound holds for \emph{every}
pairing-neutral anchor, so exhibiting any \emph{one} independently defined anchor near both
variants yields a genuine, non-circular ceiling. Measurement protocols must (and ours do)
use an anchor constructed without reference to the pairing.

\subsection{The theorem}

\begin{theorem}[Smoothness--binding frontier]\label{thm:E}
Let $s$ be a scene and suppose the caption map is $\delta$-smooth at $s$ with respect to
some pairing-neutral anchor $\bar t_s$.
\begin{itemize}
\item[(i)] \textbf{(Upper bound.)} For \emph{every} image map, every dimension $N$, and
every training procedure,
\[
M(s)\;\le\;\norm{\txt(c(s))-\txt(c(\sigma s))}\;\le\;2\sqrt{2\delta}.
\]
\item[(ii)] \textbf{(Tightness.)} For every $\delta\in(0,1)$ there exist unit vectors
$t_1,t_2$, a pairing-neutral anchor, and a unit image vector achieving
\[
M(s)\;=\;2\sqrt{2\delta}\cdot\sqrt{1-\delta/2}\;=\;2\sqrt{2\delta}\,\bigl(1-O(\delta)\bigr),
\]
so the rate $\sqrt\delta$ and the constant $2\sqrt2$ are exact.
\end{itemize}
\end{theorem}

\begin{proof}
(i) First inequality: Cauchy--Schwarz with $\norm{\img(I_s)}=1$. Second: for any unit $x$
with $\ip{x}{\bar t_s}\ge1-\delta$,
\[
\norm{x-\bar t_s}^2=2-2\ip{x}{\bar t_s}\le2\delta,
\]
so both variants lie within $\sqrt{2\delta}$ of the anchor, and the triangle inequality
through $\bar t_s$ gives $\norm{\txt(c(s))-\txt(c(\sigma s))}\le2\sqrt{2\delta}$.
(The exact diameter of the cap is $2\sqrt{2\delta-\delta^2}$; we quote the clean form.)

(ii) Work in a $2$-plane spanned by orthonormal $\bar t,w$; put
$t_{1,2}=\cos\theta\,\bar t\pm\sin\theta\,w$ with $\cos\theta=1-\delta$. Declare $\bar t$
the anchor (it is a fixed vector; in the full-system construction of
Proposition~\ref{prop:dichotomy}(ii) it can be realized as the embedding of the atom-list
caption, which carries no pairing information). Both variants satisfy the smoothness
condition with equality (Lemma~\ref{lem:midpoint} is this configuration seen from the midpoint; Figure~\ref{fig:frontier} draws the cap geometry and the induced Pareto region). Take the image vector $\img:=w$. Then
\[
M(s)=\ip{w}{t_1-t_2}=2\sin\theta=2\sqrt{1-(1-\delta)^2}=2\sqrt{2\delta-\delta^2}. \qedhere
\]
\end{proof}

\begin{corollary}[Per-item ceiling; anchor-free]\label{cor:ceiling}
For every scene and every image map,
\[
\bigl|M(s)\bigr|\;\le\;d(s):=\norm{\txt(c(s))-\txt(c(\sigma s))}.
\]
\end{corollary}

\begin{figure}[t]
\centering
\begin{minipage}[c]{0.40\textwidth}
\centering
\begin{tikzpicture}[scale=1.28, font=\small,
  arr/.style={-{Stealth[length=2mm]}, black!70}]
\draw[black!40] (0,0) circle (1.6);
\coordinate (O) at (0,0);
\draw[very thick, black!85] (56:1.6) arc (56:124:1.6);
\draw[arr] (O) -- (90:1.6) node[above, yshift=1pt] {$\bar t_s$};
\draw[black!50, dotted] (O) -- (56:1.6);
\draw[black!50, dotted] (O) -- (124:1.6);
\fill[blue!65!black]   (68:1.6) circle (1.3pt) (74:1.6) circle (1.3pt) (80:1.6) circle (1.3pt);
\fill[orange!85!black] (103:1.6) circle (1.3pt) (110:1.6) circle (1.3pt) (117:1.6) circle (1.3pt);
\node[blue!65!black,   font=\scriptsize] at (60:1.95) {$\mathcal C(s)$};
\node[orange!85!black, font=\scriptsize] at (121:1.92) {$\mathcal C(\sigma s)$};
\draw[dashed, black!70] (80:1.6) -- (103:1.6);
\node[font=\scriptsize, black!80] at (91.5:1.22) {$d\le 2\sqrt{2\delta}$};
\node[font=\scriptsize, black!55, align=center] at (0,-0.55)
  {cap $\ip{x}{\bar t_s}\ge 1-\delta$\\ (thick arc)};
\node[font=\small] at (0,-2.2) {(a)};
\end{tikzpicture}
\end{minipage}\hfill
\begin{minipage}[c]{0.56\textwidth}
\centering
\begin{tikzpicture}
\begin{axis}[width=7.4cm, height=6.4cm, axis lines=left,
  xlabel={smoothness gap $\delta$}, ylabel={binding margin $M(s)$},
  xmin=0, xmax=0.5, ymin=0, ymax=2.05,
  xtick={0,0.1,0.2,0.3,0.4,0.5}, tick label style={font=\scriptsize},
  label style={font=\scriptsize}, domain=0:0.5, samples=120]
\addplot[name path=front, thick, black] {2*sqrt(2*x)};
\path[name path=xaxis] (axis cs:0,0) -- (axis cs:0.5,0);
\addplot[black!8] fill between[of=front and xaxis];
\node[font=\scriptsize, black!70] at (axis cs:0.36,0.72) {achievable};
\node[font=\scriptsize, black!70, align=center] at (axis cs:0.14,1.80)
  {forbidden: $M>2\sqrt{2\delta}$};
\node[font=\scriptsize, anchor=south, rotate=27] at (axis cs:0.30,1.50)
  {frontier $2\sqrt{2\delta}$};
\fill[black] (axis cs:0,0) circle (1.4pt);
\node[font=\scriptsize, anchor=west] at (axis cs:0.012,0.10) {$\delta=0$: exact collapse};
\addplot[only marks, mark=*, mark size=1.8pt, blue!65!black] coordinates {(0.112,0.237)};
\addplot[only marks, mark=o, mark size=1.8pt, black!60] coordinates {(0.112,0.947)};
\draw[dashed, black!60] (axis cs:0.112,0.237) -- (axis cs:0.112,0.947);
\node[font=\scriptsize, anchor=west, blue!65!black] at (axis cs:0.128,0.26) {CLIP B/32: $d$};
\node[font=\scriptsize, anchor=west, black!60] at (axis cs:0.128,0.93) {its cap};
\end{axis}
\node[font=\small] at (3.2,-1.45) {(b)};
\end{tikzpicture}
\end{minipage}
\caption{The smoothness--binding frontier. \textbf{(a)} Proof geometry of
Theorem~\ref{thm:E}(i) (drawn as a great-circle section): both caption classes lie in the
spherical cap of cosine-radius $1-\delta$ about the pairing-neutral anchor $\bar t_s$; any
two points of the cap are at distance $\le2\sqrt{2\delta}$, and by Cauchy--Schwarz no unit
image vector can extract a larger margin from them
(Corollary~\ref{cor:classfrontier} is exactly this picture).
\textbf{(b)} The induced Pareto region: pairs $(\delta,M)$ below the curve are achievable
(Theorem~\ref{thm:E}(ii) saturates it), pairs above are impossible for \emph{any} encoder,
dimension, or training. Kang et al.'s exact collapse is the $\delta=0$ endpoint. The
measured CLIP~B/32 point sits at $25\%$ of its cap---far inside; the dashed segment is the
unused smoothness budget (\S\ref{sec:zoo}). The open circle sits on the curve at
$2\sqrt{2\cdot0.112}=0.947$; the table of \S\ref{sec:zoo} reports the mean of per-quartet
caps, $0.936$---a Jensen gap.}
\label{fig:frontier}
\end{figure}
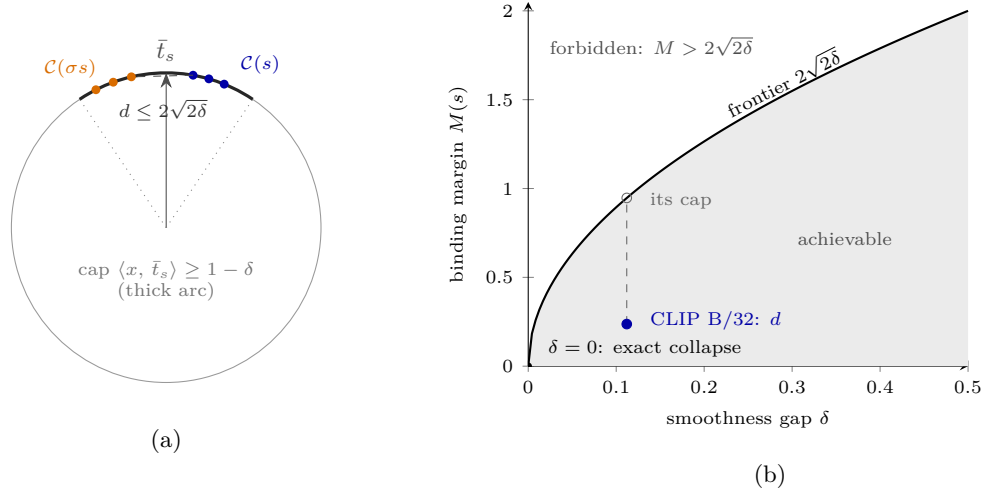

The corollary is trivial (it is the Cauchy--Schwarz half alone) but operationally central:
$d(s)$ is computable \emph{from the text tower alone}, yet it caps the margin of the full
image--text system on that item. \S\ref{sec:bench} tests precisely this quantity on a real
benchmark.

\begin{corollary}[Class-level frontier]\label{cor:classfrontier}
Suppose the caption map is $\delta$-smooth \emph{as a class} at $s$: every realization in
$\mathcal C(s)\cup\mathcal C(\sigma s)$ lies within cosine $1-\delta$ of a pairing-neutral
anchor $\bar t_s$. Then for every image realization $I\in\mathcal I(s)$ and every
cross-class pair $(c^{+},c')\in\mathcal C(s)\times\mathcal C(\sigma s)$,
\[
\ip{\img(I)}{\txt(c^{+})-\txt(c')}\;\le\;2\sqrt{2\delta}.
\]
\emph{(Proof: verbatim Theorem~\ref{thm:E}(i) applied to the pair.)} This is the formal
version of the two-jobs tension of \S\ref{sec:twojobs}: the similarity job asks each class to be
\emph{compact} (small spread around its meaning), the binding job asks the two classes to
be \emph{separated}---and both classes share the lexical neighborhood of the same anchor,
so compactness of the union caps the separation. The anchor-as-paraphrase-centroid used in
the measurements of \S\ref{sec:zoo} is the natural instance, not a convenience.
\end{corollary}

\begin{remark}[Scope: what the theorem does and does not say]\label{rem:scope}
(1) The bound is \emph{per scene} and \emph{anchor-relative}; its empirical content on a
model is only as strong as the independence of the anchor used to measure $\delta$.
(2) Being far inside the frontier is not a virtue and being on it is not a vice: the theorem
constrains the \emph{jointly achievable} pairs $(\delta,M)$, and where a model should sit
depends on its deployment mix of the two jobs.
(3) Nothing is claimed about \emph{which} models approach the bound---that is an empirical
question, answered in \S\ref{sec:zoo}: none do, yet.
(4) The exact-collapse claim of Kang et al.\ is recovered as the endpoint $\delta=0$
(anchor equal to both variants forces $M(s)\le0$); the frontier is its finite-$\delta$,
two-sided completion.
\end{remark}

\subsection{The diagnostic}

Theorem~\ref{thm:E} induces a text-only measurement. For a model and a scene family:
\[
d\;=\;\norm{\txt(c(s))-\txt(c(\sigma s))},\qquad
\delta\;=\;1-\min\bigl(\ip{\txt(c(s))}{\bar t_s},\,\ip{\txt(c(\sigma s))}{\bar t_s}\bigr),
\]
\[
\mathrm{cap}\;=\;2\sqrt{2\delta},\qquad
\text{\emph{frontier usage}}\;=\;\frac{d}{\mathrm{cap}}\in[0,1].
\]
Usage $\approx1$: the model spends its entire smoothness budget on separating the pairings
---binding is \emph{smoothness-limited}, and further binding gains must be paid for in
paraphrase geometry. Usage $\ll1$: the ceiling is far away, and binding is limited by
something else---by Paper~1's results, by the code and the data. Which regime deployed
models occupy is an empirical question.

\section{Measurement I: the frontier across the model zoo}\label{sec:zoo}

\subsection{Protocol}

Scene family: a \emph{quartet} is a choice of two objects and two attributes,
determining a $2$-scene and its swap; all $\binom{10}{2}\binom{6}{2}=45\cdot15=675$
quartets from $10$ objects $\times$ $6$ color attributes
(2-object scenes; captions ``a \{$a_1$\} \{$o_1$\} and a \{$a_2$\} \{$o_2$\}'' and the
attribute swap). Anchor: the embedding of the pairing-neutral atom-list caption
``a photo of a \{$o_1$\} and a \{$o_2$\}, \{$a_1$\} and \{$a_2$\}''---constructed from the
symbol multiset only. Every model receives the identical quartet grid and anchor template;
text towers are $\ell_2$-normalized (the geometry of Definition~\ref{def:encoder}).
In Table~\ref{tab:zoo}, `ft' abbreviates fine-tune, `LoRA-$4$' a rank-$4$ low-rank
adapter, CC3M the Conceptual Captions 3M corpus, and the `max' column is the model's
largest per-quartet usage. Aggregate usage is $\overline d/\overline{\mathrm{cap}}$ with a $2000$-resample bootstrap
$95\%$ CI. Implementation checks for NegCLIP and the LoRA models (SVLC \cite{tsvlc};
DAC \cite{dac}) are in \S\ref{sec:zoo-rigor}.

\subsection{Results}

The full grid results appear as Table~\ref{tab:zoo} in the main text;
Figure~\ref{fig:zoo} plots them.

\subsection{Findings}

\begin{enumerate}
\item \textbf{A universal band, far inside the frontier} (Figure~\ref{fig:zoo}). Thirteen base models spanning
three training objectives (InfoNCE softmax, sigmoid, and ALIGN's noisy InfoNCE), five
pretraining corpora, two text architectures (GPT-style Transformer, BERT), and a
$10\times$ scale range all land at usage $0.25$--$0.32$. \emph{No base model is within
$3\times$ of the ceiling} (NegCLIP, at usage $0.349$, reaches $2.9\times$). Binding failure in the current zoo is therefore \emph{not}
smoothness-limited---consistent with, and required by, Paper~1's account (code and data are
the binding constraints today). The frontier is the \emph{terminal} obstruction: the one
that will remain once those are fixed.
\item \textbf{Scale moves models along the frontier, not toward it.} Larger towers increase
$d$ and cap together (CLIP B/32 $\to$ L/14: $d$ $0.237\to0.293$, cap $0.936\to1.131$),
leaving usage nearly fixed.
\item \textbf{Real caption geometry traces the predicted direction.} Within every model,
quartets with larger $d$ sit at larger $\delta$ (correlation $+0.10$ to $+0.82$, positive
in all 18 models)---directionally consistent with the frontier's $d$--$\delta$ coupling
(no exponent was fitted).
\item \textbf{Mechanism separation in the fine-tuned family.} Usage can rise two ways:
raising $d$ (more binding signal) or shrinking the cap (smoother anchors, lower ceiling).
Only NegCLIP---a \emph{full} fine-tune with caption-level swap negatives, i.e.\ Paper~1's
throttle applied in the wild---raises $d$ ($0.237\to0.320$ against its architecture-matched
baseline at flat $\delta$), climbing \emph{toward} the frontier ($0.253\to0.349$, disjoint
CIs). Three of the four LoRA rank-4 compositional fine-tunes (DAC-SAM, both SVLC) leave
$d$ at or below baseline and move usage through $\delta$; DAC-LLM raises $d$ modestly
($0.262$) at flat $\delta$. \emph{Caveat}: NegCLIP differs from the LoRA
models in tuning depth, data, and negative type simultaneously; the adapter-capacity
explanation is a hypothesis, not a conclusion.
\end{enumerate}

\subsection{Measurement rigor}\label{sec:zoo-rigor}

Three checks guard the comparisons. (i) \emph{Architecture matching}: NegCLIP fine-tunes
OpenAI's QuickGELU ViT-B/32; loading it into a GELU architecture silently corrupts outputs,
so both it and its baseline use the QuickGELU graph. (ii) \emph{Framework neutrality}: the
OpenAI baseline evaluated through the open\_clip stack reproduces the HuggingFace-stack
number to four decimals (usage $0.2534$ in both). (iii) \emph{LoRA correctness gate}: DAC
and SVLC ship as rank-4 LoRA adapters whose fused-QKV forward cannot be reproduced by naive
weight merging; we run the authors' own fork, and verify that \emph{zeroing the adapters
reproduces the baseline} (usage $0.255$ on the $150$-quartet gate subset, equal to the
fork's own baseline there; the full-grid baseline is $0.253$)---so the reported deltas are
attributable to the adapters and not to the loading path.

\section{Measurement II: the ceiling on a real benchmark}\label{sec:bench}

Corollary~\ref{cor:ceiling} says $|M(s)|\le d(s)$ \emph{per item}, with $d(s)$ computable
from text alone. If the theory describes reality, $d$ should behave like a ceiling on real
benchmark decisions: not a predictor of success, but a bound---low-$d$ items \emph{must}
sit near chance (their margin is squeezed to noise), high-$d$ items are \emph{permitted}
to be solved. We test this on SugarCrepe \cite{sugarcrepe}.

\subsection{Protocol}

All five \textsc{swap}/\textsc{replace} subsets of SugarCrepe (add omitted; Remark~\ref{rem:negatives}): $4757$ items, each an image
with a positive caption and one hard negative; $1542$ distinct COCO val2017 images.
Fourteen of the eighteen models (four skipped for CPU compute: OpenCLIP L/14 and H/14,
SigLIP-large, and CLIP L/14-336, whose text tower duplicates CLIP L/14; the scale axis is
covered by CLIP L/14). For each model and item we record the actual decision
quantity and its text-side ceiling:
\[
\mathrm{margin}=\ip{\img(I)}{\txt(c_{\mathrm{pos}})}-\ip{\img(I)}{\txt(c_{\mathrm{neg}})},
\qquad
d=\norm{\txt(c_{\mathrm{pos}})-\txt(c_{\mathrm{neg}})},
\]
correct iff margin $>0$. Sanity anchor: CLIP B/32 reproduces published SugarCrepe accuracy
($0.775$ overall; swap-object $0.61$). In total $14\times4757=66{,}598$ decisions.

\subsection{Three levels of validation}

\textbf{Subset level: $d$ explains the benchmark's difficulty ordering ($r=0.990$;
Figure~\ref{fig:bench}b).} Pooling across models:

\begin{center}
\small
\begin{tabular}{lcc}
\toprule
SugarCrepe subset & mean $d$ & accuracy \\
\midrule
swap-object      & 0.266 & 0.649 \\
swap-attribute   & 0.305 & 0.684 \\
replace-relation & 0.339 & 0.701 \\
replace-attribute& 0.410 & 0.828 \\
replace-object   & 0.490 & 0.929 \\
\bottomrule
\end{tabular}
\end{center}

\noindent Pearson correlation of subset mean-$d$ with subset accuracy: $0.990$. The reading
is structural: swap negatives use the \emph{identical vocabulary} as the positive
(Definition~\ref{def:negatives}), so their $d$ is intrinsically small; replace negatives
change a word, so their $d$ is large. \emph{The canonical observation that swaps are the
hard subsets of SugarCrepe is, to within $r=0.99$, the statement that swaps are the
small-ceiling subsets.}

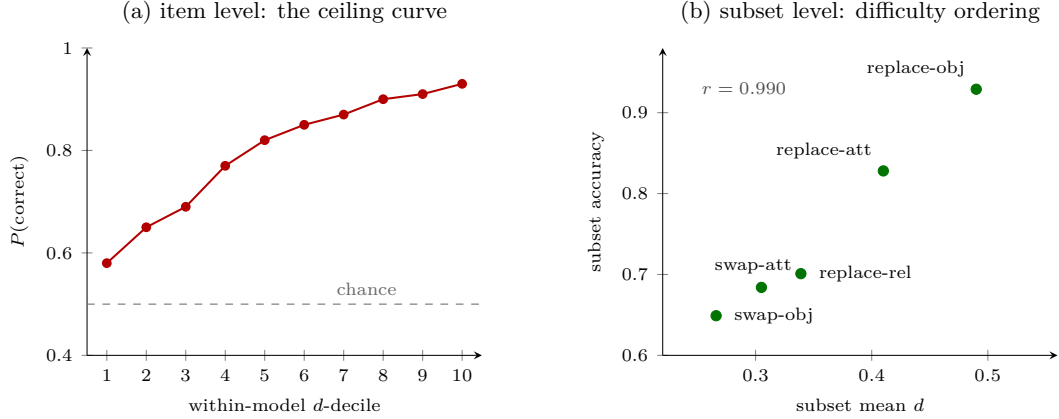
\begin{figure}[t]
\centering
\begin{adjustbox}{max width=\linewidth}\begin{tikzpicture}
\begin{axis}[width=7.0cm, height=5.8cm, axis lines=left,
  xlabel={within-model $d$-decile}, ylabel={$P(\mathrm{correct})$},
  xmin=0.5, xmax=10.5, ymin=0.4, ymax=1.0, xtick={1,2,...,10},
  tick label style={font=\scriptsize}, label style={font=\scriptsize},
  title={(a) item level: the ceiling curve}, title style={font=\small}]
\addplot[thick, mark=*, mark size=1.5pt, red!70!black] coordinates
  {(1,0.58)(2,0.65)(3,0.69)(4,0.77)(5,0.82)(6,0.85)(7,0.87)(8,0.90)(9,0.91)(10,0.93)};
\addplot[dashed, black!60] coordinates {(0.5,0.5)(10.5,0.5)};
\node[font=\scriptsize, black!60, anchor=west] at (axis cs:6.6,0.53) {chance};
\end{axis}
\begin{axis}[at={(7.9cm,0)}, width=7.0cm, height=5.8cm, axis lines=left,
  xlabel={subset mean $d$}, ylabel={subset accuracy},
  xmin=0.22, xmax=0.56, ymin=0.60, ymax=0.98,
  tick label style={font=\scriptsize}, label style={font=\scriptsize},
  title={(b) subset level: difficulty ordering}, title style={font=\small}]
\addplot[only marks, mark=*, mark size=2pt, green!45!black] coordinates
  {(0.266,0.649)(0.305,0.684)(0.339,0.701)(0.410,0.828)(0.490,0.929)};
\node[font=\scriptsize, anchor=west] at (axis cs:0.274,0.649) {swap-obj};
\node[font=\scriptsize, anchor=south] at (axis cs:0.298,0.690) {swap-att};
\node[font=\scriptsize, anchor=west] at (axis cs:0.347,0.701) {replace-rel};
\node[font=\scriptsize, anchor=south east] at (axis cs:0.408,0.832) {replace-att};
\node[font=\scriptsize, anchor=south east] at (axis cs:0.488,0.933) {replace-obj};
\node[font=\scriptsize, black!70] at (axis cs:0.29,0.93) {$r=0.990$};
\end{axis}
\end{tikzpicture}\end{adjustbox}
\caption{The per-item ceiling (Corollary~\ref{cor:ceiling}) on SugarCrepe, $14$ models,
$66{,}598$ decisions. \textbf{(a)} Accuracy by within-model $d$-decile, pooled: monotone
from near-chance to $0.93$; $73\%$ of all errors lie below the median $d$. The low end is
the theorem (small $d\Rightarrow$ margin within noise); the high end is
permitted-and-observed, not implied. \textbf{(b)} Pooled subset accuracy against subset
mean $d$: the benchmark's canonical difficulty ordering---swaps hardest---is the
statement that same-vocabulary negatives have small ceilings.}
\label{fig:bench}
\end{figure}

\medskip
\textbf{Item level: a monotone ceiling curve} (Figure~\ref{fig:bench}a). Within each
model, rank items by $d$ into deciles; pool the accuracy of each decile across models
($66{,}598$ decisions):
\[
0.58,\;0.65,\;0.69,\;0.77,\;0.82,\;0.85,\;0.87,\;0.90,\;0.91,\;0.93
\qquad\text{(deciles $1\to10$)} .
\]
Monotone from near-chance to near-ceiling, $+35$ points. Errors localize accordingly:
$73\%$ of all incorrect decisions occur on below-median-$d$ items. We stress the logically
correct, one-sided reading: the low end is the theorem \emph{plus a noise model} (small
$d$ caps $|M|$ by Corollary~\ref{cor:ceiling}; near-chance additionally assumes that
sub-ceiling margins are sign-symmetric---an empirical regularity, not part of the
corollary) (small $d$ $\Rightarrow$
margin within noise $\Rightarrow$ near-chance); the high end is \emph{permitted-and-
observed}, not implied---$d$ is necessary, not sufficient.

\medskip
\textbf{Model level: the ceiling logic, not naive correlation.} Correlating model-level
predictors with model accuracy ($n=14$):

\begin{center}
\small
\begin{tabular}{lcc}
\toprule
Predictor & Pearson $r$ (overall acc) & Pearson $r$ (swap acc) \\
\midrule
usage $d/\mathrm{cap}$ (frontier position)     & 0.45 & \textbf{0.59} \\
$d$ on the synthetic probe (\S\ref{sec:zoo})   & 0.40 & 0.26 \\
raw benchmark mean-$d$                         & 0.07 & $-0.15$ \\
$\delta$ (smoothness)                          & 0.09 & $-0.16$ \\
\bottomrule
\end{tabular}
\end{center}

\begin{itemize}
\item \emph{Raw cross-model $d$ predicts nothing}---models differ in global spread (CLIP
L/14 spreads all captions more without binding better). Normalizing by the cap removes
exactly that scale, and \emph{usage} becomes the best text-only cross-model signal for swap
accuracy.
\item \emph{$\delta$ predicts nothing---and must not.} Every model in \S\ref{sec:zoo} is far
inside the frontier, so its smoothness level is not its binding constraint. The null is the
\emph{regime test} passing: were a model on the frontier, $\delta$ would become the
predictor.
\item Case studies make necessity-vs-sufficiency concrete. \textbf{NegCLIP} raises the
ceiling and converts it into accuracy (swap $0.61\to0.76$). \textbf{DAC} binds well
($0.850$/$0.868$ overall) at (DAC-SAM) or barely above (DAC-LLM) the baseline
ceiling---better exploitation of the existing text separation by the image side, which
Corollary~\ref{cor:ceiling} permits and text-side $d$ cannot see. \textbf{SVLC-LLM+RB}, the
model whose $d$ fell \emph{furthest} below baseline ($0.202$), records the lowest overall
accuracy ($0.745$): the ceiling coming down is
visible in the benchmark.
\end{itemize}

\subsection{Caveats}
The quartet probe is synthetic 2-object color--object binding; naturalistic multi-relation
captions are future work. The anchor is one template; an anchor-robustness appendix
(a second independent anchor on several models) is planned before publication. All results
are text-side or single-benchmark; no claim is made that $d$ ranks models on unrelated
capabilities.

\section{Synthesis, related work, open problems}\label{sec:synthesis}

\subsection{The obstruction map: three contingent gates, one terminal wall}

\begin{center}
\small
\begin{adjustbox}{max width=\linewidth}
\begin{tabular}{lllll}
\toprule
Obstruction & Hypothesis added & Result & Strength & Status \\
\midrule
Geometric   & none (free placements) & satisfiable iff $k\le\lfloor N/2\rfloor$ & --- & Cited, \S\ref{sec:possibility} \\
Structural  & pairing-blind code     & margin $\equiv 0$, exact & absolute & Proved, \S\ref{sec:thmB} \\
Structural  & pooled code, leakage $L$ & margins $=\Theta(1/k)$, pinned & quantitative & Proved, \S\ref{sec:thmC} \\
Statistical & trained on $(p,Q)$     & swap signal $\propto\pi$; coverage law & quantitative & from \cite{paper1}, \S\ref{sec:thmD} \\
Semantic    & $\delta$-smooth anchors & $M(s)\le2\sqrt{2\delta}$, tight & terminal & Proved, \S\ref{sec:frontier} \\
\bottomrule
\end{tabular}
\end{adjustbox}
\end{center}

\begin{figure}[t]
\centering
\begin{tikzpicture}[
  font=\footnotesize,
  gatebox/.style={draw=teal!55!black, line width=0.7pt, fill=teal!8,
                  rounded corners=2pt, align=center, inner sep=3pt, text width=2.75cm},
  wallbox/.style={draw=black!80, line width=0.9pt, fill=black!4,
                  rounded corners=2pt, align=center, inner sep=3pt, text width=2.75cm},
  neutral/.style={black!45},
]

\def\ylane{1.55}      
\def\laneT{2.15}      
\def\laneB{0.95}      
\def\boxY{3.20}       
\def\wallL{11.45}\def\wallR{12.95}      
\def\pathL{0.55}      
\def\gA{2.0}\def\gB{5.4}\def\gC{8.8}    
\def\fX{12.2}         

\fill[black!3, rounded corners=3pt] (\pathL,\laneB) rectangle (\wallL,\laneT);
\draw[neutral, rounded corners=3pt, line width=0.5pt] (\pathL,\laneB) rectangle (\wallL,\laneT);

\draw[neutral, line width=1.5pt, -{Stealth[length=3.4mm]}]
  (\pathL+0.20,\ylane) -- (\wallL-0.05,\ylane);

\foreach \gx/\gt/\gc/\gr in {%
  \gA/{GEOMETRIC}/{need $k\le\lfloor N/2\rfloor$}/{removable: enlarge $N$ past $2k$ (\S3)},%
  \gB/{STRUCTURAL}/{need leakage $\ell\approx0$ (factored code)}/{removable: choose the code (\S4--5)},%
  \gC/{STATISTICAL}/{need $\pi>0$ \& pair coverage}/{removable: training data (\S6)}%
}{
  \node[gatebox] (gb) at (\gx,\boxY)
    {{\color{teal!45!black}\scriptsize\bfseries \gt}\\[1pt]\gc\\[2pt]%
     {\color{teal!45!black}\tiny\itshape \gr}};
  \draw[teal!55!black, line width=0.6pt] (\gx,\laneT) -- (\gx,\laneT+0.22);
  \draw[teal!55!black, dashed, line width=0.9pt] (\gx,\laneT-0.03) -- (\gx,\ylane+0.28);
  \draw[teal!55!black, dashed, line width=0.9pt] (\gx,\laneB+0.03) -- (\gx,\ylane-0.28);
  \draw[teal!55!black, dashed, line width=0.9pt, -{Stealth[length=1.8mm]}]
    (\gx,\ylane+0.28) to[out=22,in=112] (\gx+0.52,\ylane+0.10);
  \node[circle, draw=teal!55!black, fill=teal!12, inner sep=0pt,
        minimum size=3.6mm, line width=0.5pt] at (\gx-0.34,\ylane-0.02) {};
  \draw[teal!45!black, line width=0.7pt, line cap=round]
    (\gx-0.42,\ylane-0.02) -- (\gx-0.355,\ylane-0.10) -- (\gx-0.25,\ylane+0.08);
}

\node[wallbox] at (\fX,\boxY)
  {{\color{black!80}\scriptsize\bfseries SEMANTIC FRONTIER}\\[1pt]%
   $M(s)\le 2\sqrt{2\delta}$\\[2pt]%
   {\tiny \emph{every} encoder,\\ every $N$, every training}};
\draw[black!80, line width=0.6pt] (\fX,\laneT) -- (\fX,\laneT+0.16);
\fill[black!5] (\wallL,\laneB) rectangle (\wallR,\laneT);
\foreach \yy in {1.19,1.43,1.67,1.91}{
  \draw[black!35, line width=0.4pt] (\wallL,\yy) -- (\wallR,\yy);}
\foreach \yy/\off in {1.07/0, 1.31/1, 1.55/0, 1.79/1, 2.03/0}{
  \foreach \k in {0,1}{
    \draw[black!35, line width=0.4pt]
      ({\wallL+0.375+0.75*\k+0.375*\off},\yy-0.12) -- ({\wallL+0.375+0.75*\k+0.375*\off},\yy+0.12);}}
\draw[black!80, line width=1.2pt] (\wallL,\laneB) rectangle (\wallR,\laneT);
\node[fill=black!80, text=white, rounded corners=1.2pt, inner sep=1.6pt,
      font=\tiny\bfseries] at (\fX,\ylane) {IMMOVABLE};

\draw[decorate,decoration={brace,amplitude=4pt,raise=1pt},teal!55!black,line width=0.6pt]
  (0.55,\boxY+0.92) -- (10.30,\boxY+0.92);
\node[teal!45!black, font=\scriptsize\bfseries] at (5.4,\boxY+1.34)
  {THREE CONTINGENT GATES \textnormal{\;--\; engineer away}};
\node[black!80, font=\scriptsize\bfseries] at (\fX,\boxY+1.34) {ONE TERMINAL WALL};

\def\axY{0.42}
\pgfmathsetmacro{\axspan}{\wallL-\pathL}
\pgfmathsetmacro{\measX}{\pathL+0.30*\axspan}
\draw[neutral, line width=0.5pt] (\pathL,\axY) -- (\wallL,\axY);
\fill[teal!55!black] (\pathL,\axY-0.05) rectangle (\measX,\axY+0.05);
\draw[neutral, line width=0.7pt] (\pathL,\axY-0.10) -- (\pathL,\axY+0.10);
\draw[black!80, line width=1.0pt] (\wallL,\axY-0.12) -- (\wallL,\axY+0.12);
\draw[teal!45!black, line width=0.6pt] (\measX,\axY-0.05) -- (\measX,\axY-0.22);
\node[teal!35!black, font=\tiny, anchor=north] at ({(\pathL+\measX)/2},\axY-0.20)
  {today's models $\approx 25\text{--}35\%$ (\S8)};
\node[neutral, font=\tiny] at ({(\measX+\wallL)/2},\axY+0.24)
  {$\longleftarrow$ unclaimed headroom to the frontier $\longrightarrow$};
\node[black!70, font=\tiny, anchor=south east] at (\wallL-0.02,\axY+0.13) {frontier};

\end{tikzpicture}
\caption{The three-obstruction map (Section~10). Binding faces four obstructions.
Three are \emph{contingent} and can be engineered away---enlarge $N$ past $2k$
(geometry, Prop.~\ref{prop:dichotomy}), choose a low-leakage factored code
(structure, \S\ref{sec:thmC}), and supply swap negatives with pair coverage
(data, \S\ref{sec:thmD}). The fourth, the smoothness frontier
$M(s)\le 2\sqrt{2\delta}$ (Thm~\ref{thm:E}), binds \emph{every} encoder, dimension,
and training: it is terminal. The measurements (\S\ref{sec:zoo}) place today's models
far short of it, which is why the contingent obstructions govern the near term and the
frontier is the limit that remains.}
\label{fig:obstructions}
\end{figure}
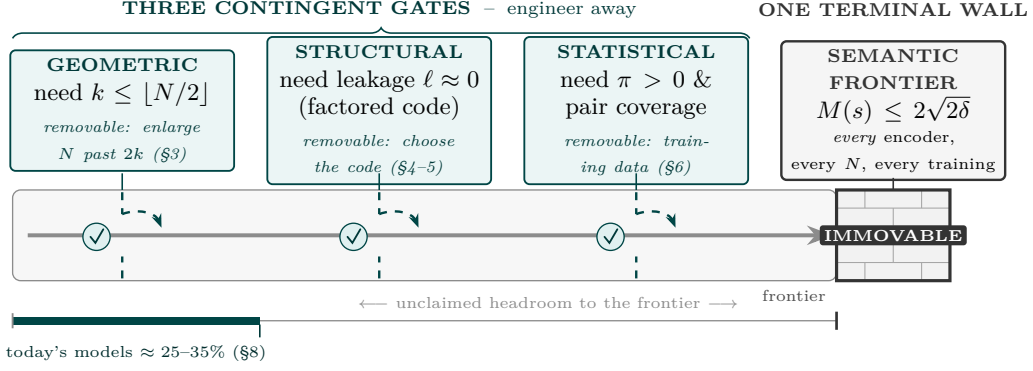

The first three obstructions are \emph{contingent} (Figure~\ref{fig:obstructions}): they can be engineered away (choose a
factored code, raise $\pi$, cover the pairs). The frontier cannot: it binds \emph{any}
encoder, \emph{any} dimension, \emph{any} training, as long as the space is asked to keep
near-paraphrases close. That is the sense in which it is terminal---and the measurements say
no current model is near it, which returns the near-term agenda to the contingent
obstructions; the frontier becomes the binding constraint only after those are repaired.

\subsection{Relation to Paper 1}
Paper~1 (``It is the code, not the pooling'') establishes the contingent story
quantitatively: the five-arm code ablation and the partial-match-leakage mechanism
(here \S\ref{sec:thmC}'s measured counterpart), and the $\pi$-throttle with its
dose--response on real data (here Proposition~\ref{prop:D}). This paper supplies the
principled complement: the axiom system those experiments implicitly assume, the
theorems that organize the code classes, and the terminal bound that none of the
engineering can move. There is no claim overlap: Paper~1 never states the frontier;
this paper claims none of Paper~1's measurements as new.

\subsection{Related work and differentiation}
Three recent lines prove limits on dual encoders by \emph{other} mechanisms, and one
documents our tension empirically:
\begin{itemize}
\item \textbf{Dimension / capacity.} LIMIT \cite{limit} bounds which retrieval patterns fit
in dimension $N$ via thresholdable-rank arguments (sign-rank in their appendix); \cite{med} pins the minimal dimension for top-$k$ retrieval
(the mathematics of \S\ref{sec:possibility}). Orthogonal axis: those bounds vanish as
$N$ grows; the frontier does not involve $N$ at all.
\item \textbf{Resolution.} ``Bound by semanticity'' \cite{semanticity} shows
identification--generalization limits from finite semantic resolution, assuming
non-compositional representations; no swap/margin content, no converse.
\item \textbf{Optima structure.} Chen et al.~\cite{chen} prove existence of
swap-insensitive contrastive optima; existence only, no quantitative bound, no
measurement.
\item \textbf{Empirics.} SugarCrepe++ \cite{sugarcrepepp} documents the
paraphrase/hard-negative dissociation that Theorem~\ref{thm:E} formalizes; no bound.
\end{itemize}
Ours is, to our knowledge, the only \emph{tight two-sided} trade-off among these, and the
only one shipped with a measured frontier position for the deployed zoo. The upper bound's
proof is elementary---we present that as a feature, not a weakness: the content is in the
formulation (the anchor quantifier that evades tautology), the converse, and the
measurements.

\subsection{Open problems}
\begin{enumerate}
\item \textbf{Margin-robust dimension.} $(K_k)$ with margin $\gamma$ changes
\S\ref{sec:possibility}'s answer from $2k$ to (in related settings) $\Theta(k\log(M/k))$
\cite{robust}; the exact constant for the constituent-ranking family $(K_k)$ appears open.
\item \textbf{Smoothness vs.\ neighborliness.} Kang's optional similarity condition
(attribute variants cluster near their object) plausibly conflicts with the neighborliness
that $(K_k)$ requires at some $(M,K,N)$---a candidate \emph{second} impossibility theorem
with real content.
\item \textbf{Breaking the NegCLIP confound.} A LoRA-rank-4 model trained on COCO with
swap negatives would isolate which of NegCLIP's three differences moves $d$.
\item \textbf{Naturalistic frontier.} Extending the diagnostic beyond 2-object color
binding to relations and counting; anchors for multi-relation scenes are not canonical.
\item \textbf{Depth.} The concept system is flat: a scene is a \emph{set} of atoms, and
$k$ counts objects, not nesting. The homomorphism form $(H)$ of \S\ref{sec:homomorphism}
is recursive, and recursion is what makes language productive. Whether the $\Theta(1/k)$ margin
of Theorem~\ref{thm:Cswap} degrades gracefully or catastrophically under nesting was open
when this program began, and is the natural formal home for the empirical observation that
grammar-structured composers help precisely on deeper-than-trained inputs. (Resolved in
Appendix~\ref{sec:depth}: the margin law $2\,b^{-D}$ is proved ---
Theorem~\ref{thm:depthA} --- and measured across three orders of magnitude.)
\item \textbf{On-frontier models.} No current model is smoothness-limited. A model trained
with aggressive $\pi$ and a factored-structure head might be the first to reach the
frontier---at which point $\delta$, not code or data, becomes its binding constraint, and
the trade-off of Theorem~\ref{thm:E} becomes an engineering decision.
\end{enumerate}

\section{The throttle theorems: what the objective pays for binding}\label{sec:throttle}

Proposition~\ref{prop:D}(i) reported, as a ported empirical law, that the trainable swap
margin scales with the swap-negative rate $\pi$. This section proves the underlying
statements. The guiding structural fact --- easy to miss and essential for honesty --- is
that the population InfoNCE optimum with universal capacity \emph{does bind}
(Lemma~\ref{lem:infonce}: the optimum is a pointwise-mutual-information scorer, which
separates swaps). So the rigorous throttle is \emph{not} ``optima do not bind''; it is that
\textbf{the objective's preference for binding over blindness is a sliver whose width is
the swap-contrast rate}, so that anything else --- capacity limits, regularization,
early stopping, leakage-prone parameterizations --- outweighs the objective inside that
sliver. Figure~\ref{fig:sliver} draws the landscape.

\subsection{Setup}

Scenes $s\sim\mathcal D$ (finite support, the concept system of \S\ref{sec:setup});
caption $c(s)$; image $I_s\sim p(\cdot\mid s)$. A \emph{contrastive sample} is a positive
$(I_s,c(s))$ together with a negative caption $c^-$:
with probability $\pi$ the \emph{explicit swap} $c(\sigma s)$; otherwise a \emph{marginal
negative} $c(s')$, $s'\sim\mathcal D$ independent. The \emph{implicit collision rate} is
$\pi_{\mathrm{impl}}:=\Pr_{s,s'}\!\left[[c(s')]=[c(s)]\right]$ --- the chance a marginal
negative lands in the positive's symbol class (Definition~\ref{def:scene}); for web-scale
corpora this is vanishing, and in-batch training has $\pi=0$.

The loss is the one-negative logistic contrastive loss (binary InfoNCE),
\[
L(S)\;=\;\mathbb E\,\ell\big(S(I_s,c(s))-S(I_s,c^-)\big),\qquad \ell(m)=\log(1+e^{-m}),
\]
over \emph{arbitrary measurable scorers} $S$ --- no architecture assumption; every bound
below therefore applies to dual encoders a fortiori. We use: $\ell$ convex, decreasing,
1-Lipschitz, $\ell(0)=\log2$, $0\le\ell(m)\le e^{-m}$ for $m\ge0$, and the exact identity
\begin{equation}\label{eq:logisticid}
\ell(-m)-\ell(m)=m .
\end{equation}

Two comparators are attached to any scorer $S$:
\begin{itemize}
\item the \textbf{class-blind average} $\bar S(I,c):=\sum_{c'\in[c]}w_{[c]}(c')\,S(I,c')$
($w$ = the $\mathcal D$-conditional weights of the class) --- a pairing-blind scorer;
\item the \textbf{reflection} $(\rho S)(I,c):=S(I,\mathrm{swap}(c))$ --- the scorer with
its binding \emph{exactly reversed} (on classes with a unique swap partner).
\end{itemize}
Regularity parameters: \emph{$\Delta$-separation} --- for a.e.\ $(s,I_s)$ and every
negative-support $c'$ with $[c']\ne[c(s)]$, $S(I_s,c(s))-S(I_s,c')\ge\Delta$ (cross-class
contrasts are easy); \emph{within-class spread} $R$ --- $|S(I,c)-S(I,c')|\le R$ for
$c'\in[c]$; and the matched swap margin $M_S(s):=S(I_s,c(s))-S(I_s,c(\sigma s))\le R$.

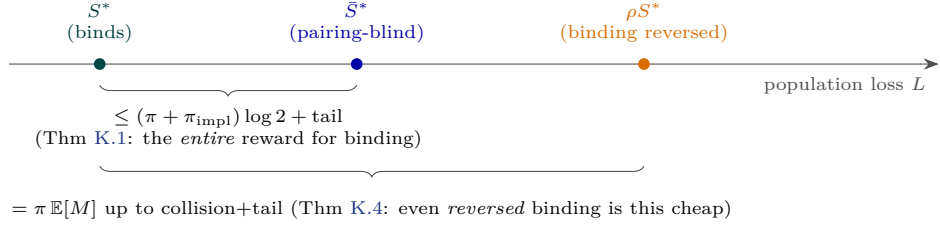
\begin{figure}[t]
\centering
\begin{tikzpicture}[font=\small, arr/.style={-{Stealth[length=2mm]}, black!70}]
\draw[arr] (0,0) -- (12.3,0) node[below left=2pt, font=\scriptsize] {population loss $L$};
\fill[teal!55!black] (1.2,0) circle (2.2pt);
\node[teal!45!black, font=\scriptsize, align=center] at (1.2,0.55) {$S^*$\\ (binds)};
\fill[blue!65!black] (4.6,0) circle (2.2pt);
\node[blue!65!black, font=\scriptsize, align=center] at (4.6,0.55) {$\bar S^*$\\ (pairing-blind)};
\fill[orange!85!black] (8.4,0) circle (2.2pt);
\node[orange!85!black, font=\scriptsize, align=center] at (8.4,0.55) {$\rho S^*$\\ (binding reversed)};
\draw[decorate,decoration={brace,amplitude=4pt,mirror},black!70]
  (1.2,-0.28) -- (4.6,-0.28);
\node[font=\scriptsize, align=center] at (2.9,-0.85)
  {$\le(\pi+\pi_{\mathrm{impl}})\log2+\mathrm{tail}$\\ (Thm~\ref{thm:T1}: the \emph{entire} reward for binding)};
\draw[decorate,decoration={brace,amplitude=4pt,mirror},black!70]
  (1.2,-1.35) -- (8.4,-1.35);
\node[font=\scriptsize, align=center] at (4.8,-1.92)
  {$=\pi\,\mathbb E[M]$ up to collision+tail (Thm~\ref{thm:T2}: even \emph{reversed} binding is this cheap)};
\end{tikzpicture}
\caption{The throttle, schematically. With universal capacity the optimum $S^*$ binds ---
but its pairing-blind average and even its exact binding reversal sit within a loss sliver
whose width is set by the swap-contrast rate $(\pi+\pi_{\mathrm{impl}})$. Training
pressures larger than the sliver, not the objective, decide the binding behaviour of the
selected solution.}
\label{fig:sliver}
\end{figure}

\subsection{The theorems}

\begin{theorem}[Blindness is nearly optimal]\label{thm:T1}
For any scorer $S$ with $\Delta$-separation and spread $R$,
\[
L(\bar S)-L(S)\;\le\;(\pi+\pi_{\mathrm{impl}})\cdot\log 2\;+\;e^{-(\Delta-2R)} .
\]
\end{theorem}
\begin{proof}
Condition on the negative type. \emph{Matched contrasts} ($c^-$ in the positive's class:
explicit swaps w.p.\ $\pi$, implicit collisions w.p.\ $\le\pi_{\mathrm{impl}}$): $\bar S$
scores positive and negative identically, so its loss is exactly $\ell(0)=\log2$, while
$S$'s is $\ge0$; excess $\le\log2$. \emph{Unmatched contrasts}: $\bar S$ moves each
argument by at most $R$, so its margin is $\ge\Delta-2R$ and its loss
$\le\ell(\Delta-2R)\le e^{-(\Delta-2R)}$; excess at most that. Weight the cases.
\end{proof}

\begin{corollary}[The throttle]\label{cor:throttle}
For any population-optimal $S^*$ with the stated regularity, the pairing-blind $\bar S^*$
is $\varepsilon$-optimal with $\varepsilon\le(\pi+\pi_{\mathrm{impl}})\log2+e^{-(\Delta-2R)}$:
\emph{the entire objective value of binding --- everything beyond bag-of-words --- is at
most $(\pi+\pi_{\mathrm{impl}})\log2$ nats plus an exponential tail.} No training
guarantee of the form ``$\varepsilon$-optimality'' can imply binding unless $\varepsilon$
is below this threshold; at web scale the threshold collapses.
\end{corollary}

\begin{corollary}[$\lambda$--$\pi$ equivalence]\label{cor:lambdapi}
Up-weighting matched terms by $\lambda$ multiplies the matched case of the proof exactly as
raising $\pi$ does: the sliver width is $(\lambda\pi+\pi_{\mathrm{impl}})$-scaled. The
empirically measured near-equivalence of the loss-weight dial and the rate dial is a
one-line consequence.
\end{corollary}

\begin{theorem}[{Anti-binding costs exactly $\pi\,\mathbb E[M]$}]\label{thm:T2}
For any scorer $S$ with the regularity above,
\[
\bigl|\,L(\rho S)-L(S)-\pi\,\mathbb E_s[M_S(s)]\,\bigr|\;\le\;\pi_{\mathrm{impl}}\,R\;+\;2e^{-(\Delta-2R)} .
\]
\end{theorem}
\begin{proof}
\emph{Explicit swaps} (w.p.\ $\pi$): $\rho S$'s margin on the pair $(c(s),c(\sigma s))$ is
$-M_S(s)$, so the loss difference is $\ell(-M_S)-\ell(M_S)=M_S(s)$ \emph{exactly}, by
\eqref{eq:logisticid}; expectation contributes $\pi\,\mathbb E[M_S]$. \emph{Implicit
collisions} (w.p.\ $\le\pi_{\mathrm{impl}}$): the same identity applied to a within-class
margin bounds the difference by $R$. \emph{Unmatched}: both scorers' margins are
$\ge\Delta-2R$, so both losses are $\le e^{-(\Delta-2R)}$ and their difference at most
twice that.
\end{proof}

Three measured phenomena from Part I become corollaries: the \textbf{linear
dose--response} in $\pi$ (the preference \emph{is} linear, with slope $\mathbb E[M]$); the
\textbf{$\lambda$--$\pi$ equivalence} (Corollary~\ref{cor:lambdapi}); and \textbf{cheap
below-chance binding} --- the exactly-reversed scorer sits only $\pi\,\mathbb E[M]$ away in
loss, so a small perturbation (leakage, implicit bias) suffices to select it, consistent
with the measured backwards-binding of learned MLP codes (\S\ref{sec:thmC}).

\subsection{Verification, and a loop closed with the frontier}

Monte-Carlo verification: a factored toy world ($M{=}6$, $K{=}4$, all 2-scenes,
$N{=}512$, images = code sum + Gaussian noise, $2\times10^5$ samples per point), scorers
$S^*$ (template), $\bar S^*$, $\rho S^*$, sweeping $\pi\in[0,0.5]$.

\begin{figure}[t]
\centering
\begin{tikzpicture}
\begin{axis}[width=9.2cm, height=6.4cm, axis lines=left,
  xlabel={explicit swap-negative rate $\pi$}, ylabel={loss gap to $S^*$},
  xmin=0, xmax=0.53, ymin=0, ymax=1.45,
  tick label style={font=\scriptsize}, label style={font=\small},
  legend style={font=\scriptsize, at={(0.03,0.97)}, anchor=north west, draw=black!30},
  legend cell align=left]
\addplot[only marks, mark=*, mark size=1.9pt, blue!65!black] coordinates
  {(0,0.18555) (0.05,0.19985) (0.1,0.21454) (0.2,0.24363) (0.3,0.27272) (0.5,0.33052)};
\addlegendentry{$L(\bar S^*)-L(S^*)$ (blind)}
\addplot[only marks, mark=triangle*, mark size=2.4pt, orange!85!black] coordinates
  {(0,0.68191) (0.05,0.74435) (0.1,0.80757) (0.2,0.93348) (0.3,1.06043) (0.5,1.31220)};
\addlegendentry{$L(\rho S^*)-L(S^*)$ (reversed)}
\addplot[blue!65!black, domain=0:0.53, dashed] {0.18555+0.29022*x};
\addlegendentry{fit $c_1+0.290\,\pi$}
\addplot[orange!85!black, domain=0:0.53, dashed] {0.68191+1.26146*x};
\addlegendentry{fit $c_0+1.261\,\pi$}
\end{axis}
\end{tikzpicture}
\caption{Verification of Theorems~\ref{thm:T1}--\ref{thm:T2} (noise $0.6$). Both gaps are
exactly linear in $\pi$. The toy sits in a low-separation regime, so unmatched contrasts
contribute the $\pi$-independent offsets $c_1,c_0$; the decomposition
$\mathrm{gap}(\pi)=c+\pi(\text{matched difference}-c)$ recovers the theorems' matched
constants precisely: $0.290+c_1=0.476$ vs.\ predicted $\log2-\mathbb E[\ell(M^*)]=0.477$,
and $1.261+c_0=1.943$ vs.\ predicted $\mathbb E[M^*]=1.949$ (Monte-Carlo error
$\approx0.005$). At noise $0.25$: $0.5445$ vs $0.5460$ and $1.9455$ vs $1.9511$.}
\label{fig:throttleverify}
\end{figure}
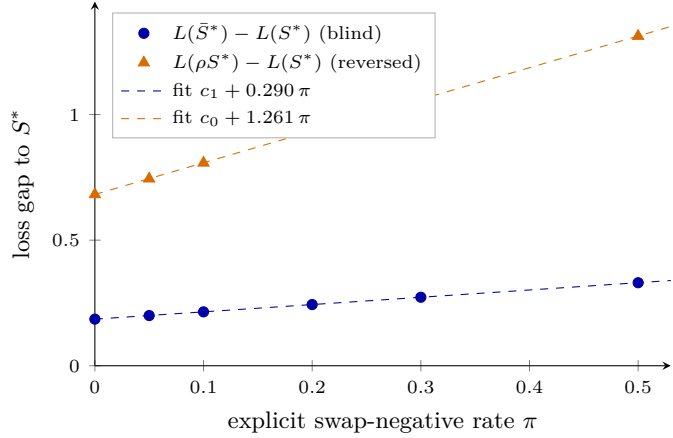

Figure~\ref{fig:throttleverify} shows the result: exact linearity, and the matched-term
constants recovered to three decimals once the unmatched offset is accounted for. The
regime analysis closes an unexpected loop with the frontier. The offsets $c_0,c_1$ are
large in the toy precisely because \emph{template} scorers have large within-class spread
$R$; the clean $\Delta-2R$ regime of the theorems requires $R$ small. But small
within-class spread of real text towers is exactly what the frontier measurements of
\S\ref{sec:zoo} establish ($d\approx0.24$ on the quartet probe): \textbf{the smoothness
that caps binding (Theorem~\ref{thm:E}) is the same measured property that certifies the
throttle bounds are tight for deployed encoders.} One measured quantity feeds two theorems.

\subsection{Honest scope and open items}

Not established: which solution training \emph{selects} inside the sliver (a dynamics or
implicit-bias question; the tempting symmetry route fails because the swap is not a
vocabulary relabeling); the finite-sample version (an $\Omega(1/\pi^2)$ testing lower
bound for detecting the correct binding sign is the natural target); the separation of
instance binding from co-occurrence priors for non-exchangeable $\mathcal D$ (the ``text
prior'' phenomenon); the multi-negative softmax version (the matched event becomes ``swap
present in batch'').

\section{The depth ceiling of recursive composition}\label{sec:depth}

Part I's concept system is flat --- a scene is a \emph{set} of atoms, and every margin in
\S\ref{sec:thmC} is $\Theta(1/k)$ in the \emph{number} of atoms pooled at one level. The
homomorphism form $(H)$ of \S\ref{sec:homomorphism} is recursive, and recursion is what makes language productive. This section reports what happens to binding under \emph{nesting}: the
margins obey an exactly geometric law in depth --- proved below, and measured across three
orders of magnitude --- giving a logarithmic ceiling on the nesting depth a single pooled
vector can resolve.

\subsection{Recursive role-binding}

\begin{definition}[Recursive role--filler composition]\label{def:recursive}
Fix a branching factor $b$ and position roles $r_1,\dots,r_b$ (random sign vectors). A
depth-$D$ structure is a full $b$-ary tree whose leaves carry filler codes (random unit
sign vectors); each internal node's code is the normalized role-bound sum of its children,
\[
\mathrm{code}(v)\;=\;\frac{\sum_{j=1}^b r_j\odot \mathrm{code}(\mathrm{child}_j(v))}
{\bigl\lVert\sum_j r_j\odot \mathrm{code}(\mathrm{child}_j(v))\bigr\rVert}.
\]
This is the classical HRR/VSA recursion \cite{plate,kanerva} and the recursive instance of
the pooled codes of \S\ref{sec:thmC} (depth 1, $b=k$ recovers them). The \emph{deep-edit
margin} at depth $D$ is $m(D)=1-\cos(\mathrm{code}(T),\mathrm{code}(T'))$ for $T'$ a
minimal edit of $T$ (two leaves swapped, or one leaf replaced).
\end{definition}

Two conventions for the roles will matter. In the \emph{shared} scheme the $b$ role
vectors $r_1,\dots,r_b$ are drawn once and reused at every level --- the reading of
Definition~\ref{def:recursive} above, and the classical HRR convention. In the
\emph{fresh} scheme a new independent role $r_{\ell,j}$ is drawn for each level $\ell$
and slot $j$. The exact law below holds for both, wherever no exact algebraic collision
occurs; the finite-dimension theorem is stated for the fresh scheme, and
\S\ref{sec:alias} shows why: shared sign roles are self-inverse, and at certain
depth/branching combinations they \emph{alias} distinct structures exactly. Throughout
we use two elementary facts about sign vectors $\rho\in\{\pm1\}^N$, both immediate from
$\rho_k^2=1$:
\begin{equation}\label{eq:signfacts}
\text{(i)}\quad \ip{\rho\odot x}{\rho\odot y}=\ip{x}{y},
\qquad\qquad
\text{(ii)}\quad \norm{\rho\odot x}=\norm{x}.
\end{equation}
Sign-binding is an inner-product isometry: every distortion in what follows comes from
\emph{superposition} (adding $b$ bound children) and \emph{normalization}, never from the
binding itself.

\subsection{The contraction lemma and the exact law}\label{sec:depthproof}

The engine of the proof is one deterministic lemma: a single level of role-bound,
normalized pooling maps the \emph{slot deficits} of its children to their sum divided by
$b$, up to an error controlled entirely by cross-slot inner products. No probability
enters until \S\ref{sec:depthfinite}; the lemma is exact bookkeeping.

\begin{lemma}[One-level contraction]\label{lem:contract}
Let $z_1,\dots,z_b$ and $z'_1,\dots,z'_b$ be unit vectors with $z_j=z'_j$ for
$j\notin E$ (the \emph{edited slots}); write $g_j:=1-\ip{z_j}{z'_j}$ for $j\in E$ and
$G:=\sum_{j\in E}g_j$ (throughout this section $g_j$, $G$ denote slot deficits; the gain
$g$ of \S\ref{sec:thmC} does not appear). Let $\rho_1,\dots,\rho_b$ be sign vectors,
$P:=\sum_j\rho_j\odot z_j$, $P':=\sum_j\rho_j\odot z'_j$, $y:=P/\norm{P}$,
$y':=P'/\norm{P'}$, and let the \emph{slot crosstalk} be
\[
\varepsilon:=\max_{i\ne j}\;\max\Bigl(
\bigl|\ip{\rho_i\odot z_i}{\rho_j\odot z_j}\bigr|,\;
\bigl|\ip{\rho_i\odot z_i}{\rho_j\odot z'_j}\bigr|,\;
\bigl|\ip{\rho_i\odot z'_i}{\rho_j\odot z'_j}\bigr|\Bigr).
\]
If $(b-1)\varepsilon\le\tfrac12$, then
\begin{equation}\label{eq:contract}
\Bigl|\bigl(1-\ip{y}{y'}\bigr)-\tfrac{G}{b}\Bigr|\;\le\;8(b-1)\varepsilon.
\end{equation}
\end{lemma}

\begin{proof}
By \eqref{eq:signfacts}(i) the diagonal part of $\ip{P}{P'}$ is
$\sum_j\ip{z_j}{z'_j}=b-G$; the $b(b-1)$ off-diagonal terms are each at most
$\varepsilon$ in magnitude, so $\ip{P}{P'}=b-G\pm b(b-1)\varepsilon$. The same count
gives $\norm{P}^2=b\pm b(b-1)\varepsilon$ and $\norm{P'}^2=b\pm b(b-1)\varepsilon$;
hence, writing $u:=(b-1)\varepsilon\le\tfrac12$,
$\norm{P}\norm{P'}=\sqrt{\norm{P}^2\norm{P'}^2}=b(1\pm u)$. Therefore
\[
1-\ip{y}{y'}\;=\;\frac{\norm{P}\norm{P'}-\ip{P}{P'}}{\norm{P}\norm{P'}}
\;=\;\frac{G\pm 2bu}{b(1\pm u)},
\]
and since $G\le 2|E|\le 2b$,
\[
\Bigl|\frac{G\pm2bu}{b(1\pm u)}-\frac{G}{b}\Bigr|
\;\le\;\frac{u\,(G+2b)}{b\,(1-u)}
\;\le\;\frac{u\cdot 4b}{b/2}\;=\;8u. \qedhere
\]
\end{proof}

For the finite-dimension theorem we need a sharper form when the edits are
\emph{small} --- the situation at every level above the edit, where the deficit has
already been contracted. The refinement rests on two identities that are \emph{exact},
with all approximation confined to a single square root.

\begin{lemma}[Refined contraction for small edits]\label{lem:contract2}
In the setting of Lemma~\ref{lem:contract}, write $z'_j=z_j+\delta_j$ for $j\in E$ (so
$\norm{\delta_j}=\sqrt{2g_j}$), and set
\[
B:=\sum_{j\in E}\sum_{i\ne j}\ip{\rho_i\odot z_i}{\rho_j\odot\delta_j},
\qquad
C_\delta:=\sum_{\substack{i,j\in E\\ i\ne j}}\ip{\rho_i\odot\delta_i}{\rho_j\odot\delta_j},
\qquad
\bar B:=B+\tfrac{1}{2}C_\delta.
\]
Then, \emph{exactly},
\begin{equation}\label{eq:exactids}
\ip{P}{P'}=\norm{P}^2-G+B,
\qquad
\norm{P'}^2=\norm{P}^2+2\bar B,
\end{equation}
and if $|\bar B|\le\norm{P}^2/4$,
\begin{equation}\label{eq:contract2}
\Bigl|\bigl(1-\ip{y}{y'}\bigr)-\frac{G}{\norm{P}^2}\Bigr|
\;\le\;\frac{2G|\bar B|}{\norm{P}^4}+\frac{|C_\delta|}{\norm{P}^2}
+\frac{3\bar B^2}{\norm{P}^4}.
\end{equation}
With the \emph{difference crosstalk}
$\tilde\varepsilon:=\max\bigl|\ip{\rho_j\odot\delta_j/\norm{\delta_j}}{\rho_i\odot w}\bigr|$
(maximum over $j\in E$, $i\ne j$, and $w\in\{z_i,\;\delta_i/\norm{\delta_i}\}$, the last
option when $i\in E$) one has
$|B|\le(b-1)\tilde\varepsilon\sum_{j\in E}\sqrt{2g_j}$ and
$|C_\delta|\le\tilde\varepsilon\bigl(\sum_{j\in E}\sqrt{2g_j}\bigr)^2\le
2|E|\,G\,\tilde\varepsilon$: every error term carries at least one factor $\sqrt{g}$ ---
the injected error is proportional to the \emph{size of the edit} --- which is what makes
the finite-$N$ error law of Theorem~\ref{thm:depthB} multiplicative rather than additive.
\end{lemma}

\begin{proof}
Let $\Delta:=\sum_{j\in E}\rho_j\odot\delta_j$, so $P'=P+\Delta$. First identity:
$\ip{P}{P'}=\norm{P}^2+\ip{P}{\Delta}$, and
$\ip{P}{\Delta}=\sum_{j\in E}\ip{\rho_j\odot z_j}{\rho_j\odot\delta_j}
+\sum_{j\in E}\sum_{i\ne j}\ip{\rho_i\odot z_i}{\rho_j\odot\delta_j}=-G+B$,
using \eqref{eq:signfacts}(i) and $\ip{z_j}{\delta_j}=\ip{z_j}{z'_j}-1=-g_j$. Second:
$\norm{P'}^2=\norm{P}^2+2\ip{P}{\Delta}+\norm{\Delta}^2$ with
$\norm{\Delta}^2=\sum_{j\in E}2g_j+C_\delta=2G+C_\delta$, so
$\norm{P'}^2=\norm{P}^2-2G+2B+2G+C_\delta=\norm{P}^2+2\bar B$.

For \eqref{eq:contract2}, set $s^2:=\norm{P}^2$ and $x:=2\bar B/s^2$, so $|x|\le\tfrac12$.
Since $\sqrt{1+x}=1+\tfrac{x}{2}-\xi$ with $0\le\xi\le0.36\,x^2$ on $|x|\le\tfrac12$
(Taylor with the second derivative of $\sqrt{1+x}$ maximized at $x=-\tfrac12$),
\[
\norm{P}\norm{P'}-\ip{P}{P'}
= s^2\sqrt{1+x}-\bigl(s^2-G+B\bigr)
= \bar B-s^2\xi+G-B
= G+\tfrac{1}{2}C_\delta-s^2\xi,
\]
and $\norm{P}\norm{P'}=s^2\sqrt{1+x}$, so
\[
\Bigl|\bigl(1-\ip{y}{y'}\bigr)-\frac{G}{s^2}\Bigr|
\;\le\; \frac{G}{s^2}\Bigl|\frac{1}{\sqrt{1+x}}-1\Bigr|
+\frac{|C_\delta|/2+s^2\xi}{s^2\sqrt{1+x}}
\;\le\; \frac{G}{s^2}\,|x|+\frac32\Bigl(\frac{|C_\delta|}{2s^2}+0.36\,x^2\Bigr),
\]
using $\bigl|1/\sqrt{1+x}-1\bigr|\le|x|$ and $1/\sqrt{1+x}\le\tfrac32$ on
$|x|\le\tfrac12$. The three resulting terms are $2G|\bar B|/s^4$, at most
$|C_\delta|/s^2$, and $0.54\,(2\bar B/s^2)^2\le3\bar B^2/s^4$.
\end{proof}

\begin{theorem}[The exact depth law]\label{thm:depthA}
In the zero-crosstalk idealization ($\varepsilon=\tilde\varepsilon=0$ at every level ---
the $N\to\infty$ limit of the random model), for the depth-$D$ full $b$-ary code under
either role scheme, provided all bound products appearing in the construction are
distinct (see \S\ref{sec:alias} for the shared-role exceptions):
\begin{equation}\label{eq:depthlaw}
m_{\mathrm{swap}}(D)\;=\;2\cdot b^{-D},\qquad
m_{\mathrm{replace}}(D)\;=\;1\cdot b^{-D},
\end{equation}
for both the sibling swap (two leaves exchanged under one deepest parent) and the far
swap (two leaves exchanged whose lowest common ancestor is the root). More generally, an
edit whose slot deficits at level $\ell_0$ sum to $G$ has root margin
$G\cdot b^{-(D-\ell_0)}$: the margin depends only on the total deficit and the number of
contraction levels, not on where in the tree the edit sits.
\end{theorem}

\begin{proof}
With $\varepsilon=0$, Lemma~\ref{lem:contract} is exact: one level maps slot deficits
summing to $G$ to a parent deficit of exactly $G/b$; induct upward.

\emph{Sibling swap.} The exchanged leaves $u,v$ occupy slots $1,2$ of one level-$1$
parent: apply Lemma~\ref{lem:contract} with $E=\{1,2\}$, $z_1=u$, $z'_1=v$, $z_2=v$,
$z'_2=u$; distinct idealized primitives are orthogonal, so $g_1=g_2=1$, $G=2$, and the
parent deficit is $2/b$. Each of the $D-1$ levels above contracts a single slot's
deficit by exactly $1/b$: $m=(2/b)\cdot b^{-(D-1)}=2b^{-D}$.

\emph{Far swap.} Each of the two subtrees hanging off the root sees its leaf
\emph{replaced} by a vector orthogonal to it: base deficit $1$, contracted through
$D-1$ levels of its subtree to $b^{-(D-1)}$ at the root's child. At the root,
$E=\{\mathrm{left},\mathrm{right}\}$ and $G=2b^{-(D-1)}$, giving $m=G/b=2b^{-D}$.

\emph{Replace.} One leaf replaced: base deficit $1$, contracted $D$ times: $b^{-D}$.
The general statement is the same induction; location independence is immediate since
the lemma's conclusion depends only on $G$ and $b$.
\end{proof}

\begin{remark}[Why the constant is exact, and why sibling equals far]\label{rem:exact}
The $G/b$ term suffers no $b$-dependent loss because sign-binding preserves each slot's
inner product exactly, by \eqref{eq:signfacts}(i), while normalization divides by
$\norm{P}\norm{P'}=b$ exactly at $\varepsilon=0$: the law is the ratio (per-slot edit
deficit)/(slots superposed), iterated. Depth $1$ with $b=k$ recovers the depth-$1$
constants --- swap $2/k$ (Theorem~\ref{thm:Cswap}) and replace $1/k$ (its companion
computation): the depth law is that analysis composed with itself. And the equality of sibling and far swaps --- which surprised us
as a measurement --- is transparent here: a far swap is two \emph{replacements} in
parallel branches, and $1+1$ contracted through $D$ levels equals $2$ contracted through
$D$ levels, by location independence.
\end{remark}

\subsection{Finite dimension: a multiplicative error law}\label{sec:depthfinite}

Finite $N$ replaces exact orthogonality by crosstalk of scale $\sqrt{1/N}$. The honest
statement is for the fresh role scheme; $Q\le b^D+1$ counts the distinct primitives (the
$b^D$ leaves plus the replacement filler).

\begin{lemma}[Level-0 crosstalk]\label{lem:crossbase}
With probability at least $1-\eta$, all pairwise inner products among the $Q$ distinct
primitives are at most $L_0=\sqrt{2\log(2Q^2/\eta)/N}$ in absolute value, and by
\eqref{eq:signfacts}(i) the bound survives binding by any fixed roles. (Hoeffding plus a
union bound over at most $Q^2$ pairs, exactly the argument of
Lemma~\ref{lem:primitives}.)
\end{lemma}

\begin{lemma}[Crosstalk propagation]\label{lem:ellfour}
Condition on all codes at level $\ell-1$ and on the event that every such code $z$ ---
and every normalized difference $\delta/\norm{\delta}$ of codes along the edited
paths --- satisfies $\norm{z}_4^4\le C_4/N$, where $\norm{\cdot}_4$ is the
coordinatewise $4$-norm (the $\ell_4$ hypothesis). Under the fresh scheme, for distinct slots
$i\ne j$ of one parent the level-$\ell$ roles give
$\ip{\rho_i\odot z_i}{\rho_j\odot z_j}=\sum_k\sigma_k\,(z_i\odot z_j)_k$ with
$\sigma:=\rho_i\odot\rho_j$ a uniform sign vector independent of the codes, so by
Hoeffding, with probability $1-\eta'$,
\[
\bigl|\ip{\rho_i\odot z_i}{\rho_j\odot z_j}\bigr|
\;\le\;\norm{z_i\odot z_j}\,\sqrt{2\log(2/\eta')}
\;\le\;\norm{z_i}_4\,\norm{z_j}_4\,\sqrt{2\log(2/\eta')}
\;\le\;\sqrt{\frac{2C_4\log(2/\eta')}{N}},
\]
and on the same event the same argument bounds the difference crosstalk
$\tilde\varepsilon$ (coefficient vectors $z_i\odot(\delta_j/\norm{\delta_j})$, and
likewise for the $\delta$-against-$\delta$ pairs). \emph{(Proved, on the stated event.)} \emph{Propagation of the $\ell_4$ hypothesis:} at level $0$,
$\norm{z}_4^4=1/N$ exactly, and a normalized primitive difference
$(u-v)/\norm{u-v}$ has $\ell_4^4=(2+o(1))/N$ directly. For the inductive
step, sibling codes are functions of disjoint leaf sets, hence independent conditional
on the roles, with sign-symmetric entries; a direct fourth-moment computation for
independent mean-zero slot entries gives
\[
\mathbb{E}\,\norm{y}_4^4\;\le\;\frac{3+C_4/b}{N}\,\bigl(1+o(1)\bigr),
\]
whose fixed point $C_4(b)=\tfrac{3b}{b-1}\,(1+o(1))$ is bounded uniformly in $D$: the
fourth moment does not grow with depth. \emph{(The expectation step is proved; the high-probability form of this inductive step --- concentration of
$\norm{y}_4^4$, of its analogue for the normalized differences along the edit path
(which obey their own recursion through \eqref{eq:exactids}), and of $\norm{P}$ around
their conditional means, e.g.\ by McDiarmid over the $bN$ fresh role signs --- is
Outlined: standard but bookkeeping-heavy, and it is the single unproved step in this
section.)}
\end{lemma}

\begin{theorem}[Finite-$N$ depth law, fresh roles]\label{thm:depthB}
There are absolute constants $C,C'$ such that for the fresh scheme with $D\ge2$ and
$N\ge C\,b^2D^2\log(bQ/\eta)$, with probability at least $1-\eta$, simultaneously for
every edit of Theorem~\ref{thm:depthA} (depth $1$ is the finite-$N$ regime of
Theorem~\ref{thm:Cswap} itself),
\[
m(D)\;=\;m_{\mathrm{ideal}}(D)\cdot(1\pm\kappa_N),
\qquad
\kappa_N\;\le\;C'\sqrt{\frac{b^{D}\,\log(bQD/\eta)}{N}},
\]
where $m_{\mathrm{ideal}}$ is the value in \eqref{eq:depthlaw}. In particular the
products $m\cdot b^{D}$ are constant to relative error $o(1)$ whenever
$b^{D}\ll N/\mathrm{polylog}$.
\end{theorem}

\begin{proof}[Proof (given the Outlined step of Lemma~\ref{lem:ellfour})]
Work on the intersection of the events of Lemmas~\ref{lem:crossbase}
and~\ref{lem:ellfour}, union-bounded over the $O(b^2D)$ slot pairs adjacent to the edit
paths of the three canonical edit types and over the levels; on it, every crosstalk appearing below is at most
$\varepsilon_N:=\sqrt{2C_4\log(Cb^2D/\eta)/N}$ (and $\tilde\varepsilon\le\varepsilon_N$),
and every base inner product is at most $L_0$. Run the induction of
Theorem~\ref{thm:depthA} carrying errors, applying Lemma~\ref{lem:contract2} at
\emph{every} level including the base: its identities \eqref{eq:exactids} are exact for
edits of any size, its smallness hypothesis $|\bar B|\le\norm{P}^2/4$ holds on the event
(each ingredient of $\bar B$ is $O(b\varepsilon_N)$ against $\norm{P}^2\ge b/2$), and at
every level $|E|\le2$ with $\sum_{j\in E}\sqrt{2g_j}\le2\sqrt3$. The base deficit itself
is $G_{\mathrm{ideal}}(1\pm L_0)$. Dividing \eqref{eq:contract2} by the level's ideal
deficit and using $s^2=b(1\pm(b-1)\varepsilon_N)$,
$|B|\le(b-1)\tilde\varepsilon\sum_{j\in E}\sqrt{2g_j}$, and
$|C_\delta|\le2|E|G\tilde\varepsilon$, each level contributes a \emph{relative} error
\[
e_\ell\;\le\;C_0\Bigl[(b-1)\,\varepsilon_N
\;+\;\frac{b-1}{b}\,\tilde\varepsilon\sum_{j\in E}\sqrt{2g_j}
\;+\;\tilde\varepsilon\;+\;(b-1)^2\tilde\varepsilon^{\,2}\Bigr]
\]
for an absolute constant $C_0$ --- the four terms coming from the normalization
crosstalk, the $2G|\bar B|/s^4$ term, the $|C_\delta|/s^2$ term measured relative to
$G/s^2$, and the $3\bar B^2/s^4$ term. Since the deficits decay geometrically along the
path, $\sum_\ell\sum_{j\in E}\sqrt{2g_j}$ is bounded by an absolute constant, so summing
over the at most $2D$ contraction steps of the two edit paths,
$\sum_\ell e_\ell\le C_1(1+bD)\,\varepsilon_N\le\tfrac12$, the last inequality by
$N\ge Cb^2D^2\log(bQ/\eta)$ with $C$ large; then
$\prod_\ell(1+e_\ell)-1\le2\sum_\ell e_\ell$, and the base factor contributes a further
$L_0$. Finally $(1+bD)\le3\,b^{D/2}$ for all $b\ge2$, $D\ge2$, and
$L_0\le\sqrt2\cdot\sqrt{\log(bQD/\eta)/N}$, so the accumulated relative error is at
most $C'\sqrt{b^{D}\log(bQD/\eta)/N}$.
\end{proof}

\begin{remark}[The proof gives more than the statement]\label{rem:polyD}
The accumulation above in fact bounds the relative error by an absolute constant times
$(1+bD)\,\varepsilon_N+L_0$ --- polynomial in $D$, not $b^{D/2}$ ---
because Lemma~\ref{lem:contract2} makes each level's injected error proportional to
$\sqrt{g_\ell}$ and the deficits decay geometrically. We state the theorem in the more
conservative form; both forms rest on the same single Outlined step.
\end{remark}

\begin{remark}[The error is multiplicative --- measured]\label{rem:multerr}
Theorem~\ref{thm:depthB} predicts the \emph{relative} spread of $m$ to be roughly flat
in $D$ at fixed $N$, with the \emph{absolute} spread shrinking like $b^{-D}$; an
additive error law would instead show flat absolute spread. Measured (per-seed relative
standard deviation of $m_{\mathrm{sib}}$, $b=2$, $N=4096$, $24$ seeds): $2.6\%$ at $D=1$
against $3.5\%$ at $D=10$, never above $6\%$ at any intermediate depth, while the
absolute standard deviation shrinks by a factor of ${\approx}370$; for reference,
$1/\sqrt N=1.6\%$. The side-by-side sweep reported in \S\ref{sec:alias} shows the same
flatness ($0.8\%$--$2.4\%$ across all $32$ of its configurations, both role schemes).
This is a structural confirmation of the contraction mechanism, not merely of its
conclusion.
\end{remark}

\subsection{The law, measured}\label{sec:measuredlaw}

Simulation, shared roles ($b=2$: $D\le10$; $b=3$: $D\le7$; $N\in\{256,\dots,65{,}536\}$;
$24$ seeds per configuration; three edit types): the products $m\cdot b^{D}$ sit at the
predicted constants $2$ (both swap types) and $1$ (replace), within $0.014$ and $0.006$
respectively at $N=65{,}536$ over every configuration (the aliased shared-role far-swap
cells of \S\ref{sec:alias} excepted) --- exactly the depth-$1$ constants of
Theorem~\ref{thm:Cswap} and its replace companion, composed multiplicatively per level, as
Theorem~\ref{thm:depthA} requires. At smaller $N$ the deviations grow --- maximum
$|m_{\mathrm{sib}}\,b^{D}-2|$ of $0.20$, $0.10$, $0.034$, $0.014$, $0.012$ across
$N=256,1024,4096,16384,65{,}536$, the other edit types behaving analogously at their own
scale --- and these are the finite-$N$ corrections whose size Theorem~\ref{thm:depthB}
bounds and whose multiplicative signature Remark~\ref{rem:multerr} confirms. A follow-up
sweep under the \emph{fresh} scheme ($N=16{,}384$, both schemes side by side, $24$
seeds, $32$ configurations in all) lands every swap product in $[1.985,2.014]$ and every
replace product in $[0.993,1.005]$, including as swaps the shared-role aliased cases
that the fresh scheme repairs (\S\ref{sec:alias}).

Crossing the law with a detection floor yields the ceiling:

\begin{corollary}[Resolvable depth]\label{cor:dstar}
Fix a readout whose decision floor is $\nu(N)=c\sqrt{\log(1/\eta)/N}$ --- the primitive
crosstalk scale of Lemma~\ref{lem:crossbase}, the floor of any linear readout against a
stored template. On the event of Theorem~\ref{thm:depthB} (depth $1$ being covered
exactly by Theorem~\ref{thm:Cswap}), binding at depth $D$ is
resolvable if $2b^{-D}(1-\kappa_N)>\nu(N)$ and unresolvable if
$2b^{-D}(1+\kappa_N)<\nu(N)$; either threshold reads
\begin{equation}\label{eq:dstar}
D\;\le\;D^*(N)\;=\;\frac{\log N}{2\log b}+O(1),
\end{equation}
a staircase gaining one level per $b^2$-fold increase of $N$. \emph{(Proved given Theorem~\ref{thm:depthB}; the staircase itself is a measurement.)}
\end{corollary}

Verified exactly: with floor $10/\sqrt N$, $D^*=1,2,3,4,5$ at $N=256,\dots,65{,}536$ for
$b=2$ (one level per $4\times$ dimension) and $1,1,2,2,3$ for $b=3$
(Figure~\ref{fig:depthlaw}). At CLIP's $N=512$ and linguistic branching: the pure crossing gives $4.5$ ($b=2$) and
$2.8$ ($b=3$), the measured $10/\sqrt N$ floor $D^*=2$ and $1$ --- single digits either
way: \textbf{the ceiling sits at ordinary
natural-language nesting depth}, which is precisely the regime where companion
grammar-vs-flat experiments (depth extrapolation, CLEVR-CoGenT, tree-structured
composers; reported separately) locate the advantage of structured models.

\subsection{Structural aliasing: some deep edits are invisible at any dimension}\label{sec:alias}

The simulation also revealed a phenomenon that goes beyond attenuation.

\begin{proposition}[Aliasing for self-inverse roles]\label{prop:alias}
With sign-vector roles shared across levels, roles are self-inverse under $\odot$
($r\odot r=\mathbf 1$), so a leaf's coefficient in the (pre-normalization) root expansion
depends only on the \emph{parity profile} of its role chain. Two leaves whose chains have
equal parity profiles have identical coefficient vectors; exchanging them leaves the
unnormalized root code (equivalently, any code normalized by deterministic per-level
scalars) \emph{exactly} unchanged --- margin $\equiv0$ at \emph{any} dimension, with no
orthogonality used. Under per-node normalization the exchange survives only in the
norm-mismatch scalars, so the residual margin is crosstalk-sized, carrying no $2b^{-D}$
term.
\end{proposition}
\begin{proof}
Unrolling the recursion, each leaf $x$ enters the root as
$(\bigodot_{\text{path}}r_{j})\odot x$ times a positive scalar --- the product of the
inverse norms along its path, level-uniform under deterministic normalization and
node-dependent (with crosstalk-sized deviations) under per-node normalization. Since
$r\odot r=\mathbf 1$, the path product depends only on which roles appear an odd number
of times. Equal parity profiles $\Rightarrow$ equal coefficient vectors $\Rightarrow$
the two leaves' contributions commute with exchange --- exactly in the deterministic
case, and up to the norm mismatch in general.
\end{proof}

The simulation confirms it: for $b=3$ at even depths, the far swap (both leaves at
position-1 chains of equal parity) has margin two to five orders of magnitude below the
law's prediction at every $N$ measured, shrinking with $N$ as pure crosstalk must. Flat
sign-role binding does not merely attenuate deep structure --- it \emph{aliases} some of
it. Fresh per-level roles repair it, measured directly: in the side-by-side sweep at
$N=16{,}384$ ($24$ seeds), the aliased configurations ($b=3$; $D=2,4,6$) have
$m_{\mathrm{far}}\cdot b^{D}$ of $0.0001$, $0.0003$, $0.0006$ under shared roles ---
and $1.994$, $2.000$, $2.000$ under fresh roles, exactly the generic law
\eqref{eq:depthlaw}. Theorem~\ref{thm:depthB} is stated for precisely that scheme;
aliasing stands as a separate warning for self-inverse shared-role binding.

\subsection{Status and positioning}

The proof route this section originally recorded only as an outline --- induction on the
Theorem~\ref{thm:Cswap} computation, with Lemma~\ref{lem:primitives}-style crosstalk
control --- is now carried out in \S\ref{sec:depthproof}--\S\ref{sec:depthfinite}. The
status is as follows: Lemmas~\ref{lem:contract}--\ref{lem:contract2} and
Theorem~\ref{thm:depthA} are proved in full; Theorem~\ref{thm:depthB} is proved modulo the single
outlined concentration step inside Lemma~\ref{lem:ellfour};
Corollary~\ref{cor:dstar} is proved given Theorem~\ref{thm:depthB}; the law's empirical
confirmation, the staircase, the multiplicative-error signature, and the fresh-role
repair of aliasing are measured. The sequence-memory analysis of Frady, Kleyko and Sommer \cite{frady}
supplies closely related crosstalk accounting, with their contraction $\lambda^K$ (from
recurrent dynamics) replaced here by $b^{-D}$ (from per-node normalization); their
theory has no per-level margin law and no two-structure discrimination --- the gap
Theorems~\ref{thm:depthA}--\ref{thm:depthB} close. Plate's classical account \cite{plate} describes
qualitative degradation but draws the opposite conclusion --- that deep structure
remains workable, and that chunking removes the limit --- chunking being exactly the
structured-scorer escape branch. \emph{Scope caution:} as a no-go the statement is class-relative (bounded-norm
homomorphic combinators); an unconstrained encoder can look up any finite set of deep
captions, so the general-encoder version requires a margin-complexity argument or a
measurable hypothesis in the style of Theorem~\ref{thm:E}.

\section{What does \emph{not} limit binding: a conjecture refuted}\label{sec:notlimit}

Pinning down the reason for a failure requires eliminating candidate reasons. This section
records a plausible, literature-adjacent conjecture that we formulated and then
\emph{refuted}; the refutation sharpened the map.

\subsection{The conjecture}

Training exposes a set $P\subseteq O\times A$ of bound pairs --- the \emph{exposure graph}
$G$, bipartite on objects and attributes. In a factored world
($W(o,a)=u(o)\odot v(a)$), classical rank-one completion theory \cite{kiraly} says the
codes are identified only up to a per-connected-component gauge, and cross-component
products only up to a sign. The conjecture: \emph{binding generalization to a pair whose
endpoints lie in different components of $G$ is information-theoretically impossible ---
no learner can beat chance --- so compositional generalization undergoes a phase
transition at the connectivity threshold of $G$.} (Adjacent published work: a connectivity-type compositional-support condition is
assumed for identification by Schug et al.~\cite{schug}; a percolation model,
established by analogy rather than learner-side theorems, is given by Lubana et
al.~\cite{lubana}.)

\begin{figure}[t]
\centering
\begin{tikzpicture}[font=\scriptsize,
  ov/.style={circle, draw=black!60, fill=black!6, inner sep=1.6pt},
  arr/.style={-{Stealth[length=1.8mm]}, black!70}]
\begin{scope}
\node[ov, draw=teal!55!black] (o1) at (0,2.4) {$o_1$};
\node[ov, draw=teal!55!black] (o2) at (0,1.6) {$o_2$};
\node[ov, draw=teal!55!black] (a1) at (2.2,2.4) {$a_1$};
\node[ov, draw=teal!55!black] (a2) at (2.2,1.6) {$a_2$};
\draw[teal!55!black] (o1)--(a1) (o1)--(a2) (o2)--(a1);
\node[teal!45!black] at (1.1,2.95) {component $C_1$};
\end{scope}
\begin{scope}
\node[ov, draw=orange!85!black] (o3) at (0,0.4) {$o_3$};
\node[ov, draw=orange!85!black] (a3) at (2.2,0.4) {$a_3$};
\node[ov, draw=orange!85!black] (a4) at (2.2,-0.4) {$a_4$};
\draw[orange!85!black] (o3)--(a3) (o3)--(a4);
\node[orange!60!black] at (1.1,-0.95) {component $C_2$};
\end{scope}
\draw[dashed, black!70] (o2) -- (a3);
\node[black!70, rotate=-26, fill=white, inner sep=1pt, font=\tiny] at (1.05,1.05) {test pair $(o_2,a_3)$};
\node[align=left, anchor=west, font=\scriptsize] at (3.4,1.5)
  {gauge: cross-component code known only up to sign,\\
   $\widehat W(o_2,a_3)=\pm\,W(o_2,a_3)$ \quad(\cite{kiraly})};
\node[align=left, anchor=west, font=\scriptsize, text=teal!45!black] at (3.4,0.4)
  {but the swap test's candidates use \emph{different templates},\\
   so $|\ip{X}{t}|$-matching never consults the sign\\
   $\Rightarrow$ accuracy $1.0$ across components (Prop.~\ref{prop:matcher})};
\end{tikzpicture}
\caption{The refuted conjecture. Across disconnected exposure components the factored code
of an unseen pair is identified only up to sign --- but the swap decision compares
templates with different \emph{supports}, and a sign-agnostic matcher decides it without
ever resolving the gauge. Coverage (both atoms seen), not connectivity, is the true
information threshold.}
\label{fig:connectivity}
\end{figure}
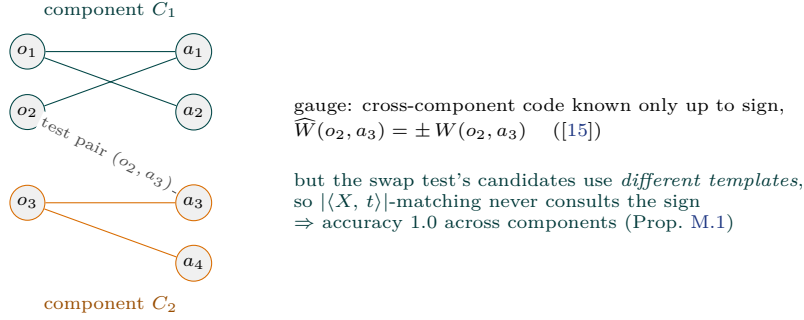

\subsection{The refutation}

\begin{proposition}[The sign-agnostic matcher succeeds across components]\label{prop:matcher}
In the noiseless factored world, let both atoms of the unseen test pairs be exposed
(possibly in different components of $G$), and let
$\widehat u,\widehat v$ be recovered by gauge propagation (exact up to a sign per
component). For the swap test on the recombined scene $s=\{(o_1,a_1'),(o_2,a_2')\}$,
score each candidate caption by
$\;\mathrm{score}(T)=|\ip{X}{t_1}|+|\ip{X}{t_2}|\;$ where $X$ is the image code and
$t_1,t_2$ the candidate's pair templates. On the primitive-control event of
Lemma~\ref{lem:primitives}, the correct caption's score is $\ge2-O(k\Lstar)$ while the
swapped caption's is $O(k\Lstar)$: the decision is correct regardless of the unresolved
signs, because the two candidates' templates have different supports and the modulus
discards the gauge.
\end{proposition}
\begin{proof}
$X=W(o_1,a_1')+W(o_2,a_2')$. The correct candidate's templates are
$\pm W(o_1,a_1'),\pm W(o_2,a_2')$ (signs unknown): $|\ip{X}{t_i}|=|1+O(\Lstar)|$ each. The
swapped candidate's templates are $\pm W(o_1,a_2'),\pm W(o_2,a_1')$ --- distinct atom
pairs, so every inner product with $X$ is a primitive-sized quantity $O(\Lstar)$
(Lemma~\ref{lem:primitives}).
\end{proof}

Simulation ($M{=}K{=}40$, $N{=}2048$, gauge-propagation
learner, exposure density swept through the fragmented regime): sign-agnostic accuracy
$=1.000$ at every density, \emph{including when $0\%$ of test pairs share a component}.
The conjecture is dead for swap tests: \textbf{coverage (both atoms seen), not
connectivity, is the information threshold} --- precisely the occupancy law that
Proposition~\ref{prop:D}(ii) reports as measured.

\subsection{What survives}

\begin{enumerate}
\item \textbf{A below-chance signature of gauge commitment.} The sign-\emph{naive} learner
(one that commits to its arbitrary per-component gauge) scores $0.27$--$0.37$ ---
systematically \emph{below} chance --- on cross-component tests. An estimator that commits
to an unidentified convention binds \emph{backwards}, echoing both the measured
backwards-binding of leakage-prone codes (\S\ref{sec:thmC}) and the selection-inside-the-
sliver picture of \S\ref{sec:throttle}; cf.\ the min-degree implicit-bias analysis of
Abbe et al.~\cite{abbe}.
\item \textbf{The occupancy law as a theorem-target.} The measured curve
$\mathrm{acc}(\rho)\approx\tfrac12+\tfrac12(1-e^{-\rho K})(1-e^{-\rho M})$ is an
isolated-atom (coupon-collector) asymptotic, derivable rather than fitted --- and distinct
from the giant-component threshold of the percolation narrative \cite{lubana}, a
reconciliation worth writing out.
\item \textbf{A design note.} A genuine connectivity no-go would require a task whose
answer \emph{depends on the gauge} (sign-sensitive retrieval with sign-flipped
distractors); benchmark-style swap tests are not such a task.
\end{enumerate}

\section{The updated map, and open problems}\label{sec:map2}

\subsection{The causal story, as the results now support it}

\emph{Binding is representable and cheap} (factored codes, $N=O(k^2\log MK)$,
\S\ref{sec:thmC}); \emph{the geometry permits it} (Corollary~\ref{cor:satisfiable});
\emph{the population objective, at universal capacity, even prefers it}
(Lemma~\ref{lem:infonce}). Real systems fail because:
\begin{enumerate}
\item \textbf{the objective's reward for binding is a sliver} of width
$(\pi+\pi_{\mathrm{impl}})\log2$ nats --- now proved (Theorems~\ref{thm:T1}--\ref{thm:T2})
--- so training pressures of any size beyond that (capacity, regularization, early
stopping, implicit bias) decide the outcome; and
\item \textbf{learned composition codes leak} (measured $\ell\approx+0.17$;
\S\ref{sec:thmC}), and inside the sliver even the exactly-reversed solution is
$\pi\,\mathbb E[M]$-cheap --- which is why below-chance binding is observed.
\end{enumerate}
Behind these \emph{contingent} causes stand \emph{terminal} limits that survive all fixes:
the smoothness frontier (Theorem~\ref{thm:E}; measured slack today) and the depth ceiling
($D^*=\Theta(\log N/\log b)$, \S\ref{sec:depth}; sitting at natural-language depth for
CLIP-sized models). And two candidate causes are \emph{eliminated}: dimension
(\S\ref{sec:possibility}) and exposure-graph connectivity (\S\ref{sec:notlimit}).

\subsection{Open problems (Part II update)}

\begin{enumerate}
\item \textbf{Selection inside the sliver.} Which solution do SGD dynamics pick when the
objective is indifferent at the $(\pi+\pi_{\mathrm{impl}})$ scale? The below-chance
signatures (leakage; gauge commitment) suggest the tie-breakers are not benign.
\item \textbf{Finite-sample throttle.} An $\Omega(1/\pi^2)$ lower bound for detecting the
correct binding sign from rate-$\pi$ swap-contrast events (two-point testing).
\item \textbf{The depth-law concentration step.} The law \eqref{eq:depthlaw} is now
Proved (Theorems~\ref{thm:depthA}--\ref{thm:depthB}) except for one concentration
estimate --- the high-probability inductive step of Lemma~\ref{lem:ellfour} --- which
remains Outlined; close it. Then the margin-vs-depth measurement on \emph{real} nested
captions through real text towers --- the depth analogue of the \S\ref{sec:zoo} zoo sweep.
\item \textbf{Instance binding vs.\ co-occurrence priors.} For non-exchangeable
$\mathcal D$, separate what $\pi=0$ training can learn (the prior $p(a\mid o)$ --- the
benchmark ``text prior'') from what it cannot (instance pairing).
\item \textbf{Part I items that remain open:} the anchor-robustness appendix for the
frontier; naturalistic multi-relation probes; the NegCLIP three-axis confound.
\end{enumerate}

\end{document}